\documentclass[english]{article}

\usepackage{geometry}
\usepackage[T1]{fontenc}
\usepackage[latin9]{inputenc}
\usepackage{bm}
\usepackage{amsmath,mathtools}
\usepackage{amssymb}
\usepackage[unicode=true,
 bookmarks=false,
 breaklinks=false,pdfborder={0 0 1},colorlinks=false]
 {hyperref}
\hypersetup{
 colorlinks,citecolor=blue,filecolor=blue,linkcolor=blue,urlcolor=blue}

\makeatletter
\usepackage{amsthm}
\usepackage{cite}  
\usepackage{comment}
\usepackage{natbib}
\usepackage{booktabs}
\usepackage{graphicx}

\usepackage{float}
\usepackage{multirow}
\usepackage{dsfont}
\usepackage{tcolorbox}

\usepackage{color}
\definecolor{yxc}{RGB}{255,0,0}
\definecolor{yjc}{RGB}{232, 9, 147}
\definecolor{ytw}{RGB}{255,69,0}
\definecolor{gen}{RGB}{0,0,200}

\allowdisplaybreaks

\newcommand{\defn}{\coloneqq}

\newcommand{\mymid}{\,|\,}

\newcommand{\soft}[1]{{#1}}

\newcommand{\cX}{\mathcal{X}}
\newcommand{\cS}{\mathcal{S}}
\newcommand{\cA}{\mathcal{A}}

\newcommand{\cprime}{c_{\mathsf{ref}}}

\newcommand{\taun}{{\tau}}

\newcommand{\cp}{c_{\mathrm{p}}}
\newcommand{\ch}{c_{\mathrm{h}}}
\newcommand{\cm}{c_{\mathrm{m}}}

\newcommand{\cbone}{c_{\mathrm{b},1}}
\newcommand{\cbtwo}{c_{\mathrm{b},2}}

\newcommand{\sprimary}{\mathcal{S}_{\mathsf{primary}}}
\newcommand{\sadj}{\mathcal{S}_{\mathsf{adj}}}

\newcommand{\Qpi}{Q^{\pi}}

\newcommand{\czero}{c_{0}}
\newcommand{\tlow}{t_{\overline{s-1}}(\tau_s) }
\newcommand{\tprime}{t_{\mathsf{ref}}}
\newcommand{\that}{t_{\mathsf{tran}}}
\newcommand{\ttilde}{\widetilde{t}}

\title{Softmax Policy Gradient Methods \\ Can Take Exponential Time to Converge\footnotetext{This work was presented in part at the Conference on Learning Theory (COLT) 2021.}}

\author{Gen Li\thanks{Department of Statistics and Data Science, The Wharton School, University of Pennsylvania, Philadelphia, PA 19104, USA.} \\
UPenn    \\
	\and
	Yuting Wei\footnotemark[1]\\ UPenn
	\and
	Yuejie Chi\thanks{Department of Electrical and Computer Engineering, Carnegie Mellon University, Pittsburgh, PA 15213, USA.}\\
	CMU\\
	\and
	Yuxin Chen\footnotemark[1] \\
 UPenn  \\
	}

\date{February 2021;~~Revised: December 2022}

\makeatother

\begin{document}

\theoremstyle{plain} \newtheorem{lemma}{\textbf{Lemma}}\newtheorem{proposition}{\textbf{Proposition}}\newtheorem{theorem}{\textbf{Theorem}}

\theoremstyle{remark}\newtheorem{remark}{\textbf{Remark}}

\maketitle

\begin{abstract}
	The softmax policy gradient (PG) method, which performs gradient ascent under softmax policy parameterization, is arguably one of the de facto implementations of policy optimization in modern reinforcement learning. For $\gamma$-discounted infinite-horizon tabular Markov decision processes (MDPs), remarkable progress has recently been achieved towards establishing global convergence of softmax PG methods in finding a near-optimal policy. However, prior results fall short of delineating clear dependencies of convergence rates on salient parameters such as the cardinality of the state space  $\mathcal{S}$ and the effective horizon $\frac{1}{1-\gamma}$, both of which could be excessively large. In this paper, we deliver a pessimistic message regarding the iteration complexity of softmax PG methods, despite assuming access to exact gradient computation. Specifically, we demonstrate that the softmax PG method with stepsize  $\eta$ can take 
\[
	\frac{1}{\eta} |\mathcal{S}|^{2^{\Omega\big(\frac{1}{1-\gamma}\big)}} ~\text{iterations}
\]
to converge, even in the presence of a benign policy initialization and an initial state distribution amenable to exploration (so that the distribution mismatch coefficient is not exceedingly large). This is accomplished by characterizing the algorithmic dynamics over a carefully-constructed MDP containing only three actions. Our exponential lower bound hints at the necessity of carefully adjusting update rules or enforcing proper regularization in accelerating PG methods. 

%
\end{abstract}


\setcounter{tocdepth}{2}
\tableofcontents

\section{Introduction}

Despite their remarkable empirical popularity in modern reinforcement learning \citep{mnih2015human,silver2016mastering},  theoretical underpinnings of policy gradient (PG) methods and their variants \citep{williams1992simple,sutton2000policy,kakade2002natural,peters2008natural,konda2000actor} remain severely obscured. Due to the nonconcave nature of value function maximization induced by complicated dynamics of the environments, it is in general highly challenging to pinpoint the computational efficacy of PG methods in finding a near-optimal policy. 
Motivated by their practical importance, a recent strand of work sought to make progress towards demystifying the effectiveness of policy gradient type methods (e.g., \citet{agarwal2019optimality,mei2020global,bhandari2019global,shani2019adaptive,bhandari2020optimization,cen2020fast,fazel2018global,zhang2020sample,zhang2020global,bhandari2020note,yang2020finding,lan2021policy,zhan2021policy,mei2021leveraging,zhang2021convergence,liu2020improved,khodadadian2021finite}), focusing primarily on canonical settings such as tabular Markov decision processes (MDPs) for discrete-state problems and linear quadratic regulators for continuous-state problems.

The current paper studies PG methods with softmax parameterization --- commonly referred to as {\em softmax policy gradient} methods --- which are among the de facto implementations of PG methods in practice. An intriguing theoretical result was recently obtained by the work \citet{agarwal2019optimality}, which established asymptotic global convergence of softmax PG methods for infinite-horizon $\gamma$-discounted tabular MDPs. Subsequently, \citet{mei2020global} strengthened the theory by demonstrating that softmax PG methods are capable of finding an $\varepsilon$-optimal policy with an iteration complexity proportional to $1/\varepsilon$ (see Table~\ref{tab:comparisons} for the precise form). While these results take an important step towards understanding the effectiveness of softmax PG methods, caution needs to be exercised before declaring fast convergence of the algorithms. 
In particular, the iteration complexity derived by \citet{mei2020global} falls short of delineating clear dependencies on important salient parameters of the MDP, such as the dimension of the state space $\cS$ and the effective horizon $1/(1-\gamma)$. These parameters are, more often than not, enormous in contemporary RL applications, and might play a pivotal role in determining the scalability of softmax PG methods.

Additionally, it is worth noting that existing literature largely concentrated on developing algorithm-dependent upper bounds on the iteration complexities. Nevertheless, we recommend caution when directly comparing computational upper bounds for distinct algorithms: the superiority of the computational upper bound for one algorithm does not necessarily imply this algorithm outperforms others, unless we can certify the tightness of all upper bounds being compared. As a more concrete example, it is of practical interest to benchmark softmax PG methods against natural policy gradient (NPG) methods with softmax parameterization, the latter of which is a variant of policy optimization lying underneath several mainstream RL algorithms such as {\em proximal policy optimization} (PPO) \citep{schulman2017proximal} and  {\em trust region policy optimization} (TRPO) \citep{schulman2015trust}. While it is tempting to claim superiority of NPG methods over softmax PG methods --- given the appealing convergence properties of NPG methods \citep{agarwal2019optimality} (see Table~\ref{tab:comparisons}) --- existing theory fell short to reach such a conclusion, due to the absence of convergence lower bounds for softmax PG methods in prior literature.

The above considerations thus lead to a natural question that we aim to address in the present paper: 
\begin{center}
{\em Can we develop a lower bound on the iteration complexity of softmax PG methods that reflects \\ explicit dependency on salient parameters of the MDP of interest?}
\end{center}


\subsection{Main result}

As an attempt to address the question posed above, our investigation delivers a somewhat surprising message that can be described in words as follows:
\begin{center}
	{\em Softmax PG methods can take (super-)exponential time to converge, even in the presence of a benign initialization and an initial state distribution amenable to exploration.   }
\end{center}
Our finding, which is concerned with a discounted infinite-horizon tabular MDP, is formally stated in the following theorem. Here and throughout, $|\cS|$ denotes the size of the state space $\cS$, $0<\gamma<1$ stands for the discount factor, $V^{\star}$ indicates the optimal value function, $\eta>0$ is the learning rate or stepsize, whereas  $V^{(t)}$ represents the value function estimate of softmax PG methods in the $t$-th iteration. All immediate rewards are assumed to fall within $[-1,1]$. See Section~\ref{sec:backgrounds} for formal descriptions. 
%
\begin{theorem} 
\label{thm:unregularized-main}
Assume that the softmax PG method adopts a uniform initial state distribution, a uniform policy initialization, and has access to exact gradient computation. Suppose that $0<\eta < (1-\gamma)^2/5$, then 
there exist universal constants  $c_1, c_2, c_3>0$ such that: for any $0.96< \gamma < 1$ and $ |\cS| \ge c_3 (1-\gamma)^{-6}$, one can find a $\gamma$-discounted MDP with state space $\cS$ that takes the softmax PG method at least
\begin{align} 
	\label{eq:t-bound-main}
	\frac{c_1}{\eta} |\cS| ^{ 2^{ \frac{c_2}{1-\gamma}}} ~\text{iterations}
\end{align}
to reach 
\begin{align} \label{eq:mean-error}
	\frac{1}{|\mathcal{S}|}\sum_{s \in  \mathcal{S}} \big[ V^{\star}(s)-V^{(t)}(s) \big]\leq 0.07.
\end{align} 
\end{theorem}
%
\begin{remark}[Action space]
	The MDP we construct contains at most three actions for each state.
\end{remark}
\begin{remark}[Stepsize range] 
	Our lower bound  operates under the assumption that $\eta < (1-\gamma)^2/5$. 
	In comparison, prior convergence guarantees for PG-type methods with softmax parameterization (e.g., \citet[Theorem~5.1]{agarwal2019optimality} and \citet[Theorem~6]{mei2020global}) required $\eta<(1-\gamma)^3/8$, 
	a range of stepsizes fully covered by our theorem.  In fact, prior works could only guarantee monotonicity of softmax PG methods (in terms of the value function) within the range $\eta<(1-\gamma)^2/5$ (see \citet[Lemma C.2]{agarwal2019optimality}). 
\end{remark}
\begin{remark}
	While we can also provide explicit numbers for the constants $c_1, c_2, c_3>0$, 
	these numbers are not informative, and hence we omit explicit numbers here to streamline the proof a bit. 
\end{remark}

For simplicity of presentation, Theorem~\ref{thm:unregularized} is stated for the long-effective-horizon regime where $\gamma > 0.96$; it continues to hold when $\gamma > c_0$ for some smaller constant $c_0>0$. Our result is obtained by exhibiting a hard MDP instance --- which is a properly augmented chain-like MDP --- for which softmax PG methods converge extremely slowly even when perfect model specification is available. Several remarks and implications of our result are in order. 
%

%


\newcommand{\topsepremove}{\aboverulesep = 0mm \belowrulesep = 0mm} \topsepremove

\begin{table}[t]

\begin{center}
	{
\begin{tabular}{c|c|c}
\toprule \hline
algorithm   &   iteration complexity &  reference    \tabularnewline 
\hline 
\multirow{2}{*}{softmax PG upper bound} 	 &   \multirow{2}{*}{asymptotic}  & \multirow{2}{*}{\citet[Thm.~5.1]{agarwal2019optimality}}  \tabularnewline
	 &  &      \tabularnewline
\hline 
	\multirow{2}{*}{softmax PG upper bound} 	 &   \multirow{2}{*}{$O \Big( {\color{red} \mathcal{C}^2_{\mathsf{spg}}(\mathcal{M})} \Big\|\frac{d_{\mu}^{\pi^{\star}}}{\mu} \Big\|_{\infty}^2\Big\|\frac{ 1}{\mu} \Big\|_{\infty} \frac{|\cS|}{  (1-\gamma)^6 \varepsilon} \Big) $}  & \multirow{2}{*}{\citet[Thm.~4]{mei2020global}}  \tabularnewline
	 &  &      \tabularnewline
\hline 
\multirow{2}{*}{softmax NPG upper bound}    &  \multirow{2}{*}{$O\big(\frac{1}{(1-\gamma)^{2}\varepsilon}  \big)$  } & \multirow{2}{*}{\citet[Thm.~5.3]{agarwal2019optimality}} \tabularnewline
	 &  &        \tabularnewline
\hline 
	\multirow{2}{*}{softmax PG lower bound}   & \multirow{2}{*}{$ \frac{(1-\gamma)^5 \Delta_{\star}^2}{ 12 \varepsilon}   $}   & \multirow{2}{*}{\citet[Thm.~10]{mei2020global}} \tabularnewline
	&  &          \tabularnewline
\hline
	\multirow{2}{*}{softmax PG lower bound} & \multirow{2}{*}{$ |\cS|^{2^{\Omega(\frac{1}{1-\gamma})} } $}   & \multirow{2}{*}{\textbf{this work}} \tabularnewline
	&  &         \tabularnewline
\toprule \hline
\end{tabular}
	}
\end{center}
	\caption{Upper and lower bounds on the iteration complexities of PG and NPG methods with softmax parameterization in finding an $\varepsilon$-optimal policy obeying $\|V^{\star}-V^{(t)}\|_{\infty}\leq\varepsilon \leq 0.15$.  
	We assume exact gradient evaluation, and hide the dependencies that are logarithmic on problem parameters. Here, $\mu$ denotes the initial state distribution, $\big\|{d_{\mu}^{\pi^{\star}}}/{\mu} \big\|_{\infty}$ is the distribution mismatch coefficient, $a^{\star}(s)$ is the optimal action at state $s$ according to $\pi^{\star}$, $\mathcal{C}_{\mathsf{spg}}(\mathcal{M})\coloneqq  \big( \inf_{s\in \cS }\inf_{t\geq 1} \pi^{(t)}(a^{\star}(s) \mymid s) \big)^{-1}$ is a quantity depending on both the PG trajectory and salient MDP parameters, whereas $\Delta_{\star} : = \min_{s\in \cS, a\neq a^{\star}(s) } \big\{ Q^{\star}(s, a^{\star}(s)) - Q^{\star}(s,a) \big\} $ is the optimality gap w.r.t.~the optimal Q-function $Q^{\star}$.  \label{tab:comparisons}  }

\end{table}

\paragraph{Comparisons with prior results.} Table~\ref{tab:comparisons} provides an extensive comparison of the iteration complexities --- including both upper and lower bounds --- of PG and NPG methods under softmax parameterization. 
As suggested by our result,  the iteration complexity $O(\mathcal{C}^2_{\mathsf{spg}}(\mathcal{M}) \frac{1}{\varepsilon})$ derived in \citet{mei2020global} (see Table~\ref{tab:comparisons}) might not be as rosy as it seems for problems with large state space and long effective horizon; in fact, the crucial quantity $\mathcal{C}_{\mathsf{spg}}(\mathcal{M})$ therein   could scale in a prohibitive manner with both $|\cS|$ and $\frac{1}{1-\gamma}$. 
\citet{mei2020global} also developed a lower bound on the iteration complexity of softmax PG methods, which falls short of capturing the influence of the state space dimension and might become smaller than $1$ unless $\varepsilon$ is very small (e.g., $\varepsilon\lesssim (1-\gamma)^3$) for problems with long effective horizons. 
In addition, \citet{mei2020escaping} provided some interesting evidence that a poorly-initialized softmax PG algorithm can get stuck at suboptimal policies for a single-state MDP (i.e., the bandit problem). This result, however, fell short of providing a complete runtime analysis and did not look into the influence of a large state space. By contrast, our theory reveals that softmax PG methods can take exponential time to reach even a moderate accuracy level.

\paragraph{Slow convergence even with benign distribution mismatch.}
Existing computational complexities for policy gradient type methods (e.g., \citet{agarwal2019optimality,mei2020global}) typically scale polynomially 
	in the so-called {\em distribution mismatch coefficient}\footnote{Here and throughout, the division of two vectors represents
	componentwise division.} $\big\|\frac{d^{\pi}_{\rho}}{\mu} \big\|_{\infty} $, where $d^{\pi}_{\rho}$ stands for a certain discounted state
	visitation distribution (see \eqref{eq:defn-d-rho} in Section~\ref{sec:backgrounds}), and $\mu$ denotes the distribution over initial states. 
It is thus natural to wonder whether the exponential lower bound in Theorem~\ref{thm:unregularized-main} is a consequence of an exceedingly large distribution mismatch coefficient. 
This, however, is not the case; in fact,  our theory chooses $\mu$ to be a benign uniform distribution so that  $\|\frac{d^{\pi}_{\rho}}{\mu}\|_{\infty} \le \|\frac{1}{\mu}\|_{\infty} \le |\mathcal{S}|$, which scales at most linearly in $|\mathcal{S}|$. 

\paragraph{Benchmarking with softmax NPG methods.} Our algorithm-specific lower bound suggests that softmax PG methods --- in their vanilla form --- might take  a prohibitively long time to converge when the state space and effective horizon are large. This is in stark contrast to the convergence rate of NPG type methods, 
whose iteration complexity is dimension-free and scales only polynomially with the effective horizon \citep{agarwal2019optimality,cen2020fast}. Consequently, our results shed light on the practical superiority of NPG-based algorithms such as PPO \citep{schulman2017proximal} and TRPO \citep{schulman2015trust}.

\paragraph{Crux of our design.} As we shall elucidate momentarily in Section~\ref{sec:mdp-construction}, our exponential lower bound is obtained  through analyzing the trajectory of softmax PG methods on a carefully-designed MDP instance with no more than 3 actions per state, when a uniform initialization scheme and a uniform initial state distribution are adopted. Our construction underscores the critical challenge of {\em credit assignments} \citep{sutton1984temporal} in RL compounded by the presence of delayed rewards, long horizon, and intertwined interactions across states. 
While it is difficult to elucidate the source of exponential lower bound without presenting our MDP construction, 
we take a moment to point out some critical properties that underlie our designs. To be specific, we seek to design a chain-like MDP containing $H=O\big( \frac{1}{1-\gamma} \big)$ key primary states $\{1,\cdots,H\}$ (each coupled with many auxiliary states), for which the softmax PG method satisfies the following properties. 
\begin{itemize}
	\item For the two key primary states, we have
		\begin{align}
			\min\big\{  \mathsf{convergence}\text{-}\mathsf{time}\text{(\,state 1)\,}, \, 
			\mathsf{convergence}\text{-}\mathsf{time}\text{(\,state 2)\,} \big\} \geq \frac{|\cS|}{\eta} .
			\label{eq:intuition-jump-start}
		\end{align}

	\item {\em (A blowing-up phenomenon)} For each key primary state $3\leq s\leq H=O\big(\frac{1}{1-\gamma}\big)$, one has
		\begin{align}
			\mathsf{convergence}\text{-}\mathsf{time}\text{(\,state }s\,) 
			\gtrsim 
			\big( \mathsf{convergence}\text{-}\mathsf{time}\text{(\,state }s-2\,) \big)^{1.5}, \qquad 3\leq s\leq H. 
			\label{eq:intuition-blowing-up}
		\end{align}
\end{itemize}
Here, it is understood that $\mathsf{convergence}\text{-}\mathsf{time}\text{(\,state }s\,)$ represents informally the time taken for the value function of state $s$ to be sufficiently close to its optimal value.  The blowing-up phenomenon described above is precisely the source of our (super)-exponential lower bound.

\subsection{Other related works}


\paragraph{Non-asymptotic analysis of (natural) policy gradient methods.} Moving beyond tabular MDPs, finite-time convergence guarantees of PG\,/\,NPG methods and their variants have recently been studied for control problems (e.g., \citet{fazel2018global,jansch2020convergence,tu2019gap,zhang2019policy}), regularized MDPs (e.g., \citet{lan2021policy,cen2020fast,zhan2021policy}),  constrained MDPs (e.g., \citet{xu2020primal,ding2020natural}), robust MDPs (e.g., \citet{zhang2021robust,li2022first}), MDPs with function approximation (e.g., \citet{agarwal2019optimality,cai2019provably,wang2019neural,liu2019neural,agazzi2020global,lan2022policy}), Markov games (e.g., \citet{daskalakis2020independent,wei2021last,zhao2021provably,cen2021fast,xie2020provable}), and their use in actor-critic methods (e.g., \citet{wu2020finite,xu2020non,alacaoglu2022natural,cen2022faster}).

\paragraph{Other policy parameterizations.} In addition to softmax parameterization, several other policy parameterization schemes have also been investigated in the context of policy optimization and reinforcement learning at large. For example, \citet{agarwal2019optimality,zhang2020variational,lan2021policy, zhan2021policy} studied the convergence of projected PG methods and policy mirror descent with direct parameterization, \citet{asadi2017alternative} introduced the so-called mallow parameterization, while \citet{mei2020escaping} studied the escort parameterization.  Part of these parameterizations were proposed in response to the ineffectiveness of softmax parameterization observed in practice.


 
\paragraph{Lower bounds.} Establishing information-theoretic or algorithmic-specific lower bounds on the statistical and computational complexities of RL algorithms --- often achieved by constructing hard MDP instances --- plays an instrumental role in understanding the bottlenecks of RL algorithms. To give a few examples,
\citet{azar2013minimax,domingues2021episodic,li2022settling,yan2022model} developed information-theoretic lower bounds on the sample complexity of 
RL under multiple sampling mechanisms (e.g., sampling with a generative model, online RL, and offline/batch RL), 
\citet{li2021q} established an algorithm-dependent lower bound on the sample complexity of Q-learning, 
whereas \citet{khamaru2020temporal,pananjady2020instance} developed instance-dependent lower bounds for policy evaluation. Additionally, \citet{agarwal2019optimality} constructed a chain-like MDP whose value function under direct parameterization  might contain very flat saddle points under a certain initial state distribution, highlighting the role of distribution mismatch coefficients in policy optimization.  Finally, exponential-time convergence of gradient descent has been observed in other nonconvex problems as well (e.g., \citet{du2017gradient}) despite its asymptotic convergence \citep{lee2016gradient}, although the context and analysis therein are drastically different from what happens in RL settings. 





\subsection{Paper organization}

The rest of this paper is organized as follows. In Section~\ref{sec:backgrounds}, we introduce the basics of Markov decision processes, and describe the softmax policy gradient method along with several key functions/quantities. Section~\ref{sec:mdp-construction} constructs a chain-like MDP,  which is the hard MDP  instance underlying our computational lower bound for PG methods.   In Section~\ref{sec:analysis}, we outline the proof of Theorem~\ref{thm:unregularized-main}, starting with the proof of a weaker version before establishing Theorem~\ref{thm:unregularized-main}. The proof of all technical lemmas are deferred to the appendix. We conclude the paper in Section~\ref{sec:discussion} with a summary of our findings.

\section{Background}
\label{sec:backgrounds}

In this section, we introduce the basics of MDPs, and formally describe the softmax PG method.  
Here and throughout, we denote by $\Delta(\cX)$  the probability simplex over a set $\cX$, and let $|\cX|$ represent the cardinality of the set $\cX$. 
Given two probability distributions $p$ and $q$ over $\cS$, we adopt the notation $\big\| \frac{p}{q} \big\|_{\infty} = \max_{s\in \cS} \frac{p(s)}{q(s)}$ and $\big\| \frac{1}{q} \big\|_{\infty} = \max_{s\in \cS} \frac{1}{q(s)}$. Throughout this paper, 
the notation $f(\mathcal{M})\gtrsim g(\mathcal{M})$ (resp.~$f(\mathcal{M})\lesssim g(\mathcal{M})$) means there exist some universal constants $c>0$ independent of the parameters of the MDP $\mathcal{M}$ such that $f(\mathcal{M})\geq c g(\mathcal{M})$ (resp.~$f(\mathcal{M})\leq c g(\mathcal{M})$), 
while the notation $f(\mathcal{M})\asymp g(\mathcal{M})$ means that $f(\mathcal{M})\gtrsim g(\mathcal{M})$ and $f(\mathcal{M})\lesssim g(\mathcal{M})$ hold simultaneously. 

\paragraph{Infinite-horizon discounted MDP.} 

Let $\mathcal{M} = (\cS, \{\cA_s \}_{s\in \cS}, P, r,\gamma)$ be an infinite-horizon discounted MDP.
Here, $\cS$ represents the state space, $\cA_s$ denotes the action space associated with state $s\in \cS$, $\gamma\in (0,1)$ indicates the discount factor,
$P$ is the probability transition kernel (namely, for each state-action pair $(s,a)$,   $P(\cdot \mymid {s,a})\in \Delta(\cS)$ denotes the transition probability  from state $s$ to the next state when action $a$ is taken),  and $r$ stands for a deterministic reward function (namely, $r(s,a)$ is the immediate reward received in state $s$ upon executing action $a$). Throughout this paper, we assume normalized rewards such that $-1 \leq r(s,a)\leq 1$ for any state-action pair $(s,a)$. 
In addition, we concentrate on the scenario where $\gamma$ is quite close to 1, and often refer to $\frac{1}{1-\gamma}$ as the effective horizon of the MDP.


\paragraph{Policy, value function, Q-function and advantage function.} 

The agent operates by adopting a policy $\pi$, which is a  (randomized) action selection rule based solely on the current state of the MDP. More precisely, for any state $s\in \cS$, we use $\pi(\cdot \mymid s)\in \Delta(\cA_s)$ to specify a probability distribution, with $\pi(a \mymid s)$ denoting the probability of executing action $a\in \cA_s$ when in state $s$. The value function $V^{\pi}: \cS \rightarrow \mathbb{R}$ of a policy $\pi$ --- which indicates the expected discounted cumulative reward induced by policy $\pi$ --- is defined as
\begin{align}
	\label{defn:value-function}
	\forall s\in \cS: \qquad V^{\pi}(s) \coloneqq \mathbb{E}
	\left[ \sum_{k=0}^{\infty} \gamma^k r(s^k,a^k ) \,\big|\, s^0 =s \right] .
\end{align} 
Here, the expectation is taken over the randomness of the MDP trajectory $\{(s^k,a^k)\}_{k\geq 0}$ and the policy, 
where $s^0=s$ and, for all $k\geq 0$, $a^k \sim \pi(\cdot \mymid s^k)$ follows the policy $\pi$ and $s^{k+1}\sim P(\cdot \mymid s^k, a^k)$ is generated by the transition kernel $P$.
Analogously, we shall also define the value function $V^{\pi}(\mu)$ of a policy $\pi$ when the initial state is drawn from a distribution $\mu$ over $\cS$, namely,
	\begin{align}
		\label{defn:V-pi-rho}
		V^\pi(\mu) := \mathbb{E}_{s\sim \mu}\big[ V^{\pi}(s) \big].
	\end{align}
Additionally, the Q-function $Q^{\pi}$ of a policy $\pi$ --- namely, the expected discounted cumulative reward under policy $\pi$ given an initial state-action pair $(s^0,a^0)=(s,a)$ --- is formally defined by 
\begin{equation}
	\label{defn:Q-function}
 	\forall (s,a)\in \cS \times \cA: \qquad Q^{\pi}(s,a) \coloneqq \mathbb{E} \left[ \sum_{k=0}^{\infty} \gamma^k r(s^k,a^k ) \,\big|\, s^0 =s, a^0 = a \right],
\end{equation} 
where the expectation is again over the randomness of the MDP trajectory $\{(s^k,a^k)\}_{k\geq 1}$ when policy $\pi$ is adopted. 
In addition, the advantage function of policy $\pi$ is defined as
\begin{align}
	\label{defn:A-pi-definition}
	A^{\pi}(s,a) \coloneqq Q^{\pi}(s, a) - V^{\pi}(s)
\end{align}
for every state-action pair $(s,a)$.

A major goal is to find a policy that optimizes the value function and the Q-function. 
Throughout this paper, we denote respectively by $V^{\star}$ and $Q^{\star}$ the optimal value function and optimal Q-function, namely,
\begin{align}
	V^{\star}(s) \coloneqq \max_\pi V^{\pi}(s), \qquad  Q^{\star}(s,a) \coloneqq \max_\pi Q^{\pi}(s,a), \qquad \text{for all }s\in \cS\text{ and } a\in \cA_s.
\end{align}

\paragraph{Softmax parameterization and policy gradient methods.} 

The family of policy optimization algorithms attempts to identify optimal policies by resorting to optimization-based algorithms. 
To facilitate differentiable optimization, a widely adopted scheme is to parameterize policies using softmax mappings. Specifically, for any real-valued parameter $\theta=[\theta(s,a)]_{s\in \cS, a\in \cA_s}$, the corresponding softmax policy $\pi_\theta\coloneqq \mathsf{softmax}(\theta)$ is defined such that
	\begin{align} \label{eq:definition_softmax}
		\forall s\in \cS\text{ and }a \in \cA_s: ~~ \pi_\theta(a \mymid s) \coloneqq \frac{ \exp\big(\theta(s,a) \big) }{ \sum_{a'\in \cA_s} \exp\big( \theta(s,a') \big)  } .
	\end{align}
With the aim of maximizing the value function under softmax parameterization, namely, 
\begin{equation} \label{eq:value_max}
	\text{maximize}_{\theta} \quad V^{\pi_\theta}(\mu), 
\end{equation} 
the softmax PG method proceeds by adopting gradient ascent update rules w.r.t.~$\theta$: 
\begin{subequations}
\label{eq:PG-update-all}
 \begin{align} \label{eq:PG-update}
	\theta^{(t+1)} =  \theta^{(t)} + \eta  \nabla_{\theta} {V}^{(t)} (\mu) , \qquad t=0,1,\cdots.
\end{align}
Here and throughout, we let ${V}^{(t)}  = V^{\pi^{(t)}}$ and ${Q}^{(t)}  = Q^{\pi^{(t)}}$ abbreviate respectively the value function and Q-function of the policy iterate $\pi^{(t)}\coloneqq \pi_{\theta^{(t)}}$ in the $t$-th iteration, and $\eta>0$ denotes the stepsize or learning rate. 
Interestingly,  the gradient $ \nabla_{\theta} {V}^{\pi_{\theta}}$ under softmax parameterization \eqref{eq:definition_softmax} admits a closed-form expression \citep{agarwal2019optimality}, that is, for any state-action pair $(s,a)$, 
\begin{equation}
	\label{eq:policy-grad-softmax-expression}
	\frac{\partial {V}^{\pi_{\theta}} (\mu) }{\partial \theta(s,a)} 
	= \frac{1}{1-\gamma} d_{\mu}^{\pi_\theta}(s) \, \pi_\theta(a \mymid s) \, A^{\pi_\theta}(s,a) .
\end{equation}
\end{subequations}
Here, $d_{\mu}^{\pi_\theta}(s)$ represents the {\em discounted state visitation distribution} of a policy $\pi$ given the initial state $s^0\sim \mu$:
	\begin{equation}
		\label{eq:defn-d-rho}
		\forall s\in \cS: \qquad 
		d_{\mu}^{\pi}(s)  
		\coloneqq (1-\gamma) \mathop{\mathbb{E}}_{s^0\sim \mu} \Bigg[ \sum_{k=0}^{\infty} \gamma^k  \mathbb{P}(s^k = s \mymid s^0) \Bigg], 
	\end{equation}
with the expectation taken over the randomness of the MDP trajectory $\{(s^k,a^k)\}_{k\geq 0}$ under the policy $\pi$ and the initial state distribution $\mu$. In words, $d_{\mu}^{\pi}(s)$ measures --- starting from an initial distribution $\mu$ --- how frequently state $s$ will be visited in a properly discounted fashion. 
Throughout this paper, we shall denote ${A}^{(t)}  \coloneqq A^{\pi^{(t)}}$ and $d_{\mu}^{(t)}(s) \coloneqq d_{\mu}^{\pi^{(t)}}(s)$ for notation simplicity.

\section{Construction of a hard MDP}
\label{sec:mdp-construction}

This section constructs a discounted infinite-horizon MDP $\mathcal{M}=\{\cS, \{\cA_s\}_{s\in \cS}, r, P, \gamma\}$, as depicted in Fig.~\ref{fig:MDP},
which forms the basis of the exponential lower bound claimed in this paper. 
In addition to the basic notation already introduced in Section~\ref{sec:backgrounds}, we remark on the action space as follows.  
\begin{itemize}
	\item {\em For each state $s\in \cS$}, we have $\mathcal{A}_s \subseteq \{a_0,a_1, a_2\}$. 
		For convenience of presentation, we allow the action space to vary with $s\in \cS$, but it always comprises no more than 3 actions.  
\end{itemize}

\begin{figure}[t]
\centering
\includegraphics[width=0.75\textwidth]{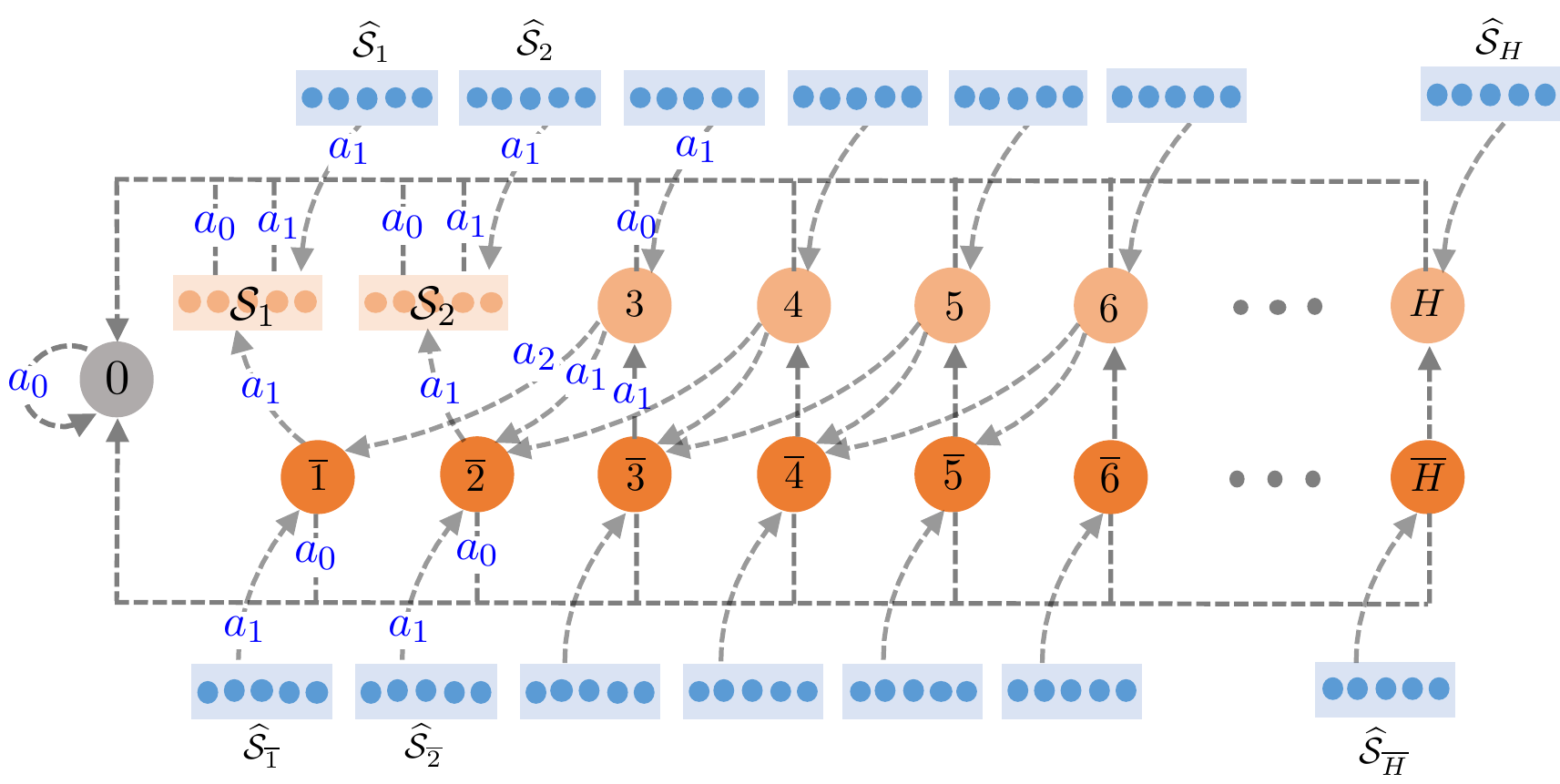}
\caption{An illustration of the constructed MDP. \vspace{-0.05in}} \label{fig:MDP}
\end{figure}

\paragraph{State space partitioning.} 
The states of our MDP exhibit certain group structure. To be precise, we partition the state space $\mathcal{S}$ into a few {\em disjoint} subsets
\begin{align}
	\cS = \{0\} \,\cup\, \sprimary \,\cup\, \sadj \,\cup\,  \mathcal{S}_1 \,\cup\, \mathcal{S}_2 \,\cup
	\, \widehat{S}_1 \, \cup\, \cdots\, \cup\, \widehat{S}_H \,\cup\, \widehat{S}_{\overline{1}}\, \cup\, \cdots \, \cup\, \widehat{S}_{\overline{H}},
	\label{eq:state-partitioning}
\end{align}
which entails:
\begin{itemize}
\item state 0 (an absorbing state);
\item two key ``buffer'' state subsets $\mathcal{S}_1$ and $\mathcal{S}_2$;
\item a set of $H-2$ key primary states $\sprimary \coloneqq \{3, \cdots, H \}$;\footnote{While we do not include states 1 and 2 here, any state in $\cS_1$ (resp.~$\cS_2$) can essentially be viewed as a (replicated) copy of state 1 (resp.~state 2).}
\item a set of $H$ key adjoint states $\sadj \coloneqq \{ \overline{1}, \overline{2}, \cdots, \overline{H} \}$; 
\item $2H$ ``booster'' state subsets $\widehat{\mathcal{S}}_1,  \cdots, \widehat{\cS}_H$, $\widehat{\cS}_{\overline{1}},  \cdots, \widehat{\cS}_{\overline{H}}$.
\end{itemize}
\begin{remark} Our subsequent analysis largely concentrates on the subsets $\cS_1$, $\cS_2$,  $\sprimary$ and $\sadj$. 
	In particular, each state $s\in \{3,\cdots, H\}$ is paired with what we call an adjoint state $\overline{s}$, whose role will be elucidated shortly. 
	In addition, state $\overline{1}$ (resp.~state $\overline{2}$) can be viewed as the adjoint state of the set $\cS_1$ (resp.~$\cS_2$). 
The sets $\sprimary$ and $\sadj$ comprise a total number of $2H-2$ states; in comparison, $\cS_1$ and $\cS_2$ are chosen to be much larger and contain a number of replicated states, a crucial design component that helps ensure the property \eqref{eq:intuition-jump-start} under a uniform initial state distribution.  
As we shall make clear momentarily, 
	the ``booster'' state sets are introduced mainly to help boost the discounted visitation distribution of the states in $\cS_1$, $\cS_2$,  $\sprimary$, and $\sadj$ at the initial stage. 
\end{remark}

We shall also specify below the size of these state subsets as well as some key parameters, where the choices of the quantities $\ch,\cbone,\cbtwo,\cm \asymp 1$ will be made clear in the analysis (cf.~\eqref{eq:assumptions-constants}). 
\begin{itemize}
\item  $H$ is taken to be on the same order as the ``effective horizon'' of this discounted MDP, namely, 
\begin{align}
	H = \frac{\ch}{1-\gamma}.
	\label{eq:H-size}
\end{align}
\item The two buffer state subsets $\cS_1$ and $\cS_2$ have size 
\begin{align}
	|\cS_1| = \cbone (1-\gamma)|\cS| \qquad \text{and}\qquad |\cS_2| = \cbtwo (1-\gamma)|\cS|.
	\label{eq:S1-S2-equal-sizes}
\end{align}

\item The booster state sets are of the same size, namely, 
\begin{align}
	|\widehat{\cS}_1|=  \cdots |\widehat{\cS}_H| = |\widehat{\cS}_{\overline{1}}|=  \cdots=| \widehat{\cS}_{\overline{H}}| = \cm (1-\gamma)|\cS|.
	\label{eq:hatS1-hatSH-equal-sizes}
\end{align}
\end{itemize}


\paragraph{Probability transition kernel and reward function.} 
We now describe the probability transition kernel and the reward function for each state subset. 
Before continuing, we find it helpful to isolate a few key parameters that will be used frequently in our construction:
\begin{subequations}
\label{eq:key-parameters-defn}
\begin{align}
\tau_s &\defn 0.5\gamma^{\frac{2s}{3}},   \label{defn:tau-s}\\
p &\defn \cp (1-\gamma),  \label{defn:p-param}\\
r_s &\defn 0.5 \gamma^{\frac{2s}{3}+\frac{5}{6}},  \label{defn:r-s}
\end{align}
\end{subequations}
where $s\in\{1,2,\cdots,H\}$, 
and $\cp >0$ is some small constant that shall be specified later (see \eqref{eq:assumptions-constants}). 
To facilitate understanding, we shall often treat $\tau_s$ and $r_s$ ($s\leq H$) as quantities that are all fairly close to $0.5$ (which would happen if $\gamma$ is close to 1 and $H=\frac{\ch}{1-\gamma}$ for $\ch$ sufficiently small).  

We are now positioned to make precise descriptions of both $P$ and $r$ as follows.
\begin{itemize}
\item {\em Absorbing state $0$: } singleton action space $\{a_0\}$,
\begin{equation}
	 P(0 \mymid 0, a_0)=1, \qquad\qquad r(0,a_0)=0.
	\label{eq:P-r-state0}
\end{equation}
This is an absorbing state, namely, the MDP will stay in this state permanently once entered. 
As we shall see below, taking action $a_0$ in an arbitrary state will enter state $0$ immediately. 

\item {\em Key primary states $s\in \{ 3,\cdots, H \}$: } action space $\{a_0,a_1, a_2\}$,
\begin{subequations}
\label{eq:P-r-primary-3-H}
\begin{align}
 P(0\mymid s,a_0) &=  1, 			&& r(s,a_0)=r_{s}+\gamma^{2}p\tau_{s-2},\\
 P\big(\,\overline{s-1}\mymid s,a_1\big) &= 1,  	&& r(s,a_1)=0, \\
  P(0\mymid s,a_2) &=  1-p, 			&& r(s,a_2)=r_{s},\\
 P\big(\,\overline{s-2}\mymid s,a_2\big) &= p,
\end{align}
\end{subequations}
where $p$, $\tau_s$ and $r_s$ are all defined in \eqref{eq:key-parameters-defn}.

\item {\em Key adjoint states  $\overline{s} \in \{ \overline{3},\cdots, \overline{H} \}$:} action space $\{a_0, a_1\}$,
\begin{subequations}
\label{eq:P-r-adjoint-states}
\begin{align}
  &P(0 \mymid \overline{s}, a_0) = 1,    && r( \overline{s}, a_0 ) = \gamma \tau_s, \\
  &P(s \mymid \overline{s}, a_1) = 1,    && r( \overline{s}, a_1 ) = 0, 
\end{align}
\end{subequations}
where $\tau_s$ is defined in \eqref{defn:tau-s}.

\item {\em Key buffer state subsets $\cS_1$ and $\cS_2$: } action space $\{a_0,a_1\}$,
\begin{subequations}
\label{eq:P-r-state-S1S2}
\begin{align}
\forall s_{1}\in\mathcal{S}_{1}:\qquad  
  P(0\mymid s_{1}, a_{0}) &=1,    &&r(s_{1},a_{0}) = -\gamma^{2}, \\
  P(0 \mymid s_{1}, a_{1}) &=1,    &&r(s_{1},a_{1}) = \gamma^{2}, &
\end{align}
%
%
\begin{align}
  \forall s_{2}\in\mathcal{S}_{2}:\qquad  
  P(0 \mymid s_{2},a_{0}) &=1,    && r(s_{2},a_{0}) = -\gamma^{4}, \\
  P(0 \mymid s_{2},a_{1}) &=1,    && r(s_{2},a_{1}) = \gamma^{4}. &
\end{align}
\end{subequations}
Given the homogeneity of the states in $\cS_1$ (resp.~$\cS_2$), 
we shall often use the shorthand notation $P(\cdot \mymid 1, a)$ (resp.~$P(\cdot \mymid 2, a)$) to abbreviate $P(\cdot \mymid s_{1}, a)$ (resp.~$P(\cdot \mymid s_{2}, a)$) for any $s_1\in \cS_1$  (resp.~$s_2\in\cS_2$) for the sake of convenience.

\item {\em Other adjoint states $\overline{1}$ and $\overline{2}$:} action space $\{a_0,a_1\}$,
\begin{subequations}
\label{eq:P-r-adjoint-12}
\begin{align}
&P(0\mymid\overline{1},a_0)=  1, &&r(\overline{1}, a_0) =\gamma \tau_1, &P(s_{1}\mymid\overline{1},a_1)=\frac{1}{|\mathcal{S}_{1}|}, ~\forall s_{1}\in\mathcal{S}_{1}, && r(\overline{1}, a_1) = 0, 
 \\
&P(0\mymid\overline{2},a_0)=  1, &&r(\overline{2}, a_0) =\gamma \tau_2, &P(s_{2}\mymid\overline{2},a_1)=\frac{1}{|\mathcal{S}_{2}|}, ~\forall s_{2}\in\mathcal{S}_{2}, && r(\overline{2}, a_1) = 0,
\end{align}
\end{subequations}
where $\tau_1$ and $\tau_2$ are defined in \eqref{defn:tau-s}.

\item {\em Booster state subsets $\widehat{\cS}_1$, $\cdots$, $\widehat{\cS}_H$, $\widehat{\cS}_{\overline{1}}$, $\cdots$, $\widehat{\cS}_{\overline{H}}$:  } singleton action space $\{a_1\}$,
\begin{subequations}
\label{eq:merging-states-definition-P}
\begin{align}
	&\forall s'\in\widehat{\mathcal{S}}_{1}, ~s\in\mathcal{S}_{1}:\qquad P(s\mymid s',a_{1})  = 1/ |\mathcal{S}_{1}|,\\
	&\forall s'\in\widehat{\mathcal{S}}_{2}, ~s\in\mathcal{S}_{2}:\qquad P(s\mymid s',a_{1})  = 1/ |\mathcal{S}_{2}|;
\end{align}
for any $s\in \{3,\cdots, H\}$, 
\begin{align}
\forall s'\in\widehat{\mathcal{S}}_{s},:\qquad & P(s\mymid s',a_{1})=1,
\end{align}
and for any $\overline{s} \in \{\overline{1},\cdots,\overline{H}\}$,
\begin{align}
\forall s'\in\widehat{\mathcal{S}}_{\overline{s}},:\qquad & P(\overline{s}\mymid s',a_{1})=1.
\end{align}
\end{subequations}
The rewards in all these cases are set to be 0 (in fact, they will not even appear in the analysis). 
In addition, any transition probability that has not been specified is equal to zero. 
\end{itemize}

\paragraph{Convenient notation for buffer state subsets $\cS_1$ and $\cS_2$.} 
By construction, it is easily seen that the states in $\cS_1$ (resp.~$\cS_2$) have identical characteristics; in fact, all states in $\cS_1$ (resp.~$\cS_2$) share exactly the same value functions and Q-functions throughout the execution of the softmax PG method.  As a result, we introduce the following convenient notation whenever it is clear from the context: 
\begin{subequations}
\label{eq:convenient-notation-V1-V2-Q1-Q2}
\begin{align}
	V^{\pi}(s_1) \eqqcolon V^{\pi}(1), \qquad Q^{\pi}(s_1,a) \eqqcolon Q^{\pi}(1,a), \qquad 	
	A^{\pi}(s_1) \eqqcolon A^{\pi}(1)  \qquad \text{for all }s_1\in \cS_1; \\
	V^{\pi}(s_2) \eqqcolon V^{\pi}(2), \qquad Q^{\pi}(s_2,a) \eqqcolon Q^{\pi}(2,a), \qquad A^{\pi}(s_2) \eqqcolon A^{\pi}(2)
	\qquad \text{for all }s_2\in \cS_2; \\		
	d^{\pi}_{\mu}(s_1) \eqqcolon d^{\pi}_{\mu}(1), \qquad\quad \pi(a\mymid s_1) \eqqcolon \pi(a\mymid 1),
	\qquad \theta(s_1,a) \eqqcolon \theta(1,a) \qquad \text{for all }s_1\in \cS_1; \\
	d^{\pi}_{\mu}(s_2) \eqqcolon d^{\pi}_{\mu}(2), \qquad\quad
	\pi(a\mymid s_2) \eqqcolon \pi(a\mymid 2),
	\qquad \theta(s_2,a) \eqqcolon \theta(2,a) \qquad \text{for all }s_2\in \cS_2.
\end{align}
\end{subequations}
%

\paragraph{Optimal values and optimal actions of the constructed MDP.} 
Before concluding this section, 
we find it convenient to determine the optimal value functions and the optimal actions of the constructed MDP, 
which would be particularly instrumental when presenting our analysis. 
This is summarized in the lemma below, whose proof  can be found in Appendix~\ref{sec:proof-lemma:basic-properties-MDP-Vstar}.

\begin{lemma}
	\label{lem:basic-properties-MDP-Vstar}
	Suppose that $\gamma^{2H}\geq 2/3$ and $H\geq 2$.
	Then one has 
\begin{subequations}
\begin{align}
	V^{\star}(0)  =0,\qquad V^{\star}(s) & =Q^{\star}(s,a_1) = \gamma^{2s},\quad ~~~ 1\leq s\leq H,\\
	 \qquad\quad\, V^{\star}(\overline{s}) &=Q^{\star}(\overline{s},a_1) =\gamma^{2s+1},\quad1\leq s\leq H,
\end{align}
\end{subequations}
and the optimal policy is to take action $a_1$ in all non-absorbing states. 
In addition, for any policy $\pi$ and any state-action pair $(s,a)$, one has $Q^{\pi}(s,a) \geq - \gamma^2$.
\end{lemma}

Lemma~\ref{lem:basic-properties-MDP-Vstar} tells us that for this MDP, the optimal policy for all non-absorbing states takes a simple form:  sticking to action $a_{1}$. 
In particular, when $\gamma \approx 1$ and $\gamma^H\approx 1$, Lemma~\ref{lem:basic-properties-MDP-Vstar} reveals that the optimal values of all non-absorbing major states are fairly close to 1, namely,  
\begin{align}
	V^{\star}(s)\approx 1 \qquad \text{for all } s\in \{1,\cdots, H\} \cup \{\overline{1},\cdots, \overline{H}\}.
	\label{eq:optimal-values-approx-1}
\end{align}
Additionally, the above lemma directly implies that the Q-function  (and hence the value function) is always bounded below by $-1$, a property that will be used several times in our analysis.

\section{Analysis: proof outline}
\label{sec:analysis}

In this section, we present the main steps for establishing our computational lower bound in  Theorem~\ref{thm:unregularized-main}. 
Before doing so, we find it convenient to start by presenting and proving a weaker version as follows.
\begin{theorem} 
	\label{thm:unregularized}
	Consider the MDP $\mathcal{M}$ constructed in Section~\ref{sec:mdp-construction} (and Fig.~\ref{fig:MDP}). 
	Assume that the softmax PG method adopts a uniform initial state distribution, a uniform policy initialization, and has access to exact gradient computation. Suppose that $0<\eta<(1-\gamma)^2/5$. There exist universal constants  $c_1, c_2, c_3>0$ such that: for any $0.96< \gamma < 1$ and $ |\cS| \ge c_3 (1-\gamma)^{-6}$, one has
\begin{align}
V^{\star}(s)-V^{(t)}(s) > 0.15,\qquad\text{for all primary states } 0.1H < s < H, 
\end{align}
provided that the iteration number satisfies
\begin{align} 
	\label{eq:t-bound}
	t < \frac{c_1}{\eta} |\cS| ^{ 2^{ \frac{c_2}{1-\gamma}}}.
\end{align}
\end{theorem}

In what follows, we shall concentrate on establishing Theorem~\ref{thm:unregularized}, on the basis of the MDP instance constructed in Section~\ref{sec:mdp-construction}. 
Once this theorem is established, we shall revisit Theorem~\ref{thm:unregularized-main} (towards the end of Section~\ref{sec:proof-outline-weak-thm}) and describe how the proof of Theorem~\ref{thm:unregularized} can be adapted to prove Theorem~\ref{thm:unregularized-main}. 


\subsection{Preparation: crossing times and choice of constants}

\paragraph{Crossing times.} 
To investigate how long it takes for softmax PG methods to converge to the optimal policy, 
 we shall pay particular attention to a family of key quantities:    
the number of iterations needed for  $V^{(t)}(s)$ to surpass a prescribed threshold $\tau$ ($\tau<1$) {\em before} it reaches its optimal value. 
To be precise,  for each $s \in \{3,\cdots,H\} \cup \{\overline{1},\cdots,\overline{H}\}$ and any given threshold $\tau>0$, 
we introduce the following {\em crossing time}:
\begin{align}
	\label{eqn:approx-optimal-time}
	t_s(\tau) := \arg\min \big\{t \mymid V^{(t)}(s) \geq \tau \big\}.	
\end{align}
When it comes to the buffer state subsets $\cS_1$ and $\cS_2$, we define the crossing times analogously as follows
\begin{align}
\label{eqn:approx-optimal-time-buffer}
	t_1(\tau) \coloneqq \arg\min \big\{t \mymid V^{(t)}(1) \geq \tau \big\}
	\qquad \text{and}\qquad 
	t_2(\tau) \coloneqq \arg\min \big\{t \mymid V^{(t)}(2) \geq \tau \big\},
\end{align}
where we recall the notation $V^{(t)}(1)$ and $V^{(t)}(2)$ introduced in \eqref{eq:convenient-notation-V1-V2-Q1-Q2}.

\paragraph{Monotonicity of crossing times.}
Recalling the definition \eqref{eqn:approx-optimal-time} of the crossing time $t_s(\cdot)$, we know that 
\begin{align}
	V^{(t)}(s) < \taun_s \qquad \text{for all }t< t_s(\taun_s),
	\label{eq:V-t-sprime-3-H-s}
\end{align}
with $\tau_{s}$ defined in expression~\eqref{defn:tau-s}. We immediately make note of the following crucial monotonicity property that will be justified later in Remark~\ref{remark:monotonicity-crossing-times}:
\begin{align}
	t_2(\taun_2)\leq t_3(\taun_3)\leq  \cdots \leq t_H(\taun_H). 
	\label{eq:monotonicity-t-s-tau-s}
\end{align}
It will also be shown in Lemma~\ref{lem:init-step} that $t_1(\taun_1)\leq  t_2(\taun_2)$ when the constants $\cbone, \cbtwo$ and $\cm$ are properly chosen. 
\begin{remark}
As we shall see shortly (i.e., Part (iii) of Lemma~\ref{lem:basic-properties-MDP-pi}), one has $t_{\overline{s}}(\gamma \tau_s) = t_{s}(\tau_s)$ for any $\overline{s}\in \{\overline{1},\cdots,\overline{H}\}$, which combined with \eqref{eq:monotonicity-t-s-tau-s} leads to
\begin{align}
	t_1(\taun_1)=t_{\overline{1}}(\gamma \tau_1)\leq t_2(\taun_2)=t_{\overline{1}}(\gamma \tau_1)\leq t_3(\taun_3)=t_{\overline{3}}(\gamma \tau_3)\leq  \cdots \leq t_H(\taun_H)=t_{\overline{H}}(\gamma \tau_H). 
	\label{eq:monotonicity-t-s-sbar}
\end{align}
\end{remark}
%


\paragraph{Choice of parameters.} We assume the following choice of parameters throughout the proof:
\begin{align}
\label{eq:assumptions-constants}
\gamma > 0.96, \; \cm<1, \;  \ch <&~0.19, \;  \eta < \frac{(1-\gamma)^2}{5}, \; 
\frac{\cbone}{\cm} \le ~\frac{1}{79776}, \; 8\le \frac{\cbtwo}{\cm} \le15 , \; \cp <\frac{1}{2016}.
\end{align}
%
In the sequel, we outline the key steps that underlie the proof of our main results, with the proofs of the key lemmas postponed to the appendix.

\subsection{A high-level picture}

 While our proof is highly technical, it is prudent to point out some key features that help paint a high-level picture about the slow convergence of the algorithm. 
 Recall that $a_1$ is the optimal action in the constructed MDP. 
 The chain-like structure of our MDP underscores a sort of {\em sequential dependency}:  the dynamic of any primary state $s\in \{3,\cdots,H\}$ depends heavily on what happens in those states prior to $s$ --- particularly state $s-1$, state $s-2$ as well as the associated adjoint states. By carefully designing the immediate rewards, we can ensure that for any $s\in \{3,\cdots,H\}$, the iterate $\pi^{(t)}(a_1\mymid s)$ corresponding to the optimal action $a_1$ keeps decreasing before  $\pi^{(t)}(a_1\mymid s-2)$ gets reasonably close to 1. As illustrated in Figure~\ref{fig:pi-trend}, this feature implies that the time taken for $\pi^{(t)}(a_1\mymid s)$ to get close to 1 grows (at least) geometrically as $s$ increases, as will be formalized in \eqref{eq:blow-up-illustration}.

\begin{figure}[ht]
\centering
\includegraphics[width=0.6\textwidth]{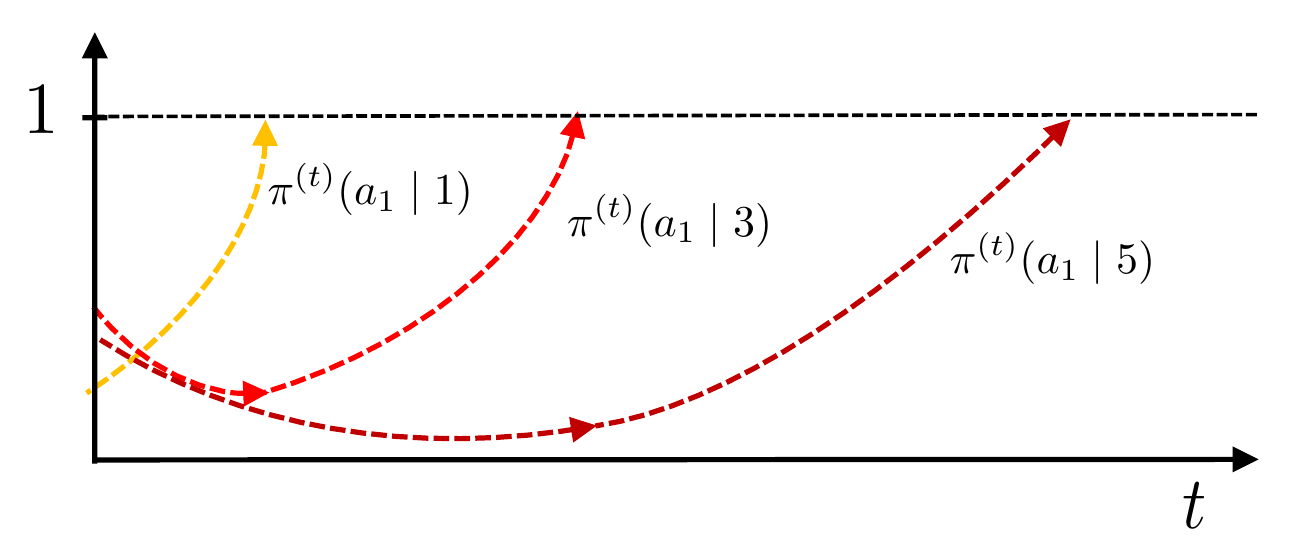}
\caption{An illustration of the dynamics of $\pi^{(t)}(a_1 \mymid s)$ vs.~the iteration count $t$. 
	The yellow line, the middle red line, and the dark red line illustrate the dynamics of $\pi^{(t)}(a_1 \mymid 1)$, 
	$\pi^{(t)}(a_1 \mymid 3)$ and $\pi^{(t)}(a_1 \mymid 5)$, respectively. 
	\vspace{-0.05in}} \label{fig:pi-trend}
\end{figure}

Furthermore, we summarize below the typical dynamics of the iterates $\theta^{(t)}(s,a)$ before they converge, which are helpful for the reader to understand the proof. We start with the key buffer state sets $\cS_1$ and $\cS_2$, which are the easiest to describe. 
\medskip

\begin{tcolorbox}
	\textbf{Dynamics of $\theta^{(t)}(s,a)$ (for key buffer state sets $\cS_1$ and $\cS_2$):} 
\begin{enumerate}
	\item Initialization: $\theta^{(0)}(1,a_0) = \theta^{(0)}(1,a_1)  = 0$ and 
		$\theta^{(0)}(2,a_0) = \theta^{(0)}(2,a_1)  = 0$
	\item All iterations (Lemma~\ref{lem:init-step}): 
	\begin{itemize}
	\item $\theta^{(t)}(1,a_1)$ and $\theta^{(t)}(2,a_1)$ keep increasing and remains the largest
	\item $\theta^{(t)}(1,a_0)$ and $\theta^{(t)}(2,a_0)$ keep decreasing and remains the smallest
	
	\end{itemize}
\end{enumerate}
\end{tcolorbox}

Next, the dynamics of $\theta^{(t)}(s,a)$ for the key primary states $3\leq s \leq H$ are much more complicated, and rely heavily on the status of several prior states $s-1$, $s-2$ and $\overline{s-1}$.  This motivates us to divide the dynamics into several stages based on the crossing times of these prior states,  which are illustrated in Figure~\ref{fig:dynamics-theta} as well.  Here, we remind the reader of the definition of $\tau_s$ in \eqref{eq:key-parameters-defn}. 

\medskip

\begin{tcolorbox}
\textbf{Dynamics of $\theta^{(t)}(s,a)$ (for key primary states $3\leq s \leq H$):} 
\begin{enumerate}
	\item Initialization: $\theta^{(0)}(s,a_0) = \theta^{(0)}(s,a_1) = \theta^{(0)}(s,a_2) = 0$
	\item Initial stage: $t< t_{s-2}(\tau_{s-2})$ (Lemma~\ref{lem:induc-theta-t-s-2})
	\begin{itemize}
	\item $\theta^{(t)}(s,a_1)$ keeps decreasing and remains the smallest
	\item $\theta^{(t)}(s,a_0)$ keeps increasing and remains the largest
	\item $\theta^{(t)}(s,a_2)$ keeps increasing
	\end{itemize}
\item Intermediate stage: $t_{s-2}(\tau_{s-2}) \leq t\le t_{\overline{s-1}}(\tau_{s})$ (Lemma~\ref{lem:inter})
	\begin{itemize}
	\item $\theta^{(t)}(s,a_1)$ keeps decreasing and remains the smallest
	\item $\theta^{(t)}(s,a_2)$ keeps increasing
	\end{itemize}
	\item Final stage (part 1): $t_{\overline{s-1}}(\tau_{s}) < t < \tprime$ (Lemma~\ref{lem:grand-final-victory})
	\begin{itemize}
	\item $\theta^{(t)}(s,a_1)$ increases a little
	\item $\theta^{(t)}(s,a_0)$ keeps decreasing and approaches $\theta^{(t)}(s,a_1)$
	\item $\theta^{(t)}(s,a_2)$ keeps increasing and becomes the largest
	\end{itemize}
	\item Final stage (part 2): $t \ge \tprime$ (Lemma~\ref{lem:grand-final-victory})
	\begin{itemize}
	\item $\theta^{(t)}(s,a_1)$ keeps increasing and becomes the largest
	\item $\theta^{(t)}(s,a_2)$ decreases a lot
	\end{itemize}
\end{enumerate}
\end{tcolorbox}

\begin{figure}
	\begin{center}
	\includegraphics[width=0.6\textwidth]{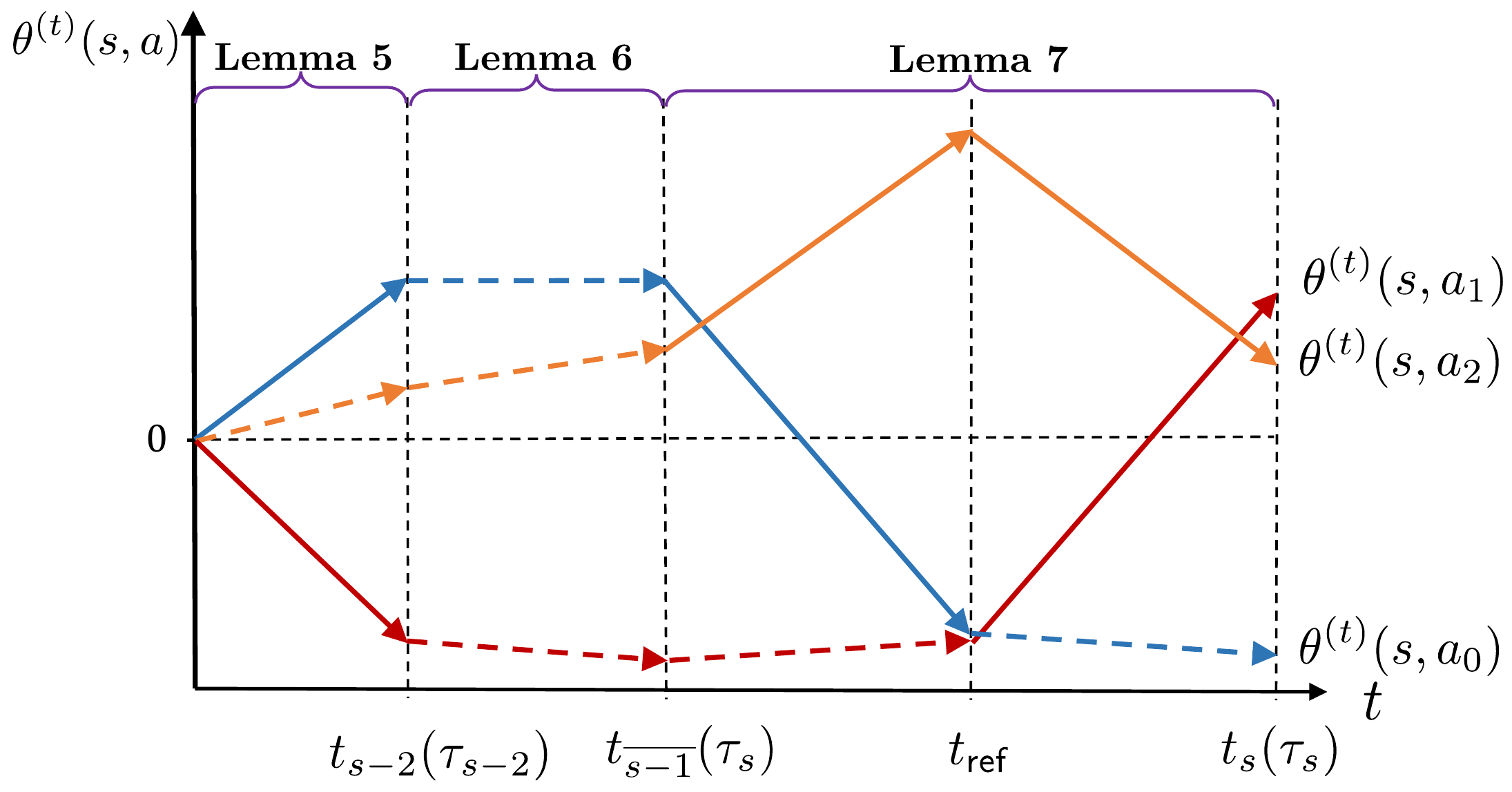}
	\end{center}
	\caption{An illustration of the dynamics of $\{\theta^{(t)}(s,a)\}_{a\in \{a_0,a_1,a_2\}}$ vs.~iteration number $t$. 
	The blue, red and yellow lines represent the dynamics of $\theta^{(t)}(s,a_0)$, $\theta^{(t)}(s,a_1)$ and $\theta^{(t)}(s,a_2)$, respectively.   
	Here, we use solid lines to emphasize the time ranges for which the dynamics of $\theta^{(t)}(s,a)$ play the most crucial roles in our lower bound analysis. \label{fig:dynamics-theta}}
\end{figure}



\subsection{Proof outline}
\label{sec:proof-outline-weak-thm}

We are now in a position to outline the main steps of the proof of Theorem~\ref{thm:unregularized-main} and Theorem~\ref{thm:unregularized}, with details deferred to the appendix. 
In the following, Steps 1-6 are devoted to analyzing the dynamics of softmax PG methods when applied to the constructed MDP~$\mathcal{M}$, which in turn establish Theorem~\ref{thm:unregularized}.  Step 7 describes how these can be easily adapated to prove Theorem~\ref{thm:unregularized-main}, by slightly modifying the MDP construction.

\subsection*{Step 1: bounding the discounted state visitation distributions}

In view of the PG update rule \eqref{eq:PG-update-all}, 
the size of the policy gradient relies heavily on the discounted state visitation distribution $d_{\mu}^{(t)}(s)$. 
In light of this observation, this step aims to quantify the magnitudes of $d_{\mu}^{(t)}(s)$, for which we start 
with several universal lower bounds regardless of the policy in use.
\begin{lemma}
\label{lem:facts-d-pi-s-LB}
For any policy $\pi$,  the following lower bounds hold true: 
\begin{subequations}
\begin{align}
	d_{\mu}^{\pi}(s) & \geq\cm\gamma(1-\gamma)^{2}, &&\text{if }s\in\{3,\cdots,H\}, 
		\label{eq:dmu-t-s-order-range-LB}\\
	d_{\mu}^{\pi}(\overline{s}) & \geq\cm\gamma(1-\gamma)^{2}, &&\text{if }\overline{s}\in\{\overline{1},\cdots,\overline{H}\}, 
		\label{eq:dmu-t-s-bar-order-range-LB}\\
	d_{\mu}^{\pi}(1) & \geq  \frac{\cm\gamma(1-\gamma)^2}{|\mathcal{S}_1|} = \gamma(1-\gamma)\frac{\cm}{\cbone}\cdot\frac{1}{|\mathcal{S}|} , 
		\label{eq:dmu-t-s-12-order-range-LB} \\
	d_{\mu}^{\pi}(2) & \geq  \frac{\cm\gamma(1-\gamma)^2}{|\mathcal{S}_2|} = \gamma(1-\gamma)\frac{\cm}{\cbtwo}\cdot\frac{1}{|\mathcal{S}|} .
		\label{eq:dmu-t-s-12-order-range-LB-S2}
\end{align}
\end{subequations}
\end{lemma}
As it turns out, the above lower bounds are order-wise tight estimates prior to certain crucial crossing times. 
This is formalized in the following lemma, where we recall the definition of $\tau_s$ in \eqref{eq:key-parameters-defn}. 
\begin{lemma}
\label{lem:facts-d-pi-s}

Under the assumption \eqref{eq:assumptions-constants}, the following results hold:
\begin{subequations}
\begin{align}
	\forall 3 \le s \le H, \; t \leq t_s(\taun_s): \qquad d_{\mu}^{(t)}(s) & \leq 14 \cm (1-\gamma)^{2}, \label{eq:dmu-t-s-order-range} \\
	\forall 2 \le s \le H, \; t \leq t_s(\taun_s): \qquad d_{\mu}^{(t)}(\overline{s}) & \leq 14 \cm (1-\gamma)^{2}, \label{eq:dmu-t-s-order-range-bar}\\ 
	\forall  t \leq t_2(\taun_2): \qquad d_{\mu}^{(t)}(2) & \leq\frac{1-\gamma}{|\cS|}\left(1+\frac{8\cm}{\cbtwo}\right), \label{eq:dmu-t-s-2-order-range} \\
	\forall t \leq t_2(\taun_2) : \qquad d_{\mu}^{(t)}(\overline{1}) &\leq 14 \cm (1-\gamma)^{2}, \label{eq:dmu-t-s-1-order-range-bar} \\
	\forall t \leq \min\{t_1(\taun_1), t_2(\taun_2)\}: \qquad d_{\mu}^{(t)}(1) & \leq\frac{1-\gamma}{|\cS|}\left(1+\frac{17\cm}{\cbone}\right).\label{eq:dmu-t-s-1-order-range}
\end{align}
\end{subequations}
 
%
\end{lemma}
\begin{remark}
	As will be demonstrated in Lemma~\ref{lem:init-step}, one has $t_1(\taun_1)\leq  t_2(\taun_2)$ for properly chosen constants $\cbone, \cbtwo$ and $\cm$. 
	Therefore, we shall bear in mind that the properties \eqref{eq:dmu-t-s-1-order-range-bar} and \eqref{eq:dmu-t-s-1-order-range} hold for any $t \leq t_1(\taun_1)$. 
\end{remark}
The proofs of these two lemmas are deferred to Appendix~\ref{sec:estimate-state-visitation}. The sets of booster states, whose cardinality is controlled by $\cm$, play an important role in sandwiching the initial distribution of the states in $\cS_1$, $\cS_2$,  $\sprimary$, and $\sadj$. 
Combining these bounds, we uncover the following properties happening before $V^{(t)}(s)$ exceeds $\taun_s$:
\begin{itemize}
	\item For any key primary state $s\in \{3,\cdots,H\} $ or any adjoint state $s\in \{\overline{1},\cdots, \overline{H}\} $, one has
		\[
			d_{\mu}^{(t)}(s) \asymp (1-\gamma)^2. 
		\]
	\item For any state $s$ contained in the buffer state subsets $\cS_1$ and $\cS_2$, we have
		\[
			d_{\mu}^{(t)}(1) \asymp \frac{(1-\gamma)^2}{|\cS_1|} 
			\qquad \text{and} \qquad
			d_{\mu}^{(t)}(2) \asymp \frac{(1-\gamma)^2}{|\cS_2|},
		\]
		where we recall the size of $\cS_1$ and $\cS_2$ in \eqref{eq:S1-S2-equal-sizes}. 
		In other words, the discounted state visitation probability of any buffer state is substantially smaller than 
		that of any key primary state $3,\cdots,H$ or adjoint state.  
		In principle,  the size of each buffer state subset plays a crucial role in determining the associated $d_{\mu}^{(t)}(s)$ 
		--- the larger the size of the buffer state subset, the smaller the resulting state visitation probability.

	\item Further, the aggregate discounted state visitation probability of the above states is no more than the order of
		\[
			(1-\gamma)^2  \cdot H \asymp 1-\gamma = o(1),
		\]
		which is vanishingly small. 
		In fact, state 0 and the booster states account for the dominant fraction of state visitations 
		at the initial stage of the algorithm. 
\end{itemize}

\subsection*{Step 2: characterizing the crossing times for the first few states ($\cS_1$, $\cS_2$, and $\overline{1}$)}

Armed with the bounds on $d_{\mu}^{(t)}$ developed in Step 1, we can move forward to study the crossing times for the key states. 
In this step, we pay attention to the crossing times for the buffer states $\cS_1, \cS_2$ as well as the first adjoint state $\overline{1}$,
which forms a crucial starting point towards understanding the behavior of the subsequent states. 
Specifically, the following lemma develops lower and upper bounds regarding these quantities,  
whose proof can be found in Appendix~\ref{sec:proof-lem:init-step}. 
\begin{lemma}
	\label{lem:init-step}
	Suppose that \eqref{eq:assumptions-constants} holds. 
	If $|\cS| \geq 1/(1-\gamma)^2$, then the crossing times satisfy
	\begin{subequations}
	\begin{align}
		\label{eqn:t1-t2-scaling} 
		\frac{\log3}{1+17\cm/\cbone} \, \frac{|\cS|}{\eta}  \,\leq\,  t_{1}(\taun_1)  \,\leq\,  t_{1}\big(\gamma^{2}-1/4\big)
		\,\leq\,   t_{2}(\taun_2) \,\leq\,  t_{2}\big(\gamma^{4}-1/4\big)  \,\leq\,  \frac{15\cbtwo}{\cm} \, \frac{|\cS|}{\eta}.
	\end{align}
	In addition, if $|\cS| \geq \frac{320\gamma^3}{\cm(1-\gamma)^{2}}$, then one has
	\begin{align}
		\label{eqn:t1-t2-bar-scaling}
		t_{2}(\taun_2) > t_{\overline{1}}\big(\gamma^3-1/4\big).
	\end{align}
	\end{subequations}
\end{lemma}

For properly chosen constants $\cbone$, $\cbtwo$ and $\cm$, Lemma~\ref{lem:init-step} delivers the following important messages: 
\begin{itemize}
	\item The cross times of these first few states are already fairly large; for instance,
	\begin{equation}
		t_1(\tau_1) \,\asymp\, t_2(\tau_2)  \,\asymp\, \frac{ |\cS| }{ \eta }, 
	\end{equation}
	which scale linearly with the state space dimension. 
	As we shall see momentarily, while $t_1(\tau_1)$ and $t_2(\tau_2)$ remain polynomially large, 
	these play a pivotal role in ensuring rapid explosion of the crossing times of the states that follow (namely, the states $\{3,\cdots,H\}$).

	\item We can guarantee a strict ordering such that  the crossing time of state 2 is at least as large as that of both state 1 and state $\overline{1}$. This property is helpful as well for subsequent analysis.

\end{itemize}


\subsection*{Step 3: understanding the dynamics of $\theta^{(t)}(s, a)$ before $t_{s-2}(\taun_{s-2})$ }

With the above characterization of the crossing times for the first few states, 
we are ready to investigate the dynamics of $\theta^{(t)}(s, a)$ ($3 \le s \le H$) at the initial stage, 
that is, the duration prior to the threshold $t_{s-2}(\taun_{s-2})$. 
Our finding for this stage is summarized in the following lemma, with the proof deferred to Appendix \ref{sec:analysis-parameter-crossing}.

\begin{lemma}
\label{lem:induc-theta-t-s-2}
Suppose that \eqref{eq:assumptions-constants} holds. 
For any $3 \le s \leq H$ and any $0 \le t \le t_{s-2}(\taun_{s-2})$, one has
\begin{align}
	& \theta^{(t)}(s, a_1) \leq -\frac{1}{2}\log\Big(1+\frac{\cm\gamma}{35} \eta(1-\gamma)^2t \Big)  \label{eqn:theta-t-relation-a1} \\
	& \quad  \text{and} \qquad \theta^{(t)}(s, a_0) \geq \theta^{(t)}(s, a_2) \geq 0. \label{eqn:theta-a2-positive}
\end{align}
%
\end{lemma}
Lemma~\ref{lem:induc-theta-t-s-2} makes clear the behavior of $\theta^{(t)}(s, a)$ during this initial stage:   
\begin{itemize}

	\item The iterate $\theta^{(t)}(s, a_1)$ associated with the optimal action $a_1$ keeps dropping at a rate of $\log\big(O(\frac{1}{\sqrt{t}})\big)$, and remains the smallest compared to the ones with other actions (since $\theta^{(t)}(s, a_1) \leq 0 \leq \theta^{(t)}(s, a_2) \leq \theta^{(t)}(s, a_0) $).

	\item The other two iterates $\theta^{(t)}(s, a_0)$ and $\theta^{(t)}(s, a_2)$ stay non-negative throughout this stage, with $a_0$ being perceived as more favorable than the other two actions. 

	\item In fact, a closer inspection of the proof in Appendix \ref{sec:analysis-parameter-crossing} reveals that  $\theta^{(t)}(s, a_2)$ remains increasing --- even though at a rate slower than that of $\theta^{(t)}(s, a_0)$ --- throughout this stage (see \eqref{induc-intermediate-qa2-large} and the gradient expression \eqref{eq:policy-grad-softmax-expression}).

\end{itemize}
\noindent In particular, around the threshold  $t_{s-2}(\taun_{s-2})$, the iterate $\theta^{(t)}(s, a_1)$ becomes as small as
\[
	\exp \Big(\theta^{(t)}(s, a_1) \Big) \leq O\Bigg( \frac{1}{ \sqrt{ \eta(1-\gamma)^2 t_{s-2}(\taun_{s-2}) } } \Bigg). 
\]
In fact, an inspection of the proof of this lemma reveals that
\[
	\pi^{(t)}(a_1 \mymid s) \leq O\Bigg( \frac{1}{  \eta(1-\gamma)^2 t_{s-2}(\taun_{s-2}) }  \Bigg) \qquad \text{when }t= t_{s-2}(\taun_{s-2}). 
\]
This means that  $\pi^{(t)}(a_1 \mymid s)$ becomes smaller for a larger $t_{s-2}(\taun_{s-2})$, making it more difficult to return/converge to 1 afterward.

\subsection*{Step 4: understanding the dynamics of $\theta^{(t)}(s, a)$ between $t_{s-2}(\taun_{s-2})$ and $t_{\overline{s-1}}(\tau_{s})$}

Next, we investigate, for any $3\leq s\leq H$, the behavior of the iterates during an ``intermediate'' stage, 
namely, the duration when the iteration count $t$ is between $t_{s-2}(\taun_{s-2})$ and $t_{\overline{s-1}}(\tau_{s})$. 
This is summarized in the following lemma, whose proof can be found in Appendix~\ref{sec:analysis-lem:inter}. 
\begin{lemma}
\label{lem:inter}
Consider any $3 \le s \le H$. 
Assume that \eqref{eq:assumptions-constants} holds. 
Suppose that
\begin{subequations}
\label{eqn:assmp-all-intermediate}
\begin{align}
	t_{s-1}(\taun_{s-1}) &> t_{\overline{s-2}}(\tau_{s-1}) + \frac{2444s}{\cm\gamma\eta(1-\gamma)^2}, \label{eqn:assmp-all-intermediate-s} \\
	t_{3}(\taun_3) &> t_2(\gamma^4-1/4).
\end{align}
\end{subequations}
Then one has
\begin{align}\label{eqn:phantom}
	\theta^{(t_{\overline{s-1}}(\tau_{s}))}(s, a_1) &\le \theta^{(t_{s-2}(\taun_{s-2}))}(s, a_1) 	\qquad \text{and} \qquad   \theta^{(t_{\overline{s-1}}(\tau_{s}))}(s, a_2)  \ge 0.    
\end{align} 
In particular, when $s = 3$, the results in \eqref{eqn:phantom} hold true without requiring the assumption~\eqref{eqn:assmp-all-intermediate}.
\end{lemma}
\begin{remark}
	Condition \eqref{eqn:assmp-all-intermediate-s} only requires $t_{s-1}(\taun_{s-1})$ to be slightly larger than $t_{\overline{s-2}}(\tau_{s-1})$, which will be justified using an induction argument when proving the main theorem. 
\end{remark}
As revealed by the claim \eqref{eqn:phantom} of Lemma~\ref{lem:inter},  
the iterate $\theta^{(t)}(s, a_2)$ remains sufficiently large during this intermediate stage.
In the meantime, Lemma~\ref{lem:inter}  guarantees that during this stage, $\theta^{(t)}(s, a_1)$  
lies below the level of $\theta^{(t_{s-2}(\taun_{s-2}))}(s, a_1)$ that has been pinned down in Lemma~\ref{lem:induc-theta-t-s-2} (which has been shown to be quite small). Both of these properties make clear that the iterates $\theta^{(t)}(s, a) $ remain far from optimal at the end of this intermediate stage.

\subsection*{Step 5: establishing a blowing-up phenomenon }

The next lemma, which plays a pivotal role in developing the desired exponential convergence lower bound, demonstrates that the cross times explode at a super fast rate. The proof is postponed to Appendix~\ref{sec:proof-lem-grand-final}. 
\begin{lemma}
\label{lem:grand-final-victory} 
Consider any $3 \le s \le H$. 
Suppose that \eqref{eq:assumptions-constants} holds and
\begin{align}
	t_{s-2}(\taun_{s-2}) \geq 	
	\Big( \frac{6300 e}{\cp(1-\gamma)} \Big)^4\frac{1}{\frac{\cm\gamma}{35} \eta(1-\gamma)^2}.
	\label{eq:lower-bound-t-sminus2-Stage1}
\end{align}
%
%
Then there exists a time instance $\tprime$ obeying $t_{\overline{s-1}}(\tau_s) \le \tprime < t_{s}(\taun_s)$ such that 
\begin{subequations}
\begin{align}
\label{eqn:jay-concert-1}
\theta^{(\tprime)}(s, a_0) \le&~\theta^{(\tprime)}(s, a_1) - \log\Big(\frac{\cp}{16128}(1-\gamma)\Big), \\
\label{eqn:jay-concert-2}
	\theta^{(\tprime)}(s, a_1) \le&~-\frac{1}{2}\log\Big(1+\frac{\cm\gamma}{35} \eta(1-\gamma)^2t_{s-2}(\taun_{s-2}) \Big)+1, 
\end{align}
and at the same time, 
\begin{align}
	\label{eqn:jay-concert-3}
	t_{s}(\taun_s) - \tprime \ge 10^{-10}\cp\cm^{0.5}\eta^{0.5}(1-\gamma)^2\Big(t_{s-2}(\taun_{s-2}) \Big)^{1.5}.
\end{align}
\end{subequations}
\end{lemma}

The most important message of Lemma~\ref{lem:grand-final-victory} lies in  property~\eqref{eqn:jay-concert-3}.
In a nutshell, this property uncovers that the crossing time $t_{s}(\taun_s)$ is substantially larger than $t_{s-2}(\taun_{s-2})$, namely,
\begin{equation}\label{eq:blow-up-illustration}
	t_{s}(\taun_s)  \gtrsim \eta^{0.5}(1-\gamma)^2\Big(t_{s-2}(\taun_{s-2}) \Big)^{1.5}, 
\end{equation}
thus leading to explosion at a super-linear rate.  By contrast, the other two properties unveil some important features happening between $\tlow$ and $t_{s}(\taun_s)$ that in turn lead to property~\eqref{eqn:jay-concert-3}. 
In words, property \eqref{eqn:jay-concert-1} requires $\theta^{(\tprime)}(s, a_0)$ to be not much larger than $\theta^{(\tprime)}(s, a_1)$;   
property \eqref{eqn:jay-concert-2} indicates that: when $t_{s-2}(\taun_{s-2})$ is large, 
 both $\theta^{(\tprime)}(s, a_1)$ and $\theta^{(\tprime)}(s, a_0)$ are fairly small, with $\theta^{(\tprime)}(s, a_2)$ being the dominant one (due to the fact $\sum_a \theta^{(\tprime)}(s, a) = 0$ as will be seen in Part  (vii) of Lemma~\ref{lem:basic-properties-MDP-pi}).

The reader might naturally wonder what the above results imply about $\pi^{(\tprime)}(a_1 \mymid s)$ (as opposed to $\theta^{(\tprime)}(s, a_1 )$). Towards this end,  
we make the observation that
 \begin{align}
	 \pi^{(\tprime)}(a_1 \mymid s) &= \frac{\exp\left( \theta^{(\tprime)}(s,a_{1})\right)}{\sum_{a}\exp\left(\theta^{(\tprime)}(s,a)\right)} 
	 \le \frac{\exp\big(\theta^{(\tprime)}(s,a_{1})\big)}{\exp\big(\theta^{(\tprime)}(s,a_{2})\big)} 
	 \overset{(\mathrm{i})}{=} \exp\big(2\theta^{(\tprime)}(s,a_{1}) + \theta^{(\tprime)}(s,a_{0}) \big) \notag \\
	 & \overset{(\mathrm{ii})}{\lesssim}  \frac{1}{(1-\gamma)\Big(  \eta(1-\gamma)^2t_{s-2}(\taun_{s-2}) \Big)^{1.5}}
	 \asymp \frac{1}{\eta^{1.5} (1-\gamma)^4\big(   t_{s-2}(\taun_{s-2}) \big)^{1.5}} ,
 \end{align}
where (i) holds true since $\sum_a \theta^{(\tprime)}(s,a)=0$ (a well-known property of policy gradient methods as recorded in Lemma~\ref{lem:basic-properties-MDP-pi}(vii)), 
and (ii) follows from the properties \eqref{eqn:jay-concert-1} and \eqref{eqn:jay-concert-2}.
In other words, $\pi^{(\tprime)}(s,a_{1})$ is inversely proportional to $\big( t_{s-2}(\taun_{s-2}) \big)^{3/2}$. 
As we shall see, the time taken for $\pi^{(\tprime)}(a_1 \mymid s)$ to converge to 1 is proportional to the inverse policy iterate $\big(\pi^{(t)}(s,a_{1})\big)^{-1}$, meaning that it is expected to take an order of $\big(t_{s-2}(\taun_{s-2})\big)^{3/2}$ iterations to increase from $\pi^{(\tprime)}(s,a_{1})$ to 1.


\subsection*{Step 6: putting all this together to establish Theorem~\ref{thm:unregularized}}

With the above steps in place, we are ready to combine them to establish the following result. 
As can be easily seen, Theorem~\ref{thm:unregularized} is an immediate consequence of Theorem~\ref{thm:is-it-the-final-result}. 
\begin{theorem} 
\label{thm:is-it-the-final-result}
Suppose that \eqref{eq:assumptions-constants} holds. There exist some universal constants $c_1, c_2, c_3 > 0$ such that 
\begin{align} \label{eq:ts-bound}
t_{s}(0.5) \ge c_1\frac{|\cS|^{\frac{2}{3}}}{\eta}\Big(c_2|\cS| \Big)^{\frac{1}{3}\cdot 1.5^{\lfloor s/2\rfloor}},
\end{align}
provided that
\begin{align}
|\cS| \ge&~\frac{c_3}{(1-\gamma)^6}.
	\label{eq:assumption-S-card-LB}
\end{align}
\end{theorem}
%


\begin{proof}[Proof of Theorem~\ref{thm:is-it-the-final-result}]
Let us define two universal constants $C_1 \defn \frac{\log 3}{1+17\cm/\cbone}$ and $C_2 \defn \frac{10^{-20}\cp^2\cm\log 3}{1+17\cm/\cbone}$.
We claim that if one can show that
\begin{align}
\label{eqn:american-quartet}
t_{s}(\taun_s) \ge C_1\frac{|\cS|}{\eta}\Big(C_2(1-\gamma)^4|\cS| \Big)^{1.5^{\lfloor (s-1)/2\rfloor}-1},
\end{align}
then the desired bound~\eqref{eq:ts-bound} holds true directly. 
In order to see this, recall that $\tau_{s} \leq 1/2$ by definition, and therefore, 
\begin{align*}
t_{s}(0.5) \ge t_{s}(\taun_s) 
\stackrel{(\mathrm{i})}{\ge} C_1\frac{|\cS|}{\eta}\Big(C_2\sqrt[3]{|\cS|} \Big)^{1.5^{\lfloor (s-1)/2\rfloor}-1} 
\stackrel{(\mathrm{ii})}{\ge} 
c_1\frac{|\cS|^{\frac{2}{3}}}{\eta}\Big(c_2|\cS| \Big)^{\frac{1}{3}\cdot 1.5^{\lfloor s/2\rfloor}}.
\end{align*}
Here, (i) follows from \eqref{eqn:american-quartet} in conjunction with the assumption \eqref{eq:assumption-S-card-LB}, whereas (ii) holds true by setting $c_1 = C_1/C_2$ and $c_2 = C_2^3$.

It is then sufficient to prove the inequality~\eqref{eqn:american-quartet}, towards which we shall resort to mathematical induction in conjunction with the following induction hypothesis
\begin{align}
\label{eqn:putting-together-t-diff}
t_{s}(\taun_{s}) > t_{\overline{s-1}}(\tau_{s}) + \frac{2444(s+1)}{\cm\gamma\eta(1-\gamma)^2}, 
\qquad \text{for } s\geq 3.
\end{align}

\begin{itemize}
		
\item
We start with the cases with $s=1,2,3$. It follows from Lemma~\ref{lem:init-step} that
\begin{equation}
	t_2(\taun_2) \ge t_1(\taun_1) \ge \frac{\log 3}{1+17\cm/\cbone} \frac{|\cS|}{\eta} = C_1 \frac{|\cS|}{\eta},
	\label{eq:t2-t1-taus-lower-bound-S}
\end{equation}
which validates the above claim~\eqref{eqn:american-quartet} for $s = 1$ and $s=2$.
In addition, Lemma~\ref{lem:grand-final-victory} 
ensures that 
\begin{align}
	\notag t_{3}(\taun_3) - 
	\max \Big\{ t_{\overline{1}}(\gamma^{3}-1/4),~ t_{\overline{2}}(\tau_{3}) \Big\}
	&\ge 10^{-10}\cp\cm^{0.5}\eta^{0.5}(1-\gamma)^2\Big(t_{1}(\taun_{1}) \Big)^{1.5}\\
	&\ge \frac{9776}{\cm\gamma\eta(1-\gamma)^2},
\end{align}
where the last inequality is guaranteed by \eqref{eq:t2-t1-taus-lower-bound-S} and the assumption $|\cS| \geq \max\left\{\frac{4888}{C_1\cm\gamma(1-\gamma)^2}, \frac{4}{C_2(1-\gamma)^4}\right\}$.
This implies that the inequality~\eqref{eqn:putting-together-t-diff} is satisfied when $s = 3.$

\item
Next, suppose that the inequality~\eqref{eqn:american-quartet} holds true up to state $s-1$ and the inequality~\eqref{eqn:putting-together-t-diff} holds up to $s$  for some $3\leq s\leq H$. 
To invoke the induction argument, it suffices to show that the inequality~\eqref{eqn:american-quartet} continues to hold for state $s$ and the inequality~\eqref{eqn:putting-together-t-diff} remains valid for $s+1$. 
This will be accomplished by taking advantage of Lemma~\ref{lem:grand-final-victory}.

Given that the inequality~\eqref{eqn:american-quartet} holds true for every state up to $s-1$, one has
\begin{align*}
t_{s-1}(\taun_{s-1}) \geq t_{s-2}(\taun_{s-2}) \ge C_1\frac{|\cS|}{\eta}\Big(C_2(1-\gamma)^4|\cS| \Big)^{1.5^{\lfloor (s-3)/2\rfloor}-1}
\geq 	
\Big( \frac{6300 e}{\cp(1-\gamma)} \Big)^4\frac{1}{\frac{\cm\gamma}{35} \eta(1-\gamma)^2},
\end{align*}
where the last inequality is satisfied provided that $|\cS| > \max\left\{\big( \frac{6300 e}{\cp} \big)^4\frac{35}{C_1\cm\gamma(1-\gamma)^6}, \frac{4}{C_2(1-\gamma)^4}\right\}.$ 
Therefore, Lemma~\ref{lem:grand-final-victory} is applicable for both $s$ and $s+1$, thus leading to 
\begin{align*}
t_{s}(\taun_{s}) - t_{\overline{s-1}}(\tau_{s}) 
&\ge 10^{-10}\cp\cm^{0.5}\eta^{0.5}(1-\gamma)^2\Big(t_{s-2}(\taun_{s-2}) \Big)^{1.5} \\
& \ge 10^{-10}\cp\cm^{0.5}\eta^{0.5}(1-\gamma)^2 \left(C_1\frac{|\cS|}{\eta}\Big(C_2(1-\gamma)^4|\cS| \Big)^{1.5^{\lfloor (s-3)/2\rfloor}-1}\right)^{1.5} \\
& \ge C_1\frac{|\cS|}{\eta}\Big(C_2(1-\gamma)^4|\cS| \Big)^{1.5^{\lfloor (s-1)/2\rfloor}-1}. 
\end{align*}
Here, the last step relies on the condition $10^{-10}\cp\cm^{0.5}\eta^{0.5}(1-\gamma)^2 (C_1\frac{|\cS|}{\eta})^{0.5}\geq 1.$ This in turn establishes the property \eqref{eqn:american-quartet} for state $s$ (given that $t_{\overline{s-1}}(\tau_{s})\geq 0$).  
In addition, Lemma~\ref{lem:grand-final-victory} --- when applied to $s+1$ --- gives 
\begin{align*}
t_{s+1}(\taun_{s+1}) - t_{\overline{s}}(\tau_{s+1}) 
&\ge 10^{-10}\cp\cm^{0.5}\eta^{0.5}(1-\gamma)^2\Big(t_{s-1}(\taun_{s-1}) \Big)^{1.5} \\
& \ge 10^{-10}\cp\cm^{0.5}\eta^{0.5}(1-\gamma)^2 \left(C_1\frac{|\cS|}{\eta}\Big(C_2(1-\gamma)^4|\cS| \Big)^{1.5^{\lfloor (s-2)/2\rfloor}-1}\right)^{1.5} \\
& \ge C_1\frac{|\cS|}{\eta}\Big(C_2(1-\gamma)^4|\cS| \Big)^{1.5^{\lfloor s/2\rfloor}-1}\\
& \ge \frac{2444(s+2)}{\cm\gamma\eta(1-\gamma)^2}, 
\end{align*}
where the last step follows as long as $|\cS| > \max\left\{\frac{4888}{C_1\cm\gamma(1-\gamma)^2}, \frac{4}{C_2(1-\gamma)^4}\right\}$. 
We have thus established the property \eqref{eqn:putting-together-t-diff} for state $s+1$. 
\end{itemize}

Putting all the above pieces together, we arrive at the inequality \eqref{eqn:american-quartet}, thus establishing Theorem~\ref{thm:is-it-the-final-result}.
\end{proof}

\subsection*{Step 7: adapting the proof to establish Theorem~\ref{thm:unregularized-main}}

Thus far, we have established Theorem~\ref{thm:unregularized}, and are well equipped to return to the proof of Theorem~\ref{thm:unregularized-main}. 
As a remark,  Theorem~\ref{thm:unregularized} and its analysis posits that for a large fraction of the key primary states (as well as their associated adjoint states), softmax PG methods can take a prohibitively large number of iterations to converge.  The issue, however, is that there are in total only $O(H)$ key primary states and adjoint states, accounting for a vanishingly small fraction of all $|\cS|$ states. In order to extend Theorem~\ref{thm:unregularized} to Theorem~\ref{thm:unregularized-main} (the latter of which is concerned with the error averaged over the entire state space), we would need to show that the value functions associated with those booster states --- which account for a large fraction of the state space --- also converge slowly.

In the MDP instance constructed in Section~\ref{sec:mdp-construction}, however, the action space associated with the booster states is a singleton set, meaning that the action is always optimal. As a result, we would first need to modify/augment the action space of  booster states, so as to ensure that their learned actions remain suboptimal before the algorithm converges for the associated key primary states and adjoint states.


\paragraph{A modified MDP instance.}

We now augment the action space for all booster states in the MDP $\mathcal{M}$ constructed in in Section~\ref{sec:mdp-construction}, leading to a \emph{slightly modified MDP} denoted by $\mathcal{M}_{\mathsf{modified}}$:
\begin{itemize}
\item 
for any key primary state $s\in \{3,\cdots, H\}$ and any associated booster state $\widehat{s}\in\widehat{\mathcal{S}}_{s}$, take the action space of $\widehat{s}$ to be $\{a_0, a_1\}$ and let
\begin{align}
P(0\mymid \widehat{s},a_0) =  0.9, ~~ P(s\mymid \widehat{s},a_0) =  0.1, ~~ r(\widehat{s},a_0)&=0.9\gamma\tau_s,~~
P(s\mymid \widehat{s},a_{1}) = 1,  ~~ r(\widehat{s},a_1)=0;   
\end{align}
\item
for any key adjoint state $\overline{s} \in \{\overline{1},\cdots,\overline{H}\}$ and any associated booster state $\widehat{s}\in\widehat{\mathcal{S}}_{\overline{s}}$,  
take the action space of $\widehat{s}$ to be $\{a_0, a_1\}$ and let
\begin{align}  
P(0\mymid \widehat{s},a_0) =  0.9, ~~P(\overline{s}\mymid \widehat{s},a_0) =  0.1,~~ r(\widehat{s},a_0)=0.9\gamma^2\tau_s,~~
P(\overline{s}\mymid \widehat{s},a_{1}) = 1,  ~~ r(\widehat{s},a_1)&=0;  
\end{align}
\item all other comoponents of $\mathcal{M}_{\mathsf{modified}}$ remain identical to those of the original $\mathcal{M}$. 
\end{itemize}

\paragraph{Analysis for the new booster states.} 
Given that the dynamics of non-booster states are un-affected by the booster states, it suffices to perform analysis for the booster states. 
Let us first consider any key primary state $s$ and any associated booster state $\widehat{s}$. 
\begin{itemize}
	\item As can be easily seen, 
\begin{align*}
Q^{(t)}(\widehat{s},a_{0}) & =r(\widehat{s},a_{0})+\gamma P(s\mymid\widehat{s},a_{0})V^{(t)}(s)+\gamma P(0\mymid\widehat{s},a_{0})V^{(t)}(0)=0.9\gamma\tau_{s}+0.1\gamma V^{(t)}(s),\\
Q^{(t)}(\widehat{s},a_{1}) & =r(\widehat{s},a_{1})+\gamma P(s\mymid\widehat{s},a_{0})V^{(t)}(s)=\gamma V^{(t)}(s),
\end{align*}
where we have used the basic fact $V^{(t)}(0)=0$ (see \eqref{eq:V-star-0-0-analysis-pi} in Lemma~\ref{lem:basic-properties-MDP-pi}). 
Given that $V^{(t)}(\widehat{s})$ is a convex combination of $Q^{(t)}(\widehat{s},a_{0})$ and $Q^{(t)}(\widehat{s},a_{1})$,
one can easily see that: 
if $V^{(t)}(s) < \tau_s$, then one necessarily has $V^{(t)}(\widehat{s}) < \gamma\tau_s$

\item 
Similarly, the optimal Q-function w.r.t.~$\widehat{s}$ is given by 
\begin{align*}
Q^{\star}(\widehat{s},a_{0}) & =r(\widehat{s},a_{0})+\gamma P(s\mymid\widehat{s},a_{0})V^{\star}(s)+\gamma P(0\mymid\widehat{s},a_{0})V^{\star}(0)=0.9\gamma\tau_{s}+0.1\gamma V^{\star}(s),\\
Q^{\star}(\widehat{s},a_{1}) & =r(\widehat{s},a_{1})+\gamma P(s\mymid\widehat{s},a_{0})V^{\star}(s)=\gamma V^{\star}(s),
\end{align*}
which together with Lemma~\ref{lem:basic-properties-MDP-Vstar} and the definition \eqref{eq:key-parameters-defn} of $\tau_s$ indicates that $V^{\star}(\widehat{s})=Q^{\star}(\widehat{s},a_{1})=\gamma^{2s+1}$.

\item The above facts taken collectively imply that: if $V^{(t)}(s) < \tau_s$, then
\begin{align}
	V^{\star}(\widehat{s}) - V^{(t)}(\widehat{s}) > \gamma^{2s+1} - \gamma \tau_s
	= \gamma \big(\gamma^{2s} - 0.5\gamma^{\frac{2s}{3}}\big) > 0.22,
\end{align}
provided that $\gamma$ is sufficiently large (which is satisfied under the condition \eqref{eq:assumptions-constants}). 
\end{itemize}
Similarly, for any key adjoint state $\overline{s}$ and any associated booster state $\widehat{s}$,  if $V^{(t)}(\overline{s}) < \gamma\tau_s$, then one must have
\begin{align}
	V^{\star}(\widehat{s}) - V^{(t)}(\widehat{s}) > 0.22.
\end{align}

Repeating the same proof as for Theorem~\ref{thm:unregularized}, one can easily show that (with slight adjustment of the universal constants)   
\begin{align} 
	\label{eq:tbound-main}
	t_s(\tau_s) = t_{\overline{s}}(\gamma\tau_s) \ge \frac{1}{\eta} |\mathcal{S}|^{ 2^{ \Omega(\frac{1}{1-\gamma})}} 
	\qquad\text{for all } s >  0.1H. 
\end{align}
This taken together with the above analysis suffices to establish Theorem~\ref{thm:unregularized-main}, given the following two simple facts: 
(i) there are $2 H\cm(1-\gamma)|\cS|= 2\cm\ch |\cS|$ booster states, and (ii) more than $90\%$ of them need a prohibitively large number of iterations~(cf.~\eqref{eq:tbound-main}) to reach $0.22$-optimality. 
Here, we can take $\cm\ch > 0.18$ which satisfies \eqref{eq:assumptions-constants}. The proof is thus complete.

\section{Discussion}
\label{sec:discussion}

This paper has developed an algorithm-specific lower bound on the iteration complexity of the softmax policy gradient method, obtained by analyzing its trajectory on a carefully-designed hard MDP instance.  We have shown that the iteration complexity of softmax PG methods can scale pessimistically, in fact (super-)exponentially, with the dimension of the state space and the effective horizon of the discounted MDP of interest. Our finding makes apparent the potential {\em inefficiency} of softmax PG methods in solving large-dimensional and long-horizon problems. In turn, this suggests the necessity of carefully adjusting update rules and/or enforcing proper regularization in accelerating policy gradient methods.

Our work relies heavily on proper exploitation of the structural properties of the MDP in algorithm-dependent analysis,  which might shed light on lower bound construction for other algorithms as well. 
For instance, if the objective function (i.e., the value function) is augmented by a regularization term, 
how does the choice of regularization affect the global convergence behavior? 
While \citet{agarwal2019optimality} demonstrated polynomial-time convergence of PG methods in the presence of  log-barrier regularization, non-asymptotic analysis of PG methods with other popular regularization --- particularly entropy regularization --- remains unavailable in existing literature. How to understand the (in)-effectiveness of entropy-regularized PG methods is of fundamental importance in the theory of policy optimization. 
Additionally, the current paper concentrates on the use of constant learning rates; 
it falls short of accommodating more adaptive learning rates, which might be a potential solution to accelerate vanilla PG methods. Furthermore, our strategy for lower bound construction might be extended to unveil algorithmic bottlenecks of policy optimization in multi-agent Markov games as well. 
All this is worthy of future investigation.


\section*{Acknowledgements}
 
Y.~Wei is supported in part by the the NSF grants CCF-2106778 and  CAREER award DMS-2143215. 
Y.~Chi is supported in part
by the grants ONR N00014-18-1-2142 and N00014-19-1-2404, ARO W911NF-18-1-0303,
 NSF CCF-1806154, CCF-2007911 and CCF-2106778.
Y.~Chen is supported in part by the Alfred P.~Sloan Research Fellowship, the Google Research Scholar Award, 
the AFOSR grant FA9550-22-1-0198, 
the ONR grant N00014-22-1-2354,  
and the NSF grants CCF-2221009, CCF-1907661, IIS-2218713 and IIS-2218773.

\bibliography{bibfileRL_2021}
\bibliographystyle{apalike}


\newpage
\appendix


\begin{table}
\begin{center}
\begin{tabular}{ll}
\toprule 
\hline 
$\cS$, $\cA_s$, $\gamma$ \vphantom{$\frac{1^{7}}{1}$}		& state space, action space associated with state $s$, discount factor				\tabularnewline
$P(s'\mymid s,a)$ 			& probability of transitioning from state $s$ to state $s'$ upon execution of action $a$ 	\tabularnewline
$r(s,a)$				& immediate reward gained in state $s$ when action $a$ is taken; $r(s,a)\in[-1,1]$ 			\tabularnewline
$\eta$					& stepsize or learning rate									\tabularnewline
$\pi^{(t)}$, $\theta^{(t)}$ 		& policy estimate and its associated parameterization in the $t$-th iteration 			\tabularnewline
$V^{\pi}$, $V^{(t)}$, $V^{\star}$ 	& value function of $\pi$, value function of $\pi^{(t)}$, optimal value function 		\tabularnewline
$Q^{\pi}$, $Q^{(t)}$, $Q^{\star}$ 	& Q-function of $\pi$, Q-function of $\pi^{(t)}$, optimal Q-function 				\tabularnewline
$A^{\pi}$, $A^{(t)}$		 	& advantage function of $\pi$, advantage function of $\pi^{(t)}$				\tabularnewline 
$\mu$					& initial state distribution (used in defining the objective function \eqref{eq:value_max})	\tabularnewline
$d^{\pi}_{\mu}$, $d^{(t)}_{\mu}$ 	& discounted state visitation distribution of $\pi$ and $\pi^{(t)}$ from initial state distribution $\mu$ \tabularnewline
$\tau_s$, $p$, $r_s$ 			& useful quantities: $\tau_s=0.5\gamma^{\frac{2s}{3}}$, $p=\cp (1-\gamma)$ and $r_s = 0.5\gamma^{\frac{2s}{3}+\frac{5}{6}}$\tabularnewline
$H$ 					& number of key primary states: $H=\frac{\ch}{1-\gamma}$ 					\tabularnewline
$\cS_1$, $\cS_2$ 			& buffer state subsets: $|\cS_1|=c_{\mathrm{b},1} (1-\gamma)|\cS|$, $|\cS_2|=c_{\mathrm{b},2} (1-\gamma) |\cS|$		\tabularnewline
$\widehat{\cS}_s$ 			& booster state sets w.r.t.~state $s$: $|\widehat{\cS}_s| =\cm(1-\gamma)|\cS|$			\tabularnewline
$\cS_{\mathsf{primary}}$, $\cS_{\mathsf{adj}}$ & set of key primary states, set of key adjoint states 					\tabularnewline
$\overline{s}$				& adjoint state associated with primary state $s$						\tabularnewline
$\mathcal{M}$, $\mathcal{M}_{\mathsf{modified}}$ & MDPs constructed to prove Theorem~\ref{thm:unregularized} and Theorem~\ref{thm:unregularized-main} \tabularnewline
$t_{s}(\tau)$ 	\vphantom{$\frac{1}{1^{7^{7^7}}}$}			& crossing time: $\arg\min\{t\mymid V^{(t)}(s)\geq \tau\}$					\tabularnewline
\toprule \hline 
\end{tabular}\caption{Summary of notation and parameters.\label{tab:notation}}
\end{center}
\end{table}

\section{Preliminary facts}
\label{sec:preliminaries}

\subsection{Basic properties of the constructed MDP}
\label{sec:preliminary-basic-properties}

In this section, we provide more basic properties about the MDP we have constructed (see Section~\ref{sec:mdp-construction}).  
Specifically, we present a miscellaneous collection of basic relations regarding more general policies, 
postponing the proof to Appendix~\ref{sec:proof-lemma:basic-properties-MDP-pi}.
\begin{lemma}
	\label{lem:basic-properties-MDP-pi}
	Consider any policy $\pi$, and recall the quantities defined in \eqref{eq:key-parameters-defn}. Suppose that $\gamma^{2H}\geq 1/2$ and $0<\cp\leq 1/6$.
	\begin{itemize}
		\item[(i)] For any state $s\in \{3,\cdots, H\}$, one has
	\begin{subequations}	\label{eq:Q-s2-pi-LB-UB-a}
\begin{align}
	\gamma^{\frac{3}{2}}\tau_{s-1}  \leq Q^{\pi}(s,a_{0}) &= r_{s}+\gamma^{2}p\tau_{s-2} \leq\gamma^{\frac{1}{2}}\tau_{s}, 	\label{eq:Q-s2-pi-LB-UB-a0}\\
	 Q^{\pi}(s,a_{1})& =\gamma V^{\pi}(\overline{s-1}), 	\label{eq:Q-s2-pi-LB-UB-a1}\\
	Q^{\pi}(s,a_{2})& =r_{s}+\gamma pV^{\pi}(\overline{s-2})  \leq\gamma^{\frac{1}{2}}\tau_{s}. 	\label{eq:Q-s2-pi-LB-UB-a2}
\end{align}
\end{subequations}
	%
	If one further has  $V^{\pi}(\overline{s-2})\geq 0$, then $Q^{\pi}(s,a_{2}) \geq\gamma^{\frac{3}{2}}\tau_{s-1}$.

	\item[(ii)] If $V^{\pi}(s) \geq \tau_s$ for some $s\in \{3,\cdots, H\}$, then we necessarily have
	\begin{align}
		\pi(a_{1}\mymid s) \geq \frac{1-\gamma}{2}. 
		\label{eq:pi-a2-s-LB-V-large-8923}
	\end{align}

	\item[(iii)] For any $\overline{s}\in \{\overline{1},\cdots,\overline{H}\}$, one has
\begin{align}
	Q^{\pi}(\overline{s},a_{0}) & =\gamma\tau_{s}  \qquad\text{and} \qquad Q^{\pi}(\overline{s},a_{1})  = \gamma V^{\pi}(s),  
	\label{eq:Qpi-s-adj}
\end{align}
 where we recall the definition of $V^{\pi}(1)$ and $V^{\pi}(2)$ in \eqref{eq:convenient-notation-V1-V2-Q1-Q2}. 
 In addition, if  $\pi(a_1\mymid \overline{s})>0$, then  
	\begin{align}
		V^{\pi}(\overline{s}) \geq \gamma\tau_s
		\qquad \text{holds if and only if} \qquad
		V^{\pi}(s) \geq \tau_{s}. 
		\label{eq:Vpi-s-lower-bound-2683}
	\end{align}
	This means that: if $\pi^{(t)}(a_1\mymid \overline{s})>0$ holds for all $t\geq 0$, then one necessarily has
	\begin{align}
		t_{\overline{s}}(\gamma \tau_s) = t_{s}(\tau_s).
		\label{eq:equivalence-ts-tsbar}
	\end{align}
			
	\item[(iv)] For any  policy $\pi$, we have
\begin{subequations}
\label{eq:Q-pi-s12-a0-S1S2}
\begin{align}
 & Q^{\pi}(1,a_{0})=-\gamma^{2},\quad Q^{\pi}(1,a_{1})=\gamma^{2}, \quad V^{\pi}(1)=-\gamma^{2}\pi(a_{0}\mymid 1)+\gamma^{2}\pi(a_{1}\mymid 1),\\
 & Q^{\pi}(2,a_{0})=-\gamma^{4},\quad Q^{\pi}(2,a_{1})=\gamma^{4}, \quad V^{\pi}(2)=-\gamma^{4}\pi(a_{0}\mymid 2)+\gamma^{4}\pi(a_{1}\mymid 2).
\end{align}
\end{subequations}
%

	\item[(v)]	Consider any policy $\pi$ obeying $\min_{a,s}\pi(a\mymid s)>0$. For every $s\in \{3,\cdots, H\}$, 
if $V^{\pi}(s) \geq \gamma^{\frac{1}{2}}\tau_{s}$ occurs, then one necessarily has $V^{\pi}(s-1) \geq \tau_{s-1}.$

	\item[(vi)] If $V^{\pi}(s-2) < \tau_{s-2}$ and $\pi(a_1 \mymid \overline{s-2}) >0$, then 
		\[
			\Qpi(s,a_0) - \Qpi(s,a_2) = \gamma p \big(\gamma\tau_{s-2} - V^\pi(\overline{s-2}) \big) > 0. 
		\]
		If $V^{\pi}(s-1) \leq \tau_{s-1}$ and $V^{\pi}(\overline{s-2})\geq 0$, then
		\[
			\min\big\{Q^{\pi}(s, a_0), Q^{\pi}(s, a_2) \big\} - Q^{\pi}(s, a_1) \geq (1-\gamma)/8.
		\]
	\item[(vii)] Consider the softmax PG update rule \eqref{eq:PG-update-all}. One has for any $s \in \cS$ and any $\theta$,
	\begin{align} \label{eq:zero-sum}
		\sum_a \frac{\partial V^{\pi_{\theta}}(\mu)}{\partial \theta(s, a)} = 0
		\qquad \text{and}\qquad 
		\sum_a \theta^{(t)}(s,a) = 0
	\end{align}
	\end{itemize}

\end{lemma}

\begin{remark}
\label{remark:monotonicity-crossing-times}
As it turns out,  invoking Part (v) of Lemma~\ref{lem:basic-properties-MDP-pi} recursively reveals 
that: for any $2\leq s \leq H$ and any $t< t_s(\taun_s)$, we have 
\begin{align}
	V^{(t)}(s^{\prime})< \gamma^{1/2}\tau_{s^{\prime}} < \taun_{s^{\prime}} \qquad  \text{for all }s^{\prime}\text{ obeying }s \le s^{\prime} \leq H. 
	\label{eq:V-t-sprime-3-H}
\end{align}
This in turn implies that $t_2(\taun_2)\leq t_3(\taun_3)\leq  \cdots \leq t_H(\taun_H)$ according to the definition \eqref{eqn:approx-optimal-time}.  
\end{remark}

%

%

Let us point out some implications of Lemma~\ref{lem:basic-properties-MDP-pi} that help guide our lower bound analysis. 
Once again, it is helpful to look at the results of this lemma when $\gamma \approx 1$ and $\gamma^{H} \approx 1$. 
In this case, the quantities defined in \eqref{eq:key-parameters-defn} obey $\tau_s\approx r_s \approx 1/2$, allowing us to obtain the following messages: 
\begin{itemize}
\item Lemma~\ref{lem:basic-properties-MDP-pi}(i) implies that, under mild conditions, 
\begin{align*}
	Q^{\pi}(s,a_0)\approx Q^{\pi}(s,a_2) \approx 1/2
\end{align*}
holds any $s\in \{3,\cdots, H\}$ and any policy $\pi$. 
In comparison to the optimal values \eqref{eq:optimal-values-approx-1}, 
this result uncovers the strict sub-optimality of actions $a_0$ and $a_2$, and indicates that one cannot possibly approach the optimal values unless $\pi(a_{1}\mymid s)\approx 1$. 

\item  As further revealed by Lemma~\ref{lem:basic-properties-MDP-pi}(ii), one needs to ensure  a sufficiently large  $\pi(a_{1}\mymid s)$ --- i.e., $\pi(a_{1}\mymid s) \geq (1- \gamma)/2$ --- in order to achieve $V^{\pi}(s) \gtrapprox 1/2$.

\item  Lemma~\ref{lem:basic-properties-MDP-pi}(iii) establishes an intimate connection between $V^{\pi}(s)$ and $V^{\pi}(\overline{s})$: if we hope to attain $V^{\pi}(\overline{s})\gtrapprox 1/2$ for an adjoint state $\overline{s}$, then one needs to first ensure that its associated primary state achieves  $V^{\pi}(s)\gtrapprox 1/2$. The equivalence property \eqref{eq:equivalence-ts-tsbar} allows one to propagate the crossing time of state $s$ to that of state $\overline{s}$.  

\item In Lemma~\ref{lem:basic-properties-MDP-pi}(iv), we make clear that the Q-functions w.r.t.~the buffer states are independent of the policy in use. 

\item Lemma~\ref{lem:basic-properties-MDP-pi}(v) further establishes an intriguing connection between the crossing time of state $s$ and that of the preceding state $s-1$. 

\item Lemma~\ref{lem:basic-properties-MDP-pi}(vi) uncovers that: (a) if $V^{\pi}(s-2)$ is not sufficiently large, then the Q-value associated with $(s,a_0)$ dominates the one associated with $(s,a_2)$; (b)  if $V^{\pi}(s-1)$ is not large enough, then the Q-value associated with $(s,a_1)$ is dominated by that of the other two. 

\item As indicated by Lemma~\ref{lem:basic-properties-MDP-pi}(vii), the sum of the iterate $\theta^{(t)}(s,a)$ over $a$ remains unchanged throughout the execution of the algorithm. 

\end{itemize}

Another key feature that permeates our analysis is a certain monotonicity property of value function estimates as the iteration count $t$ increases, 
which we discuss in the sequel. 
To begin with, akin to the monotonicity properties of gradient descent \citep{beck2017first}, the softmax PG update is known to achieve monotonic performance improvement in a pointwise manner, as summarized in the following lemma. The interested reader is referred to \citet[Lemma C.2]{agarwal2019optimality} for details.
%
\begin{lemma}
	\label{lem:ascent-lemma-PG}
	Consider the softmax PG method \eqref{eq:PG-update-all}. One has
	\begin{align*}
		V^{(t+1)}(s) \geq V^{(t)}(s)
		\qquad \text{and}\qquad 
		Q^{(t+1)}(s,a) \geq Q^{(t)}(s,a)
	\end{align*}
	for any state-action pair $(s,a)$ and any $t\geq 0$, provided that $0<\eta< (1-\gamma)^2 / 5$. 
\end{lemma}

The preceding monotonicity feature, in conjunction with the uniform initialization scheme, ensures non-negativity of value function estimates throughout the execution of the algorithm. 
\begin{lemma}
	\label{lem:non-negativity-PG}
	Consider the  softmax PG method \eqref{eq:PG-update-all}, and suppose the initial policy $\pi^{(0)}(\cdot \mymid s)$ for any $s\in \cS$ is given by a uniform distribution over the action space $\cA_s$ and $0< \eta < (1-\gamma)^2/5$. Then one has
	\begin{align*}
		\forall t\geq 0, ~ \forall s\in \cS: \qquad  V^{(t)}(s) \geq 0.
	\end{align*}
\end{lemma}
\begin{proof}
	The only negative rewards in our constructed MDP are $r(s_1,a_0)$ for $s_1\in \cS_1$ and $r(s_2,a_0)$ for $s_2\in \cS_2$. When $\pi^{(0)}(\cdot\mymid s_1)$ is uniformly distributed, the MDP specification \eqref{eq:P-r-state-S1S2} gives
	\[
		\forall s_1\in \cS_1: \qquad V^{(0)}(s_1) = 0.5 r(s_1,a_0) + 0.5 r(s_1,a_1) = 0.
	\]
	Similarly, one has $V^{(0)}(s_2)=0$ for all $s_2\in \cS_2$. Applying Lemma~\ref{lem:ascent-lemma-PG}, we can demonstrate that $V^{(t)}(s)\geq V^{(0)}(s) \geq 0$ for any $s\in \cS_1 \cup \cS_2$ and any $t\geq 0$. From the Bellman equation, it is easily seen that the  value function $V^{(t)}$ of any other state is a linear combination of $\{r(s,a)\mymid s\notin \cS_1, s\notin \cS_2\}$, $\{V^{(t)}(s_1) \mymid s_1\in \cS_1 \}$ and $\{V^{(t)}(s_2)\mymid s_2\in \cS_2\}$, which are all non-negative. It thus follows that $V^{(t)}(s)\geq 0$ for any $s\in \cS$ and any $t\geq 0$.
\end{proof}

\subsection{A type of recursive relations}

In addition, we make note of a sort of recursive relations that appear commonly when studying the dynamics of gradient descent \citep{beck2017first}. The proof of the following lemma can be found in Appendix~\ref{sec:proof-lem:opt-lemma}. 
\begin{lemma}
\label{lem:opt-lemma}

Consider a positive sequence $\{x_t\}_{t\geq 0}$.
\begin{enumerate}
\item[(i)] Suppose that $x_t\leq x_{t-1}$ for all $t>0$. If there exists some quantity $c_{\mathrm{l}} >0$ obeying  $c_{\mathrm{l}} x_0\leq 1/2$ and
\begin{subequations}
\begin{align}
	x_t \geq x_{t-1} - c_{\mathrm{l}} x_{t-1}^2 \qquad \text{for all }t > 0,
	\label{eq:xt-xtminus1-relation-LB}
\end{align}
then one has
\begin{align}
	x_{t}\geq\frac{1}{2c_{\mathrm{l}} t + \frac{1}{x_{0}}} \qquad \text{for all }t \geq 0. 
	\label{eq:xt-sequence-LB-123}
\end{align}
\end{subequations}

\item[(ii)] If  there exists some quantity $c_{\mathrm{u}} >0$ obeying
\begin{subequations}
\begin{align}
	x_t \leq x_{t-1} - c_{\mathrm{u}} x_{t-1}^2 \qquad \text{for all }t > 0, 
	\label{eq:xt-xtminus1-relation-UB}
\end{align}
then it follows that
\begin{align}
	x_{t}\leq\frac{1}{c_{\mathrm{u}} t + \frac{1}{x_{0}}} \qquad \text{for all }t \geq 0. 
	\label{eq:xt-sequence-UB-123}
\end{align}
\end{subequations}

\item[(iii)]  Suppose that $0<x_{t} < c_{x}$ for all $t<t_{0}$ and $x_{t_0}\ge c_{x}$ for some quantity $c_x>0$. Assume that
\begin{subequations}
\begin{equation}
	x_{t}\geq x_{t-1}+c_{-}x_{t-1}^{2}\qquad\text{for all }0<t\leq t_0
	\label{eq:xt-x-tminus1-lower-bound-plus}
\end{equation}
for some quantity $c_- >0$. Then one necessarily has
\begin{align}
	t_0 \leq\frac{1+c_{-}c_x}{c_- x_{0}}.  \label{eq:t0-UB-opt-lemma}
\end{align}

\end{subequations}

\item[(iv)] Suppose that 
\begin{subequations}
\begin{equation}
	0 \leq x_{t}\leq x_{t-1}+c_{+}x_{t-1}^{2}\qquad\text{for all }0 < t \leq t_{0}
	\label{eq:xt-x-tminus1-upper-bound-plus}
\end{equation}
for some quantity $c_+ >0$. Then one necessarily has
\begin{align}
	t_0 \geq\frac{\frac{1}{x_0}-\frac{1}{x_{t_0}}}{c_+}.  \label{eq:t0-LB-opt-lemma}
\end{align}
\end{subequations}

\end{enumerate}

\end{lemma}

\subsection{Proof of Lemma~\ref{lem:basic-properties-MDP-Vstar}}
\label{sec:proof-lemma:basic-properties-MDP-Vstar}

(i) Let us start with state $0$. Given that this is an absorbing state and that $r(0,a_0)=0$, we have
$ V^{\star}(0) = 0$.

 \medskip
\noindent
(ii) Next, we turn to the buffer states in $\cS_1$ and $\cS_2$. For any $s_1\in \cS_1$, the Bellman equation gives
\begin{subequations}
\label{eq:Q-s1-a0-a1}
\begin{align}
Q^{\star}(s_{1},a_{0}) & =r(s_{1},a_{0})+\gamma V^{\star}(0)=-\gamma^{2};\\
	Q^{\star}(s_{1},a_{1}) & =r(s_{1},a_{1})+\gamma V^{\star}(0) = \gamma^{2} .
\end{align}
\end{subequations}
This in turn implies that $V^{\star}(s_1)= Q^{\star}(s_{1},a_{1})=\gamma^2$. 
Repeating the same argument, we arrive at $V^{\star}(s_2) = Q^{\star}(s_2,a_1) = r(s_{2},a_{1}) = \gamma^4$ for any $s_2\in \cS_2$.


\medskip
\noindent
(iii) We then move on to the adjoint states $\overline{1}$ and $\overline{2}$. From the construction \eqref{eq:P-r-adjoint-12}, 
the Bellman equation yields
\begin{align*}
Q^{\star}(\overline{1},a_{0}) & =r(\overline{1},a_{0})+\gamma V^{\star}(0)=\gamma\tau_{1} < \gamma /2,\\
Q^{\star}(\overline{1},a_{1}) & =r(\overline{1},a_{1})+\frac{\gamma}{|\mathcal{S}_{1}|}\sum_{s_{1}\in\mathcal{S}_{1}}V^{\star}(s_{1})= \frac{\gamma}{|\mathcal{S}_{1}|}\sum_{s_{1}\in\mathcal{S}_{1}}V^{\star}(s_{1}) = \gamma^3,
\end{align*}
where the last identity follows since $V^{\star}(s_{1})=\gamma^2$. 
This in turn indicates that $V^{\star}(\overline{1})=\max\{Q^{\star}(\overline{1},a_{0}), Q^{\star}(\overline{1},a_{1})\} =\gamma^3$, provided that $\gamma^2 \geq 1/2$.
Similarly, repeating this argument shows that  $V^{\star}(\overline{2})=\gamma^5$, as long as $\gamma^4 \geq 1/2$. As before, the optimal action in state $\overline{1}$ (resp.~$\overline{2}$) is $a_1$.

\medskip
\noindent
(iv) The next step is to determine $V^{\star}(s)$ for  any $s\in \{3,\cdots,H\}$. Suppose that $V^{\star}(\overline{s-2})=\gamma^{2s-3}$ and  $V^{\star}(\overline{s-1})=\gamma^{2s-1}$. Then the construction \eqref{eq:P-r-primary-3-H} together with the Bellman equation yields
\begin{align*}
Q^{\star}(s,a_{0}) & =r(s,a_{0})+\gamma V^{\star}(0)=r_{s}+\gamma^{2}p\tau_{s-2} < 2/3;\\
Q^{\star}(s,a_{1}) & =r(s,a_{1})+\gamma V^{\star}(\overline{s-1})=\gamma\gamma^{2s-1}=\gamma^{2s}; \\
Q^{\star}(s,a_{2}) & =r(s,a_{2})+\gamma(1-p)V^{\star}(0)+\gamma pV^{\star}(\overline{s-2})=r_{s}+p\gamma^{2s-2} < 2/3.
\end{align*}
Consequently, one has  $V^{\star}(s)=Q^{\star}(s,a_{1})=\gamma^{2s}$ --- namely, $a_1$ is the optimal action --- as long as $\gamma^{2s}\geq 2/3$.

\medskip
\noindent
(v) We then turn attention to $V^{\star}(\overline{s})$ for any $\overline{s}\in \{\overline{3},\cdots,\overline{H}\}$.
Suppose that $V^{\star}(s)= \gamma^{2s}$. In view of the construction \eqref{eq:P-r-adjoint-states} and the Bellman equation, one has
\begin{align*}
Q^{\star}(\overline{s},a_{0}) & =r(\overline{s},a_{0})+\gamma V^{\star}(0)=\gamma\tau_{s} < 1/2;\\
Q^{\star}(\overline{s},a_{1}) & =r(\overline{s},a_{1})+\gamma V^{\star}(s)=\gamma^{2s+1}. 
\end{align*}
Hence, we have $V^{\star}(\overline{s})=Q^{\star}(\overline{s},a_{1})=\gamma^{2s+1}$ --- with the optimal action being $a_1$ --- provided that $\gamma^{2s+1} \geq 1/2$. 

\medskip
\noindent
(vi) Applying an induction argument based on Steps (iii), (iv) and (v), we conclude that
\begin{align}
	V^{\star}(s)= \gamma^{2s} 
	\qquad \text{and} \qquad V^{\star}(\overline{s})=\gamma^{2s+1}
	\label{eq:Vstar-overline-s-gamma-2s}
\end{align}
for all $3\leq s\leq H$, with the proviso that $\gamma^{2H}\geq 2/3$ and $\gamma^{2H+1}\geq 1/2$.

\medskip
\noindent
(vii) In view of our MDP construction,  a negative immediate reward (which is either $-\gamma^2$ or $-\gamma^4$) is accrued only when the current state lies in the buffer sets $\cS_1$ and $\cS_2$ and when action $a_0$ is executed. However, once $a_0$ is taken, the MDP will transition to the absorbing state 0, with all subsequent rewards frozen to 0. 
In conclusion, the entire MDP trajectory cannot receive negative immediate rewards more than once, thus indicating that $Q^{\pi}(s,a)\geq \min\{-\gamma^2, -\gamma^4\}=-\gamma^2$ irrespective of $\pi$ and  $(s,a)$.

\subsection{Proof of Lemma~\ref{lem:basic-properties-MDP-pi}}
\label{sec:proof-lemma:basic-properties-MDP-pi}

\paragraph{Proof of Part (i).}
Before proceeding, we make note of a straightforward fact
\begin{align}
	\label{eq:V-star-0-0-analysis-pi} 
	V^{\pi}(0) = 0,
\end{align}
given that state 0 is an absorbing state and $r(0,a_0)=0$.

For any $s\in\{ 3, \cdots, H\}$, 
the construction \eqref{eq:P-r-primary-3-H} together with \eqref{eq:V-star-0-0-analysis-pi} and the Bellman equation yields
\begin{subequations}
\label{eq:Q-pi-s-a012-identity}
\begin{align}
	Q^{\pi}(s,a_{0}) & =r(s,a_{0})+\gamma V^{\pi}(0)=r_{s}+\gamma^{2}p\tau_{s-2} ;
	\label{eq:Qpi-s-a0-identity}\\
	Q^{\pi}(s,a_{1}) & =r(s,a_{1})+\gamma V^{\pi}(\overline{s-1})=\gamma V^{\pi}(\overline{s-1}); \label{eq:Qpi-s-a1-identity} \\
	Q^{\pi}(s,a_{2}) & =r(s,a_{2})+\gamma(1-p)V^{\pi}(0)+\gamma pV^{\pi}(\overline{s-2})=r_{s}+\gamma pV^{\pi}(\overline{s-2}). 
	\label{eq:Qpi-s-a2-identity}
\end{align}
\end{subequations}
Recalling the choices of $\tau_s$, $r_s$ and $p$ in \eqref{eq:key-parameters-defn},  
we can continue the derivation in \eqref{eq:Qpi-s-a0-identity} to reach
\[
	Q^{\pi}(s,a_{0})= 0.5 \gamma^{\frac{2s}{3}+\frac{5}{6}}+\cp (1-\gamma) \gamma^{\frac{2s}{3}+\frac{2}{3}}
\]
\[
\Longrightarrow\qquad
	\gamma^{\frac{3}{2}}\tau_{s-1}= 0.5 \gamma^{\frac{2s}{3}+\frac{5}{6}}
	\leq Q^{\pi}(s,a_{0})
	\leq 0.5 \gamma^{\frac{2s}{3}+\frac{1}{2}}
	=\gamma^{\frac{1}{2}}\tau_{s} .
\]
Here, the last inequality is valid when $\cp \leq 1/6$, given that $\gamma^{\frac{1}{3}}+\frac{1-\gamma}{6}\gamma^{\frac{1}{6}} \leq 1$ holds for any $\gamma <1$.

In addition, combining \eqref{eq:Qpi-s-a2-identity} with \eqref{eq:Vstar-overline-s-gamma-2s}, we arrive at
\[
	Q^{\pi}(s,a_{2})\leq r_{s}+\gamma pV^{\star}(\overline{s-2})
	= 0.5 \gamma^{\frac{2s}{3}+\frac{5}{6}}+ \cp (1-\gamma) \gamma^{2s-2}
	\leq  0.5 \gamma^{\frac{2s}{3}+\frac{1}{2}}
	=\gamma^{\frac{1}{2}}\tau_{s} .
\]
This is guaranteed to hold when $\cp \leq 1/6$, given that  $\gamma^{\frac{1}{3}}+\frac{1-\gamma}{3}\gamma^{\frac{4s}{3}-\frac{5}{2}}\leq \gamma^{\frac{1}{3}}+\frac{1-\gamma}{3}\gamma^{\frac{3}{2}}\leq 1$ is valid for all $\gamma<1$ and $s\geq 3$.
Moreover, if one further has  $V^{\pi}(\overline{s-2})\geq 0$, then it is seen from \eqref{eq:Qpi-s-a2-identity} that
\begin{align}
	Q^{\pi}(s,a_{2}) \geq r_{s} = 0.5 \gamma^{\frac{2s}{3}+\frac{5}{6}} = \gamma^{\frac{3}{2}}\tau_{s-1}.
\end{align}

\paragraph{Proof of Part (ii).}

By virtue of the construction \eqref{eq:P-r-primary-3-H}, we can invoke the Bellman equation to show that
\begin{align}
 V^{\pi}(s) & =\pi(a_{0}\mymid s)Q^{\pi}(s,a_{0}) +\pi(a_{1}\mymid s)Q^{\pi}(s,a_{1}) +\pi(a_{2}\mymid s)Q^{\pi}(s,a_{2})  \notag\\
 	& = \pi(a_{1}\mymid s)\cdot\gamma V^{\pi}(\overline{s-1})+\pi(a_{0}\mymid s)Q^{\pi}(s,a_{0}) +\pi(a_{2}\mymid s)Q^{\pi}(s,a_{2}) \notag\\
	& \leq\pi(a_{1}\mymid s) \gamma^{2s}+\big\{ \pi(a_{0}\mymid s)+\pi(a_{2}\mymid s)\big\} \gamma^{\frac{1}{2}}\tau_{s} \notag\\
	& = \gamma^{2s} \pi(a_{1}\mymid s) + \gamma^{\frac{1}{2}}\tau_{s} \big(1-\pi(a_{1}\mymid s)\big). \label{eq:V-pi-s-UB-567}
\end{align}
Here, the second identity comes from  \eqref{eq:Qpi-s-a1-identity}, 
 the penultimate line follows from \eqref{eq:Q-s2-pi-LB-UB-a}, \eqref{eq:Vstar-overline-s-gamma-2s}, 
as well as the facts $ V^{\pi}(\overline{s-1})\leq  V^{\star}(\overline{s-1})$,  
while the last inequality exploits the fact $\pi(a_{0}\mymid s)+\pi(a_{2}\mymid s)=1-\pi(a_{1}\mymid s)$. 

If $V^{\pi}(s)\geq \tau_s$, 
then this together with the upper bound \eqref{eq:V-pi-s-UB-567} necessarily requires that
\[
	\tau_s \leq \gamma^{2s}\pi(a_{1}\mymid s)+\gamma^{\frac{1}{2}}\tau_{s}\big(1-\pi(a_{1}\mymid s)\big),
\]
which is equivalent to saying that
\begin{equation}
\pi(a_{1}\mymid s)\geq\frac{\tau_s-\gamma^{\frac{1}{2}}\tau_{s}}{\gamma^{2s}-\gamma^{\frac{1}{2}}\tau_{s}}
	=\frac{1-\gamma^{\frac{1}{2}}}{2\gamma^{\frac{4s}{3}}-\gamma^{\frac{1}{2}}} 
\ge \frac{1-\gamma^{\frac{1}{2}}}{\gamma^{\frac{4s}{3}}}
	= \frac{1-\gamma}{\gamma^{\frac{4s}{3}}(1+\gamma^{\frac{1}{2}})}\geq\frac{1-\gamma}{2} .
	\label{eq:lower-bound-pi-a2-s-5689}
\end{equation}
%
Putting these arguments together establishes the advertised result \eqref{eq:pi-a2-s-LB-V-large-8923}.

\paragraph{Proof of Part (iii).}
For any $\overline{s}\in \{\overline{3},\cdots,\overline{H}\}$, in view of the construction \eqref{eq:P-r-adjoint-states} and the Bellman equation, one has
\begin{align*}
Q^{\pi}(\overline{s},a_{0}) & =r(\overline{s},a_{0})+\gamma V^{\pi}(0)=\gamma\tau_{s} ;\\
Q^{\pi}(\overline{s},a_{1}) & =r(\overline{s},a_{1})+\gamma V^{\pi}(s)=\gamma V^{\pi}(s). 
\end{align*}
Regarding state $\overline{1}$, we have
\begin{align*}
	Q^{\pi}(\overline{1},a_{0})  &=  r(\overline{1},a_{0}) + \gamma V^{\pi}(0) =\gamma\tau_{1} ;  \\
	Q^{\pi}(\overline{1},a_{1})  &=  r(\overline{1},a_{1}) + \gamma \frac{1}{|\cS_1|} \sum\nolimits_{s'\in \cS_1} V^{\pi}(s') =\gamma V^{\pi}(1). 
\end{align*}
Similarly, one obtains $Q^{\pi}(\overline{2},a_{0})=\gamma\tau_2$ and $Q^{\pi}(\overline{2},a_{1}) = \gamma V^{\pi}(2)$. 

Next, let us decompose $V^{\pi}(\overline{s})$ as follows:
\begin{align*}
	V^{\pi}(\overline{s}) & =\pi(a_{0}\mymid\overline{s}) Q^{\pi}(\overline{s}, a_0) +\pi(a_{1}\mymid\overline{s})Q^{\pi}(\overline{s}, a_1)\\
 & =\gamma\tau_{s}\pi(a_{0}\mymid\overline{s})+\gamma\pi(a_{1}\mymid\overline{s})V^{\pi}(s)=\gamma\tau_{s}+\gamma\pi(a_{1}\mymid\overline{s})\big(V^{\pi}(s)-\tau_{s}\big),
\end{align*}
where we have used $\pi(a_{0}\mymid\overline{s})+\pi(a_{1}\mymid\overline{s})=1$. 
From this relation and the assumption $\pi( a_{1}\mymid\overline{s} )>0$, it is straightforward to see that $V^{\pi}(\overline{s})\geq \gamma\tau_s$ if and only if $V^{\pi}(s) \geq \tau_{s}$. 
The claim \eqref{eq:equivalence-ts-tsbar} regarding $t_s(\tau_s)$ and $t_{\overline{s}}(\gamma\tau_s)$ then follows directly from the definition of $t_s$ (see \eqref{eqn:approx-optimal-time} and \eqref{eqn:approx-optimal-time-buffer}).

\paragraph{Proof of Part (iv).}

For any $s_{1}\in\mathcal{S}_{1}$, the Bellman equation yields
\begin{align*}
Q^{\pi}(s_{1},a_{0}) & =r(s_{1},a_{0})+\gamma V^{\pi}(0)=-\gamma^{2}+0=-\gamma^{2},\\
Q^{\pi}(s_{1},a_{1}) & =r(s_{1},a_{1})+\gamma V^{\pi}(0)=\gamma^{2}+0=\gamma^{2},
\end{align*}
and hence
\[
	V^{\pi}(s_{1})=\pi(a_{0}\mymid s_{1})Q^{\pi}(s_{1},a_{0})+\pi(a_{1}\mymid s_{1})Q^{\pi}(s_{1},a_{1})=-\gamma^{2}\pi(a_{0}\mymid s_{1})+\gamma^{2}\pi(a_{1}\mymid s_{1}).
\]
A similar argument immediately yields that for any $s_{2}\in\mathcal{S}_{2}$,
\[
Q^{\pi}(s_{2},a_{0})=-\gamma^{4},\qquad Q^{\pi}(s_{2},a_{1})=\gamma^{4},
\qquad\text{and}\qquad 
V^{\pi}(s_{2})=-\gamma^{4}\pi(a_{0}\mymid s_{2})+\gamma^{4}\pi(a_{1}\mymid s_{2}).
\]
These together with our notation convention \eqref{eq:convenient-notation-V1-V2-Q1-Q2} establish \eqref{eq:Q-pi-s12-a0-S1S2}.

%

\paragraph{Proof of Part (v).}

Suppose instead that $V^{\pi}(s-1) < \tau_{s-1}$. In view  of the basic property~\eqref{eq:Vpi-s-lower-bound-2683} in Lemma~\ref{lem:basic-properties-MDP-pi}, 
this necessarily requires that 
\begin{align}
	V^{\pi}(\overline{s-1}) < \gamma\tau_{s-1}.
	\label{eq:Vpi-overline-s-gamma-tau-56}
\end{align}
Taking \eqref{eq:Vpi-overline-s-gamma-tau-56}  together with the relation~\eqref{eq:Q-s2-pi-LB-UB-a1} allows us to reach
\begin{align}
\label{eqn:qpi-a1-dvorak}
	\Qpi(s, a_1) = \gamma V^{\pi}(\overline{s-1}) < \gamma^{2}\tau_{s-1} = \gamma^{\frac{4}{3}}\tau_s.
\end{align}
In addition, the properties~\eqref{eq:Q-s2-pi-LB-UB-a0} and \eqref{eq:Q-s2-pi-LB-UB-a2} imply that
\begin{align*}
	Q^\pi(s,a_0) < \gamma^{\frac{1}{2}}\tau_s \qquad \text{and } \qquad Q^\pi(s,a_2)< \gamma^{\frac{1}{2}}\tau_s.
\end{align*}
Putting everything together implies that 
\begin{align*}
	V^{\pi}(s) & \leq \max\Big\{ \Qpi(s,a_{0}) , \Qpi(s,a_{1}) , \Qpi(s,a_{2}) \Big\} < \gamma^{\frac{1}{2}}\tau_s,
\end{align*}
which contracticts the assumption  $V^{\pi}(s) \geq \gamma^{\frac{1}{2}}\tau_s$. This  establishes the claimed result for any $s\in \{3,\cdots, H\}$.

\paragraph{Proof of Part (vi).}

First, due to explicit expressions of the Q functions~\eqref{eq:Qpi-s-a0-identity} and \eqref{eq:Qpi-s-a2-identity}, one has 
\begin{align*}
	Q^{\pi}(s,a_{0}) - Q^{\pi}(s,a_{2})& = \gamma^{2}p\tau_{s-2} - \gamma pV^{\pi}(\overline{s-2})
	= 
	\gamma p \big(\gamma\tau_{s-2} - V^\pi(\overline{s-2}) \big) > 0,
\end{align*}
where the last relation holds since $V^\pi(\overline{s-2})<\gamma\tau_{s-2}$ when $V^\pi(s-2)<\tau_{s-2}$
(see \eqref{eq:Vpi-s-lower-bound-2683}).

In addition, following the same derivation as for \eqref{eqn:qpi-a1-dvorak}, 
we see that the condition $V^{\pi}(s-1) \leq \tau_{s-1}$ implies 
\begin{align*}
	\Qpi(s, a_1) \leq \gamma^{2}\tau_{s-1}. 
\end{align*}
It is also seen from Part (i) of this lemma that
\begin{align*}
	Q^{\pi}(s,a_{0}) \geq \gamma^{3/2}\tau_{s-1} 
	\qquad \text{and} \qquad
	Q^{\pi}(s,a_{2}) \geq \gamma^{3/2}\tau_{s-1}, 
\end{align*}
provided that $V^{\pi}(\overline{s-2}) \geq 0$.
Combining these two inequalities, we arrive at the claimed bound
\begin{align*}
	\min\big\{Q^{\pi}(s, a_0), Q^{\pi}(s, a_2)\big\} - Q^{\pi}(s, a_1) \geq \gamma^{3/2}\tau_{s-1} - \gamma^{2}\tau_{s-1} 
	=\frac{\gamma^{3/2}\left(1-\gamma\right)\tau_{s-1}}{1+\gamma^{1/2}} \geq (1-\gamma)/8,
\end{align*}
where the last inequality holds if $\gamma^{2s/3+5/6} \geq \gamma^{s} \geq 1/2$.  


\paragraph{Proof of Part (vii).}

According to the update rule \eqref{eq:PG-update-all}, we have --- for any policy $\pi$ --- that
\begin{align*}
\sum_a \frac{\partial V^{\pi_{\theta}}(\mu)}{\partial \theta(s, a)} =&~\sum_a \frac{1}{1-\gamma}d^{\pi_{\theta}}_{\mu}(s) \pi_{\theta}(a \mymid s)\big(Q^{\pi_{\theta}}(s, a) - V^{\pi_{\theta}}(s)\big) \\
=&~\frac{1}{1-\gamma}d^{\pi_{\theta}}_{\mu}(s) \Big(\sum_a \pi_{\theta}(a \mymid s)Q^{\pi_{\theta}}(s, a) - V^{\pi_{\theta}}(s)\sum_a \pi_{\theta}(a \mymid s)\Big) = 0,
\end{align*}
where we have used the identities  $\sum_a \pi_{\theta}(a \mymid s)=1$ and $V^{\pi}(s) = \sum_a \pi(a \mymid s)Q^{\pi}(s, a) $. 
As a result, if $\sum_a \theta^{(0)}(s,a)=0$, then it follows from the PG update rule that $\sum_a \theta^{(t)}(s,a)=0$. 

\subsection{Proof of Lemma~\ref{lem:opt-lemma}}
\label{sec:proof-lem:opt-lemma}

\paragraph{Proof of Part (i).} Dividing both sides of (\ref{eq:xt-xtminus1-relation-LB}) by $x_{t}x_{t-1}$,
we obtain
\[
\frac{1}{x_{t-1}}\geq\frac{1}{x_{t}}-\frac{c_{\mathrm{l}} x_{t-1}}{x_{t}}.
\]
If $c_1x_{0}\leq1/2$, then the monotonicity assumption gives $c_{\mathrm{l}} x_{t}\leq1/2$
for all $t\geq0$. It then follows that
\[
\frac{x_{t}}{x_{t-1}} \geq 1-c_{\mathrm{l}} x_{t-1}\geq\frac{1}{2} \qquad
\Longrightarrow\qquad\frac{1}{x_{t-1}}\geq\frac{1}{x_{t}}-\frac{c_{\mathrm{l}} x_{t-1}}{x_{t}}\geq\frac{1}{x_{t}}-2c_{\mathrm{l}}.
\]
Apply this relation recursively to deduce that
\[
\frac{1}{x_{t}}\leq\frac{1}{x_{t-1}}+2c_{\mathrm{l}} \leq\cdots\leq\frac{1}{x_{0}}+2c_{\mathrm{l}} t. 
\]
This readily concludes the proof of \eqref{eq:xt-sequence-LB-123}.

\paragraph{Proof of Part (ii).} 
Similarly, divide both sides of (\ref{eq:xt-xtminus1-relation-UB})
by $x_{t}x_{t-1}$ to derive
\[
\frac{1}{x_{t-1}}\leq\frac{1}{x_{t}}-\frac{c_{\mathrm{u}}x_{t-1}}{x_{t}}\leq\frac{1}{x_{t}}-c_{\mathrm{u}},
\]
given the monotonicity and positivity assumption $0< x_{t}\leq x_{t-1}$. Invoking
this inequality recursively gives
\[
\frac{1}{x_{t}}\geq\frac{1}{x_{t-1}}+c_{\mathrm{u}}\geq\cdots\geq\frac{1}{x_{0}}+c_{\mathrm{u}}t,
\]
thus establishing the advertised bound \eqref{eq:xt-sequence-UB-123}.

\paragraph{Proof of Part (iii).} 
We now turn attention to \eqref{eq:t0-UB-opt-lemma}. 
As is clearly seen, the non-negative sequence $\{x_t\}$ majorizes another sequence  $\{y_t\}$ generated as follows (in the sense that $x_t\geq y_t$ for all $0<t\leq t_0$)
\begin{equation}
	y_0 = x_0 \qquad \text{and} \qquad
	y_{t}= y_{t-1}+c_{-}y_{t-1}^{2}\quad\text{for all }0<t\leq t_0
	\label{eq:xt-x-tminus1-lower-bound-plus-yt}
\end{equation}
Dividing both sides of the second equation of \eqref{eq:xt-x-tminus1-lower-bound-plus-yt}
by $y_{t-1}y_{t}$, we reach
\[
	\frac{1}{y_{t-1}} = \frac{1}{y_{t}}+c_{-}\frac{y_{t-1}}{y_{t}} \geq \frac{1}{y_{t}}+\frac{c_{-}}{1+c_{-} c_x}.
\]
To see why the last inequality holds, note that, according to the
first equation of \eqref{eq:xt-x-tminus1-lower-bound-plus-yt} and the
assumption $x_{t-1}< c_{x}$ (and hence $y_{t-1}\leq x_{t-1} < c_{x}$), we have
\[
	\frac{y_{t}}{y_{t-1}} = 1+c_{-}y_{t-1} \leq 1+c_{-}c_{x}.
\]
As a result, we can apply the preceding inequalities recursively to
derive
\[
\frac{1}{y_{0}}\geq\frac{1}{y_{1}}+\frac{c_{-}}{1+c_{-}c_{x}}\geq\cdots\geq\frac{1}{y_{t_0}}+\frac{c_{-}}{1+c_{-}c_{x}}t_0
\geq\frac{c_{-}}{1+c_{-}c_{x}}t_0,
\]
and hence we arrive at \eqref{eq:t0-UB-opt-lemma},
\[
	t_0 \leq \frac{1+c_-c_{x}}{c_-y_0} = \frac{1+c_-c_{x}}{c_-x_0}.
\]

\paragraph{Proof of Part (iv).} 
The proof of \eqref{eq:t0-LB-opt-lemma} is quite similar to that of \eqref{eq:t0-UB-opt-lemma}. 
Let us construct another non-negative sequence $\{z_t\}$ as follows
\begin{equation}
	z_0 = x_0 \qquad \text{and} \qquad
	z_{t}= z_{t-1}+c_{+}z_{t-1}^{2}\quad\text{for all }0<t\leq t_0.
	\label{eq:xt-x-tminus1-upper-bound-plus-zt}
\end{equation}
Comparing this with \eqref{eq:xt-x-tminus1-upper-bound-plus} clearly reveals that $z_{t} \geq x_t$. 
Divide both sides of \eqref{eq:xt-x-tminus1-upper-bound-plus-zt} by $z_tz_{t-1}$ to reach
\[
\frac{1}{z_{t-1}} = \frac{1}{z_{t}}+c_{+}\frac{z_{t-1}}{z_{t}}\leq\frac{1}{z_{t}}+c_{+},
\]
where the last inequality is valid since, by construction, $z_t\geq z_{t-1}$. 
Applying this relation recursively yields
\[
\frac{1}{z_{0}}\leq\frac{1}{z_{t_{0}}}+c_{+}t_{0},
\]
which taken together with the fact  $z_0 = x_0 $ and $z_{t_0}\geq x_{t_0}$ leads to
\[
	t_{0}\geq \frac{\frac{1}{z_{0}}-\frac{1}{z_{t_{0}}}}{c_{+}} \geq \frac{\frac{1}{x_{0}}-\frac{1}{x_{t_{0}}}}{c_{+}}. 
\]

\section{Discounted state visitation probability (Lemmas~\ref{lem:facts-d-pi-s-LB}-\ref{lem:facts-d-pi-s})}
\label{sec:estimate-state-visitation}

In this section, we establish our bounds concerning the discounted state visitation probability, 
as claimed in Lemma~\ref{lem:facts-d-pi-s-LB} and Lemma~\ref{lem:facts-d-pi-s}.
Throughout this section, we denote by $\mathbb{P}(\cdot \mymid \pi)$ the probability distribution when policy $\pi$ is adopted. 
Also, we recall that $\mu$ is taken to be a uniform distribution over all states.

\subsection{Lower bounds: proof of Lemma~\ref{lem:facts-d-pi-s-LB}}
\label{sec:proof-lem:facts-d-pi-s-LB}

Consider an arbitrary policy $\pi$, and let $\{s^k\}_{k\geq 0}$ represent an MDP trajectory. 
For any $s\in \{3,\cdots,H\}$, it follows from the definition \eqref{eq:defn-d-rho} of $d_{\mu}^{\pi}$ that
\begin{align}
	d_{\mu}^{\pi}(s) & =(1-\gamma)\sum_{k=0}^{\infty}\gamma^{k}\mathbb{P}\big(s^{k}=s\mymid s^{0}\sim\mu,\pi\big) \label{eq:d-mu-ts-expansion-sk}\\
 & \geq(1-\gamma)\gamma\mathbb{P}\big(s^{1}=s\mymid s^{0}\sim\mu,\pi\big)
   \geq (1-\gamma)\gamma\sum_{s^{\prime}\in\widehat{\mathcal{S}}_{s}}\mathbb{P}\big(s^{1}=s\mymid s^{0}=s^{\prime},\pi \big)\mathbb{P}(s_{0}=s^{\prime} \mymid s^0 \sim \mu) \notag\\
 & =  (1-\gamma)\gamma \cdot \frac{|\widehat{\mathcal{S}}_{s}|}{|\mathcal{S}|} = \cm \gamma (1-\gamma)^{2}. \notag
\end{align}
Here, the penultimate identity is valid due to the construction \eqref{eq:merging-states-definition-P} and the assumption that $\mu$ is uniformly distributed, 
whereas the last identity results from the assumption \eqref{eq:hatS1-hatSH-equal-sizes}. This establishes  \eqref{eq:dmu-t-s-order-range-LB}. Repeating the same argument also reveals that 
\begin{align}
	d_{\mu}^{\pi}(\overline{s}) \geq  \cm \gamma (1-\gamma)^{2}  \notag
\end{align}
for any $\overline{s}\in \{\overline{1},\cdots, \overline{H}\}$, thus validating the lower bound  \eqref{eq:dmu-t-s-bar-order-range-LB}.

In addition, for any $s\in\mathcal{S}_{1}$, the MDP construction \eqref{eq:merging-states-definition-P} allows one to derive
\begin{align*}
d_{\mu}^{\pi}(s) & =(1-\gamma)\sum_{k=0}^{\infty}\gamma^{k}\mathbb{P}\big(s^{k}=s\mymid s^{0}\sim\mu,\pi\big)\geq\gamma(1-\gamma)\mathbb{P}\big(s^{1}=s\mymid s^{0}\sim\mu,\pi\big)\\
 & \ge\gamma(1-\gamma)\mathbb{P}\big(s^{1}=s\mymid s^{0}\in\widehat{\mathcal{S}}_{1}\big)\,\mathbb{P}\big(s_{0}\in\widehat{\mathcal{S}}_{1}\mymid s^{0}\sim\mu\big)\\
 & =\gamma(1-\gamma)\cdot\frac{1}{|\mathcal{S}_{1}|}\cdot\frac{|\widehat{\mathcal{S}}_{1}|}{|\mathcal{S}|}=\gamma(1-\gamma)\frac{\cm}{\cbone}\cdot\frac{1}{|\mathcal{S}|}.
\end{align*}
Here, the last line holds due to the fact that $\mu$ is uniformly distributed and the assumptions \eqref{eq:S1-S2-equal-sizes} and \eqref{eq:hatS1-hatSH-equal-sizes}.  
 We have thus concluded the proof for \eqref{eq:dmu-t-s-12-order-range-LB}.  The proof for \eqref{eq:dmu-t-s-12-order-range-LB-S2} follows from an identical argument and is hence omitted.

\subsection{Upper bounds: proof of Lemma~\ref{lem:facts-d-pi-s}}
\label{sec:proof-lem:facts-d-pi-s}

\subsubsection{Preliminary facts}
\label{sec:proof-lem:facts-d-pi-s-preliminary}

Before embarking on the proof,  we collect several basic yet useful properties that happen when $t<  t_s(\taun_s)$. The first-time readers can proceed directly to Appendix~\ref{sec:proof-upper-bounds-dmu-s}. 
%


\paragraph{Properties about $\soft{Q}^{(t)}(\overline{s},a)$.}

Combine the property \eqref{eq:Qpi-s-adj} in Lemma~\ref{lem:basic-properties-MDP-pi} with \eqref{eq:V-t-sprime-3-H-s} to yield that: for any $1\leq s \leq H$ and any $t< t_s(\taun_s)$, one has
\begin{align}
	& Q^{(t)}(\overline{s}, a_1) = \gamma V^{(t)}(s) < \gamma \taun_s = Q^{(t)}(\overline{s}, a_0) 
	 \label{eq:Qt-monotone-sprime-t-a10} 
\end{align}
%
%
%
In addition, combining the property \eqref{eq:Qpi-s-adj} in Lemma~\ref{lem:basic-properties-MDP-pi} with \eqref{eq:V-t-sprime-3-H} yields: for any $2\leq s \leq H$, 
\begin{align}
	V^{(t)}(\overline{s^{\prime}}) \leq \max\Big\{ Q^{(t)}\big(\overline{s^{\prime}},a_{0}\big),Q^{(t)}\big(\overline{s^{\prime}},a_{1}\big)\Big\}
	= \max\Big\{\gamma\tau_{s^{\prime}},\gamma V^{(t)}(s^{\prime})\Big\} = \gamma\tau_{s^{\prime}}
	\label{eq:Vt-sprime-bar-UB5928}
\end{align}
holds for all $s^{\prime}$ obeying $s \le s^{\prime} \leq H$ and all $t< t_s(\taun_s)$. 
As a remark, \eqref{eq:Qt-monotone-sprime-t-a10} indicates that $a_1$ remains unfavored (according to the current estimate $Q^{(t)}$) 
before the iteration number hits  $t_s(\taun_s)$.


\paragraph{Properties about $\soft{Q}^{(t)}(s+1,a)$ and $\soft{Q}^{(t)}(s+2,a)$.}
First, combining \eqref{eq:Vt-sprime-bar-UB5928} with the relation \eqref{eq:Q-pi-s-a012-identity} reveals that: for any $2\leq s \leq H-1$ and any $t< t_s(\taun_s)$,
\begin{subequations}
\label{eq:Qt-splus1-a012-bounds-123}
\begin{align}
	& Q^{(t)}(s+1,a_{0})  =r_{s+1}+\gamma^{2}p\tau_{s-1} \geq r_{s+1} ,  \\
	& Q^{(t)}(s+1,a_{1})  =\gamma V^{(t)}(\overline{s}) \leq \gamma^{2}\tau_{s} = \gamma^{1/2} r_{s+1}, \\
	& Q^{(t)}(s+1,a_{2})  =r_{s+1}+\gamma pV^{(t)}(\overline{s-1})\geq r_{s+1}  
\end{align}
\end{subequations}
hold as long as $V^{(t)}(\overline{s-1})\geq 0$ (which is guaranteed by Lemma~\ref{lem:non-negativity-PG}). 
Similarly, \eqref{eq:Vt-sprime-bar-UB5928} and \eqref{eq:Q-pi-s-a012-identity} also give
\begin{align*}
	& Q^{(t)}(s+2,a_{0})  =r_{s+2}+\gamma^{2}p\tau_{s} \\
	& Q^{(t)}(s+2,a_{1})  =\gamma V^{(t)}(\overline{s+1}) = \gamma^{2}\tau_{s+1} = \gamma^{1/2} r_{s+2},\\
	& Q^{(t)}(s+2,a_{2})  =r_{s+2}+\gamma pV^{(t)}(\overline{s}) \leq r_{s+2}+\gamma^2 p \tau_s 
\end{align*}
for any $1\leq s \leq H-2$ and any $t< t_s(\taun_s)$. 
Consequently,  we have
\begin{subequations}
\label{eq:Qt-s1-s2-a1-a2-order}
\begin{align}
	Q^{(t)}(s+1,a_{1}) & \leq\min\big\{ Q^{(t)}(s+1,a_{0}),Q^{(t)}(s+1,a_{2}) \big\},\quad &\text{if }2\leq s\leq H-1 \\
	Q^{(t)}(s+2,a_{2}) & \leq Q^{(t)}(s+2,a_{0}), \quad &\text{if }1\leq s\leq H-2
\end{align}
\end{subequations}
for all  $t< t_s(\taun_s)$. 
In other words, the above two inequalities reveal that actions $a_1$ and $a_2$ are perceived as suboptimal (based on the current Q-function estimates) before the iteration count surpasses $t_s(\taun_s)$.

Next, consider any $2\leq s\leq H-1$ and any $ t< t_s(\taun_s)$. 
It has already been shown above that
\begin{subequations}
\label{C1:lem-monotone-relation-pi-q}
\begin{align}
	{Q}^{(t)}(s+1, a) \geq  {Q}^{(t)}(s+1, a_1) ,\qquad a \in \{ a_0, a_2\}.
\end{align}
A similar argument also implies that, for any $ t< t_s(\taun_s)$,
\begin{align}
	{Q}^{(t)}(s+2, a_0) \geq  {Q}^{(t)}(s+2, a_2) , 
\end{align}
\end{subequations}
which forms another property useful for our subsequent analysis.

\subsubsection{Proof of the upper bounds \eqref{eq:dmu-t-s-order-range} and \eqref{eq:dmu-t-s-order-range-bar}}
\label{sec:proof-upper-bounds-dmu-s}


We now turn attention to upper bounding $d_{\mu}^{(t)}(s)$ for any $s\in \{3,\cdots,H\}$. 
%
By virtue of the expansion \eqref{eq:d-mu-ts-expansion-sk}, 
upper bounding $d_{\mu}^{(t)}(s)$ requires controlling $\mathbb{P}\big(s^{k}=s\mymid s^{0}\sim\mu,\pi^{(t)}\big)$ for all $k\geq 0$.   
In light of this, our analysis consists of (i) developing upper bounds on the inter-related quantities $\mathbb{P}\big(s^{k}=s\mymid s^{0}\sim\mu,\pi^{(t)}\big)$ and $\mathbb{P}\big(s^{k}= \overline{s} \mymid s^{0}\sim\mu,\pi^{(t)}\big)$ for any $k\geq 0$, and (ii) combining these upper bounds to control $d_{\mu}^{(t)}(s)$. 
At the core of our analysis is the following upper bounds on the $t$-th policy iterate, which will be established in Appendix~\ref{sec:proof-lem:pi-t-UB-monotone}. 
\begin{lemma}
	\label{lem:pi-t-UB-monotone}
Under the assumption \eqref{eq:assumptions-constants}, for any   $2\leq s\leq H$ and any $t < t_s(\taun_s)$, one has
	\begin{subequations}
	\label{eq:pi-t-UB-monotone}
	\begin{align}
		\pi^{(t)}(a_1\mymid \overline{s})\leq \pi^{(t)}(a_0\mymid \overline{s}) 
		\qquad \text{and} \qquad
		\pi^{(t)}(a_1\mymid \overline{s})\leq 1/2 .
		\label{eq:pi-t-UB-monotone-sbar} 
	\end{align}
	Furthermore,  
	\begin{align}
		\pi^{(t)}(a_1 \mymid s+1) \le \min \big\{ \pi^{(t)}(a_0 \mymid s+1), \pi^{(t)}(a_2 \mymid s+1)  \big\}   
		\quad \text{and} \quad \pi^{(t)}(a_1 \mymid s+1) \le 1/3  \label{eq:pi-t-UB-monotone-splus1} 
	\end{align}
	hold if $2\leq s \leq H-1$, and
	\begin{align}
		\pi^{(t)}(a_2 \mymid s+2) \le  \pi^{(t)}(a_0 \mymid s+2) 
		\qquad \text{and} \qquad
		\pi^{(t)}(a_2 \mymid s+2) \le 1/2  \label{eq:pi-t-UB-monotone-splus2}
	\end{align}
		hold if $1\leq s \leq H-2$. 
	\end{subequations}
	%
\end{lemma}
In words, Lemma~\ref{lem:pi-t-UB-monotone} posits that, at the beginning, the policy iterate $\pi^{(t)}$ does not assign too much probability mass on actions that are currently perceived as suboptimal  (see the remarks in Appendix~\ref{sec:proof-lem:facts-d-pi-s-preliminary}).
With this lemma in place, we are positioned to establish the advertised upper bound.

\paragraph{Step 1: bounding $\mathbb{P}\big(s^{k}=s\mymid s^{0}\sim\mu,\pi^{(t)}\big)$.}
For any $t < t_s(\taun_s)$ and any $s\in \{3,\cdots,H\}$, making use of the upper bound \eqref{eq:pi-t-UB-monotone-sbar} and the MDP construction in Section~\ref{sec:mdp-construction} yields
\begin{subequations}
\begin{align*}
\mathbb{P}\big(s^{0}=s\mymid s^{0}\sim\mu,\pi^{(t)}\big) & =1/{|\cS|},\\
	\mathbb{P}\big(s^{1}=s\mymid s^{0}\sim\mu,\pi^{(t)}\big) & \leq \mathbb{P}\big(s^{0}\in\widehat{\mathcal{S}}_{s}\big) +\pi^{(t)}(a_{1}\mymid\overline{s})\,\mathbb{P}(s^{0}=\overline{s}) \leq \frac{|\widehat{\mathcal{S}}_{s}|}{|\cS|}+\frac{1}{2} \cdot \frac{1}{|\cS|} \leq \frac{2|\widehat{\mathcal{S}}_{s}|}{|\cS|}  = 2\cm (1-\gamma), \\
\mathbb{P}\big(s^{k}=s\mymid s^{0}\sim\mu,\pi^{(t)}\big) & =\pi^{(t)}(a_{1}\mymid\overline{s})\,\mathbb{P}\big(s^{k-1}=\overline{s}\mymid s^{0}\sim\mu,\pi^{(t)}\big)\leq\frac{1}{2}\mathbb{P}\big(s^{k-1}=\overline{s}\mymid s^{0}\sim\mu,\pi^{(t)}\big)
\end{align*}
\end{subequations}
for all $k\geq 2$. Note that the above calculation exploits the fact that $\mu$ is a uniform distribution.

\paragraph{Step 2: bounding $\mathbb{P}\big(s^{k}= \overline{s} \mymid s^{0}\sim\mu,\pi^{(t)}\big)$.}
Given that $\mu$ is a uniform distribution, one has
\begin{subequations}
\label{eq:P-s012-sbar-1234}
\begin{align}
\mathbb{P}\big(s^{0}=\overline{s}\mymid s^{0}\sim\mu,\pi^{(t)}\big) & =1/|\cS|   \label{eq:P-s0-sbar-1234}
\end{align}
for any $s\in \cS$. 
With \eqref{eq:pi-t-UB-monotone-splus1} and \eqref{eq:pi-t-UB-monotone-splus2} in mind,  
the MDP construction in Section~\ref{sec:mdp-construction} allows one to show that
\begin{align}
\mathbb{P}\big(s^{1}=\overline{s}\mymid s^{0}\sim\mu,\pi^{(t)}\big) & \leq\mathbb{P}\big(s^{0}\in\widehat{\mathcal{S}}_{\overline{s}}\big)+\pi^{(t)}(a_{1}\mymid s+1)\,\mathbb{P}(s^{0}=s+1)+\pi^{(t)}(a_{2}\mymid s+2)\,\mathbb{P}(s^{0}=s+2) \notag\\
	& \le\frac{|\widehat{\mathcal{S}}_{\overline{s}}|}{|\cS|}+\frac{1}{3|\cS|}+\frac{1}{2|\cS|} \leq \frac{2|\widehat{\mathcal{S}}_{\overline{s}}|}{|\cS|}  = 2\cm (1-\gamma)
	\label{eq:P-s1-sbar-1234}
\end{align}
holds for any  $2\leq s \leq H-2$ and any $t< t_s(\taun_s)$,  and in addition,
\begin{align}
& \mathbb{P}\big(s^{k} =\overline{s}\mymid s_{0}\sim\mu,\pi^{(t)}\big) \nonumber\\
&\qquad \le\pi^{(t)}(a_{1}\mymid s+1)\,\mathbb{P}\big(s^{k-1}=s+1\mymid s^{0}\sim\mu,\pi^{(t)}\big)+\pi^{(t)}(a_{2}\mymid s+2)\,\mathbb{P}\big(s^{k-1}=s+2\mymid s^{0}\sim\mu,\pi^{(t)}\big) \notag\\
 &\qquad \leq\frac{1}{3}\mathbb{P}\big(s^{k-1}=s+1\mymid s_{0}\sim\mu,\pi^{(t)}\big)+\frac{1}{2}\mathbb{P}\big(s^{k-1}=s+2\mymid s^{0}\sim\mu,\pi^{(t)}\big)
	\label{eq:P-s2-sbar-1234}
\end{align}
\end{subequations}
hold for any $k\geq 2$, $2\leq s \leq H-2$, and any $t< t_s(\taun_s)$. 
Moreover,  invoking \eqref{eq:pi-t-UB-monotone-splus1} and the MDP construction once again reveals that
\begin{align*}
\mathbb{P}\big(s^{k}=\overline{H-1}\mymid s^{0}\sim\mu,\pi^{(t)}\big) & \le\pi^{(t)}(a_{1}\mymid H)\,\mathbb{P}\big(s^{k-1}=H\mymid s^{0}\sim\mu,\pi^{(t)}\big)\leq\frac{1}{3}\mathbb{P}\big(s^{k-1}=H\mymid s^{0}\sim\mu,\pi^{(t)}\big) \\
\mathbb{P}\big(s^{k}=\overline{H}\mymid s^{0}\sim\mu,\pi^{(t)}\big) & =0
\end{align*}
hold for any $k\geq 2$ and any $t< t_s(\taun_s)$. In addition, it is seen that
\begin{align*}
\mathbb{P}\big(s^{1}=\overline{H-1}\mymid s^{0}\sim\mu,\pi^{(t)}\big) & \leq\mathbb{P}\big(s^{0}\in\widehat{\mathcal{S}}_{\overline{H-1}}\big)+\mathbb{P}(s^{0}=H)
	= \frac{\big|\widehat{\mathcal{S}}_{\overline{H-1}}\big|}{|\cS|}+\frac{1}{|\cS|}\leq\frac{2\big|\widehat{\mathcal{S}}_{\overline{H-1}}\big|}{|\cS|}=2\cm(1-\gamma),\notag\\
\mathbb{P}\big(s^{1}=\overline{H}\mymid s^{0}\sim\mu,\pi^{(t)}\big) & \leq\mathbb{P}\big(s^{0}\in\widehat{\mathcal{S}}_{\overline{H}}\big)=\frac{\big|\widehat{\mathcal{S}}_{\overline{H}}\big|}{|\cS|}=\cm(1-\gamma).\notag
\end{align*}

\paragraph{Step 3: putting all this together.}

Combining the preceding upper bounds on both $\mathbb{P}\big(s^{k}= s \mymid s^{0}\sim\mu,\pi^{(t)}\big)$ and $\mathbb{P}\big(s^{k}= \overline{s} \mymid s^{0}\sim\mu,\pi^{(t)}\big)$ $(k\geq 1$) and recognizing the monotonicity property \eqref{eq:monotonicity-t-s-tau-s}, we immediately arrive at the following crude bounds
\begin{align*}
\max_{3\leq s\leq H, t< t_s(\taun_s)}\Big\{\mathbb{P}\big(s^{0}=s\mymid s^{0}\sim\mu,\pi^{(t)}\big),\mathbb{P}\big(s^{1}=s\mymid s^{0}\sim\mu,\pi^{(t)}\big)\Big\} & \leq 1/|\cS| \leq 2\cm(1-\gamma)\\
\max_{2\leq s\leq H, t< t_s(\taun_s)}\Big\{\mathbb{P}\big(s^{0}=\overline{s}\mymid s^{0}\sim\mu,\pi^{(t)}\big),\mathbb{P}\big(s^{1}=\overline{s}\mymid s^{0}\sim\mu,\pi^{(t)}\big)\Big\} & \leq 2\cm(1-\gamma)\\
\max_{3\leq s\leq H, t< t_s(\taun_s)}\mathbb{P}\big(s^{k}=s\mymid s^{0}\sim\mu,\pi^{(t)}\big) & \leq\frac{5}{6}\max_{2\leq s\leq H, t< t_s(\taun_s)}\mathbb{P}\big(s^{k-1}=\overline{s}\mymid s^{0}\sim\mu,\pi^{(t)}\big)\\
\max_{2\leq s\leq H, t< t_s(\taun_s)}\mathbb{P}\big(s^{k}=\overline{s}\mymid s^{0}\sim\mu,\pi^{(t)}\big) & \leq\frac{5}{6}\max_{3\leq s\leq H, t< t_s(\taun_s)}\mathbb{P}\big(s^{k-1}=s\mymid s^{0}\sim\mu,\pi^{(t)}\big)
\end{align*}
for any $k\geq 2$. It is then straightforward to deduce  that 
\begin{subequations}
\begin{align}
	\max_{3\leq s\leq H, t< t_s(\taun_s)}\mathbb{P}\big(s^{k}=s\mymid s^{0}\sim\mu,\pi^{(t)}\big) & \leq\left(\frac{5}{6}\right)^{k-1} 2\cm(1-\gamma) \label{eq:P-3-s-H-Psk-s-UB}\\
	\max_{2\leq s\leq H, t< t_s(\taun_s)}\mathbb{P}\big(s^{k}=\overline{s}\mymid s^{0}\sim\mu,\pi^{(t)}\big) & \leq\left(\frac{5}{6}\right)^{k-1} 2\cm(1-\gamma)
	\label{eq:P-3-s-H-Psk-sbar-UB}
\end{align}
\end{subequations}
for any $k\geq 1$. In turn, these bounds give rise to
\begin{subequations}
\label{eq:dmu-ts-intermediate-UB-14}
\begin{align}
d_{\mu}^{(t)}(s) & =(1-\gamma)\sum_{k=0}^{\infty}\gamma^{k}\mathbb{P}\big(s^{k}=s\mymid s^{0}\sim\mu,\pi^{(t)}\big)
	\leq(1-\gamma)\left\{ 2\cm(1-\gamma)+\sum_{k=1}^{\infty} \left(\frac{5}{6}\right)^{k-1} 2\cm(1-\gamma) \right\} \notag\\
 	& \leq2\cm(1-\gamma)^{2}+\frac{1}{1-5/6}\cdot2\cm(1-\gamma)^{2}=14\cm(1-\gamma)^{2}
	\label{eq:dmu-ts-original-intermediate-UB-14}
\end{align}
for any $3\leq s \leq H$  and any $t< t_s(\taun_s)$. 
This establishes the claimed upper bound \eqref{eq:dmu-t-s-order-range} as long as Lemma~\ref{lem:pi-t-UB-monotone} is valid. 
Further, replacing $s$ with $\overline{s}$ in \eqref{eq:dmu-ts-intermediate-UB-14} also reveals that
\begin{align}
	d_{\mu}^{(t)}(\overline{s}) \leq 14\cm(1-\gamma)^{2} 
	\label{eq:dmu-tsbar-intermediate-UB-14}
\end{align}
\end{subequations}
for any $2\leq s \leq H$  and any $t< t_s(\taun_s)$, thus concluding the proof of \eqref{eq:dmu-t-s-order-range-bar}.

%


\subsubsection{Proof of the upper bound  \eqref{eq:dmu-t-s-2-order-range} }

We now consider any $s\in \mathcal{S}_2$. From our MDP construction, we have
\begin{subequations}
\begin{align*}
\mathbb{P}\big(s^{0}=s\mymid s^{0}\sim\mu,\pi^{(t)}\big) & = 1 / |\cS|,\\
	\mathbb{P}\big(s^{1}=s\mymid s^{0}\sim\mu,\pi^{(t)}\big) & \le \mathbb{P}\big(s^{1}=s\mymid s^{0}\in\widehat{\mathcal{S}}_{2},\pi^{(t)}\big) \mathbb{P}\big(s^{0}\in\widehat{\mathcal{S}}_{2}\big)+\mathbb{P}\big(s^{1}=s\mymid s^{0}=\overline{2},\pi^{(t)}\big)\,\mathbb{P}\big(s^{0}=\overline{2}\big) \\
	& \leq \frac{1}{|\cS_2|} \frac{|\widehat{\mathcal{S}}_{2}|}{|\cS|}+\frac{1}{|\cS_{2}|}\, \frac{1}{|\cS|}
	\leq \frac{2|\widehat{\mathcal{S}}_{2}|}{|\cS_2|\, |\cS|} 
	= \frac{2\cm}{\cbtwo |\cS|},\\
\mathbb{P}\big(s^{k}=s\mymid s^{0}\sim\mu,\pi^{(t)}\big) & \le\mathbb{P}\big(s^{k}=s\mymid s^{k-1}=\overline{2},\pi^{(t)}\big)\,\mathbb{P}\big(s^{k-1}=\overline{2}\mymid s^{0}\sim\mu,\pi^{(t)}\big)\\
 & \leq\frac{1}{2|\cS_{2}|} \mathbb{P}\big(s^{k-1}=\overline{2}\mymid s^{0}\sim\mu,\pi^{(t)}\big)
\end{align*}
\end{subequations}
for any $k\geq 2$ and any $s\in \cS_2$. 
In addition, our bound in \eqref{eq:P-3-s-H-Psk-sbar-UB} gives
%
\begin{align*}
	\mathbb{P}\big(s^{k-1}=\overline{2}\mymid s^{0}\sim\mu,\pi^{(t)}\big) 
	\leq \left(\frac{5}{6}\right)^{k-2} 2\cm (1-\gamma)
\end{align*}
for any $k\geq 2$ and any $t<t_2(\taun_2)$.
Consequently, we arrive at
\begin{align}
	\mathbb{P}\big(s^{k}=s\mymid s^{0}\sim\mu,\pi^{(t)}\big) 
 	& \leq \frac{1}{|\cS_{2}|} \mathbb{P}\big(s^{k-1}=\overline{2}\mymid s^{0}\sim\mu,\pi^{(t)}\big) 
	\leq \frac{\cm (1-\gamma)}{|\cS_{2}|} \left( \frac{5}{6} \right)^{k-2}  = \frac{\cm }{\cbtwo |\cS|} \left( \frac{5}{6} \right)^{k-2}.
\end{align}
Armed with the preceding inequalities, we can derive
\begin{align*}
d_{\mu}^{(t)}(s) & =(1-\gamma)\sum_{k=0}^{\infty}\gamma^{k}\mathbb{P}\big(s^{k}=s\mymid s^{0}\sim\mu,\pi^{(t)}\big)\\
 & \leq(1-\gamma)\left\{ \frac{1}{|\cS|}+\gamma\cdot\frac{2\cm}{\cbtwo |\cS|}+\sum_{k=2}^{\infty}\gamma^{k}\frac{\cm }{\cbtwo |\cS|}  \left(\frac{5}{6}\right)^{k-2}\right\} \\
	& \leq\frac{1-\gamma}{|\cS|}\left(1+\frac{2\cm}{\cbtwo}\right)+\frac{\cm(1-\gamma)}{(1-5/6) \cbtwo |\cS|}
	= \frac{1-\gamma}{|\cS|}\left(1+\frac{8\cm}{\cbtwo}\right)
\end{align*}
for any $s\in \cS_2$ and any $t< t_2(\taun_2)$, thus concluding the advertised upper bound for $s\in \cS_2$.

\subsubsection{Proof of the upper bound  \eqref{eq:dmu-t-s-1-order-range-bar} }

It follows from our MDP construction that
\begin{align*}
\mathbb{P}\big(s^{0}=\overline{1}\mymid s^{0}\sim\mu,\pi^{(t)}\big) & =1/|\cS| ,\\
\mathbb{P}\big(s^{1}=\overline{1}\mymid s^{0}\sim\mu,\pi^{(t)}\big) & \le\mathbb{P}\big(s^{0}\in\widehat{\cS}_{\overline{1}}\big)+\mathbb{P}\big(s^{0}=3\big)
	=\frac{|\widehat{\cS}_{\overline{1}}|}{|\cS|}+\frac{1}{|\cS|}.
\end{align*}
Moreover,  for any $k\geq 2$ and any $t<t_3(\taun_3)$, one can derive
\begin{align}
	\mathbb{P}\big(s^{k}=\overline{1}\mymid s^{0}\sim\mu,\pi^{(t)}\big) & =\pi^{(t)}(a_{2}\mymid3)\,\mathbb{P}\big(s^{k-1}=3\mymid s^{0}\sim\mu,\pi^{(t)}\big)
	\leq\left(\frac{5}{6}\right)^{k-2}2\cm(1-\gamma),
	\label{eq:P-sk-1bar-UB-5284}
\end{align}
where the last inequality arises from \eqref{eq:P-3-s-H-Psk-s-UB}. Putting these bounds together leads to
\begin{align*}
d_{\mu}^{(t)}(\overline{1}) & =(1-\gamma)\sum_{k=0}^{\infty}\gamma^{k}\mathbb{P}\big(s^{k}=\overline{1}\mymid s^{0}\sim\mu,\pi^{(t)}\big)\leq(1-\gamma)\left\{ \frac{1}{|\cS|}+\gamma\left(\frac{|\widehat{\cS}_{\overline{1}}|}{|\cS|}+\frac{1}{|\cS|}\right)+\sum_{k=2}^{\infty}\left(\frac{5}{6}\right)^{k-2}2\cm(1-\gamma)\right\} \notag\\
 & \leq(1-\gamma)\left\{ \frac{2|\widehat{\cS}_{\overline{1}}|}{|\cS|}+\frac{1}{1-5/6}2\cm(1-\gamma)\right\} =14\cm(1-\gamma)^{2},
\end{align*}
where we have used the assumption that $|\widehat{\cS}_{\overline{1}}|=\cm (1-\gamma)|\cS|$. 
When $t<t_2(\taun_2)$, the monotonicity property \eqref{eq:monotonicity-t-s-tau-s} indicates that $t<t_3(\taun_3)$, thus concluding the proof of \eqref{eq:dmu-t-s-1-order-range-bar}.

\subsubsection{Proof of the upper bound  \eqref{eq:dmu-t-s-1-order-range} }

In view of our MDP construction, for any $s \in \cS_1$ and any $t<\min\{t_1(\taun_1),t_2(\taun_2) \}$ we have
\begin{subequations}
\begin{align*}
\mathbb{P}\big(s^{0}=s\mymid s^{0}\sim\mu,\pi^{(t)}\big) & =1/|\cS|,\\
\mathbb{P}\big(s^{1}=s\mymid s^{0}\sim\mu,\pi^{(t)}\big) & \le\mathbb{P}\big(s^{1}=s\mymid s^{0}\in\widehat{\mathcal{S}}_{1},\pi^{(t)}\big)\mathbb{P}\big(s^{0}\in\widehat{\mathcal{S}}_{1}\big)+\mathbb{P}\big(s^{1}=s\mymid s^{0}=\overline{1},\pi^{(t)}\big)\,\mathbb{P}\big(s^{0}=\overline{1}\big)\\
 & \leq\frac{1}{|\cS_{1}|}\frac{|\widehat{\mathcal{S}}_{1}|}{|\cS|}+\frac{1}{|\cS_{1}|}\,\frac{1}{|\cS|}\leq\frac{1}{|\cS_{1}|}\frac{2|\widehat{\mathcal{S}}_{1}|}{|\cS|}=\frac{2\cm}{\cbone|\cS|},\\
\mathbb{P}\big(s^{k}=s\mymid s^{0}\sim\mu,\pi^{(t)}\big) & \le\mathbb{P}\big(s^{k}=s\mymid s^{k-1}=\overline{1},\pi^{(t)}\big)\,\mathbb{P}\big(s^{k-1}=\overline{1}\mymid s^{0}\sim\mu,\pi^{(t)}\big)\\
 & \leq\frac{1}{|\cS_{1}|}\mathbb{P}\big(s^{k-1}=\overline{1}\mymid s^{0}\sim\mu,\pi^{(t)}\big)\leq\frac{2\cm}{\cbone|\cS|}\left(\frac{5}{6}\right)^{k-3},
\end{align*}
\end{subequations}
where $k$ is any integer obeying $k\geq 2$. 
Here, the last inequality comes from \eqref{eq:P-sk-1bar-UB-5284}. 
These bounds taken collectively demonstrate that
\begin{align*}
d_{\mu}^{(t)}(s) & =(1-\gamma)\sum_{k=0}^{\infty}\gamma^{k}\mathbb{P}\big(s^{k}=s\mymid s^{0}\sim\mu,\pi^{(t)}\big)\\
 & \leq(1-\gamma)\left\{ \frac{1}{|\cS|}+\gamma\cdot\frac{2\cm}{\cbone |\cS|}+\sum_{k=2}^{\infty}\gamma^{k}\frac{2\cm}{\cbone|\cS|}  \left(\frac{5}{6}\right)^{k-3}\right\} \\
	& \leq\frac{1-\gamma}{|\cS|}\left(1+\frac{2\cm}{\cbone}\right)+\frac{\frac{6}{5}\cdot 2\cm(1-\gamma)}{(1-5/6)\cbone |\cS|}
	\le \frac{1-\gamma}{|\cS|}\left(1+\frac{17\cm}{\cbone}\right)
\end{align*}
for any $s \in \cS_1$ and any $t<\min\{t_1(\taun_1),t_2(\taun_2) \}$. This completes the proof.

%

%
%
%
%
%

\subsubsection{Proof of Lemma~\ref{lem:pi-t-UB-monotone}}
\label{sec:proof-lem:pi-t-UB-monotone}

In order to prove this lemma, we are in need of the following auxiliary result, whose proof can be found in Appendix~\ref{sec:proof-lemma:monotone-relation-pi-q}.
\begin{lemma}
\label{lem:monotone-relation-pi-q}
	Consider any state $1\leq s \leq H $. Suppose that $ 0 <\eta \leq (1-\gamma)/2$. 
	%
	\begin{itemize}
		\item[(i)] If the following conditions
		%
		\begin{align*}
			\soft{Q}^{(t)}(s, a_0) &- \soft{Q}^{(t)}(s, a_1) \ge 0,\quad \quad\quad  \soft{Q}^{(t)}(s, a_2) - \soft{Q}^{(t)}(s, a_1) \ge 0 \\
			 & \pi^{(t-1)}(a_{1} \mymid s ) \leq  \min \big\{  \pi^{(t-1)}(a_{0} \mymid s ), \pi^{(t-1)}(a_{2} \mymid s )  \big\}
		\end{align*}
		hold, then one has $\pi^{(t)}(a_{1} \mymid s )\leq 1/3$ and $\pi^{(t)}(a_{1} \mymid s ) \leq  \min \big\{  \pi^{(t)}(a_{0} \mymid s ), \pi^{(t)}(a_{2} \mymid s )  \big\}$. 
		\item[(ii)] If the following conditions
		%
		\begin{align*}
			\soft{Q}^{(t)}(s, a_0) - \soft{Q}^{(t)}(s, a_2) \ge 0
			\quad \text{and} \quad
			\pi^{(t-1)}(a_{2} \mymid s ) \leq  \pi^{(t-1)}(a_{0} \mymid s )
		\end{align*}
		hold, then one has $\pi^{(t)}(a_{2} \mymid s )\leq 1/2$ and $\pi^{(t)}(a_{2} \mymid s ) \leq  \pi^{(t)}(a_{0} \mymid s )$. 
		\item[(iii)] If the following conditions
		%
		\begin{align*}
			\soft{Q}^{(t)}(\overline{s}, a_0) - \soft{Q}^{(t)}(\overline{s}, a_1) \ge 0
			\quad \text{and} \quad
			\pi^{(t-1)}(a_{1} \mymid \overline{s} ) \leq  \pi^{(t-1)}(a_{0} \mymid \overline{s} )
		\end{align*}
		hold, then one has $\pi^{(t)}(a_{1} \mymid \overline{s} )\leq 1/2$ and $\pi^{(t)}(a_{1} \mymid \overline{s} ) \leq  \pi^{(t)}(a_{0} \mymid \overline{s} )$.			
	\end{itemize}
\end{lemma}
	\begin{remark} In words, Lemma~\ref{lem:monotone-relation-pi-q} develops nontrivial upper bounds on the policy associated with actions that are currently perceived as suboptimal. As we shall see, such upper bounds --- which are strictly below 1 --- translate to some contraction factors that enable the advertised result of this lemma. 
	\end{remark}

 With Lemma~\ref{lem:monotone-relation-pi-q} in place, we proceed to prove Lemma~\ref{lem:pi-t-UB-monotone} by induction. 
Let us start from the base case with $t=0$.  
Given that the initial policy is chosen to be uniformly distributed, we have 
\begin{align*}
	\pi^{(0)}(a_1\mymid s) & = \pi^{(0)}(a_0\mymid s) = \pi^{(0)}(a_2\mymid s),   && 3\leq s\leq H; \\
	\pi^{(0)}(a_1\mymid \overline{s}) & = \pi^{(0)}(a_0\mymid \overline{s}),  && 1\leq s\leq H.
\end{align*}
Therefore, the claim~\eqref{eq:pi-t-UB-monotone} trivially holds for $t=0$.

Next, we move on to the induction step. 
Suppose that the induction hypothesis~\eqref{eq:pi-t-UB-monotone} holds for the $t$-th iteration, 
and we intend to establish it for the $(t+1)$-th iteration.  
Apply Lemma~\ref{lem:monotone-relation-pi-q} with Conditions \eqref{eq:Qt-monotone-sprime-t-a10} and \eqref{eq:pi-t-UB-monotone-sbar} to yield
\[
	\pi^{(t+1)}(a_1\mymid \overline{s})\leq \pi^{(t+1)}(a_0\mymid \overline{s}) 
\]
%
with the proviso that $0<\eta \leq (1-\gamma)/2$. Clearly, this also implies that $\pi^{(t+1)}(a_1\mymid \overline{s})\leq 1/2$.  
Further, invoke Lemma~\ref{lem:monotone-relation-pi-q} once again with Condition~\eqref{C1:lem-monotone-relation-pi-q} and the induction hypothesis~\eqref{eq:pi-t-UB-monotone} 
to arrive at
\begin{align*}
	&\pi^{(t+1)}(a_1 \mymid s+1) \le \min \big\{ \pi^{(t+1)}(a_0 \mymid s+1), \pi^{(t+1)}(a_2 \mymid s+1)  \big\} , && \text{if }2\leq s\leq H-1;\\
	&\pi^{(t+1)}(a_2 \mymid s+2) \le  \pi^{(t+1)}(a_0 \mymid s+2) , && \text{if }1\leq s\leq H-2.
\end{align*}
A straightforward consequence is $\pi^{(t+1)}(a_1 \mymid s+1) \le 1/3$ and $\pi^{(t+1)}(a_2 \mymid s+2) \le 1/2$. 
The proof is thus complete by induction.

\subsubsection{Proof of Lemma~\ref{lem:monotone-relation-pi-q}}
\label{sec:proof-lemma:monotone-relation-pi-q}

First of all, suppose that $\soft{Q}^{(t)}(s, a_0) - \soft{Q}^{(t)}(s, a_1) \ge 0$ and $\pi^{(t-1)}(a_{0}\mymid s)\geq\pi^{(t-1)}(a_{1}\mymid s)$ hold true. 
%
Combining this result with the PG update rule \eqref{eq:PG-update-all}  gives 
\begin{align}
\theta^{(t)}(s,a_{1}) & =\theta^{(t-1)}(s,a_{1})+\frac{\eta}{1-\gamma}d_{\mu}^{(t-1)}(s)\,\pi^{(t-1)}(a_{1}\mymid s)\,\soft{A}^{(t-1)}(s,a_{1})\nonumber\\
 & \leq\theta^{(t-1)}(s,a_{1})+\frac{\eta}{1-\gamma}d_{\mu}^{(t-1)}(s)\,\pi^{(t-1)}(a_{1}\mymid s)\,\soft{A}^{(t-1)}(s,a_{0}). \notag
\end{align}
%
%
%
%
Consequently, applying this inequality  and using the PG update rule \eqref{eq:PG-update-all} yield
\begin{align}
 & \theta^{(t)}(s,a_{1})-\theta^{(t)}(s,a_{0})\nonumber\\
 & \qquad \leq\theta^{(t-1)}(s,a_{1})+\frac{\eta}{1-\gamma}d_{\mu}^{(t-1)}(s)\,\pi^{(t-1)}(a_{1}\mymid s)\,\soft{A}^{(t-1)}(s,a_{0})\nonumber\\
 & \qquad\quad\qquad-\theta^{(t-1)}(s,a_{0})-\frac{\eta}{1-\gamma}d_{\mu}^{(t-1)}(s)\,\pi^{(t-1)}(a_{0}\mymid s)\,\soft{A}^{(t-1)}(s,a_{0})\nonumber\\
 & \qquad \leq \Big\{\theta^{(t-1)}(s,a_{1})-\theta^{(t-1)}(s,a_{0})\Big\}+\Big\{\pi^{(t-1)}(a_{0}\mymid s)-\pi^{(t-1)}(a_{1}\mymid s)\Big\}\,\Big|\frac{\eta}{1-\gamma}d_{\mu}^{(t-1)}(s)\soft{A}^{(t-1)}(s,a_{0})\Big|,
	\label{eq:theta-t-s-a10-bound}
\end{align}
%
where the last line arises by combining terms and invoking the assumption $\pi^{(t-1)}(a_{0}\mymid s) \geq \pi^{(t-1)}(a_{1}\mymid s)$.


Additionally, it is seen from the definition of the advantage function that
\begin{equation}
	\big|A^{(t-1)}(s,a_{0})\big| \leq \max_{\pi,a}\big|Q^{\pi}(s,a)\big|+  \max_{\pi}\big|V^{\pi}(s)\big|\leq 2, 
	\label{eq:advantage-function-t-1-UB}
\end{equation}
where the last inequality follows from Lemma~\ref{lem:basic-properties-MDP-Vstar}. 
Recognizing that $d^{(t-1)}_{\mu}(s) \leq 1$, one obtains
%
\begin{align} \label{eq:d-eta-A-UB-12}
	\Big|\frac{\eta}{1-\gamma}d_{\mu}^{(t-1)}(s)A^{(t-1)}(s,a_{0})\Big|\leq\frac{\eta}{1-\gamma}\cdot 2  \leq 1,
\end{align}
with the proviso that $0<\eta\leq (1-\gamma)/2$. 

Substituting \eqref{eq:d-eta-A-UB-12} into \eqref{eq:theta-t-s-a10-bound} then yields
%
\begin{align}
	\eqref{eq:theta-t-s-a10-bound} 
	& \leq \Big\{ \theta^{(t-1)}(s,a_{1})-\theta^{(t-1)}(s,a_{0}) \Big\} + \Big\{\pi^{(t-1)}(a_{0}\mymid s)-\pi^{(t-1)}(a_{1}\mymid s)\Big\} \notag\\
	& \leq \Big\{ \theta^{(t-1)}(s,a_{1})-\theta^{(t-1)}(s,a_{0}) \Big\} - \Big\{ \theta^{(t-1)}(s, a_{0})-\theta^{(t-1)}(s, a_{1})\Big\}=0,
	\label{eq:theta-t-s-a10-bound7984}
\end{align}
where both the first line and the last identity rely on the fact that $\theta^{(t-1)}(s,a_{1})\leq \theta^{(t-1)}(s,a_{0})$ --- an immediate consequence of the assumption $\pi^{(t-1)}(a_{1} \mymid s)\leq \pi^{(t-1)}(a_{0}\mymid s)$. 
To see why the inequality \eqref{eq:theta-t-s-a10-bound7984} holds, it suffices to make note of the following consequence of softmax parameterization: 
\begin{align*}
\pi^{(t-1)}(a_{0}\mymid s)-\pi^{(t-1)}(a_{1}\mymid s) & =\pi^{(t-1)}(a_{1}\mymid s)\Big\{\exp\big[\theta^{(t-1)}(s,a_{0})-\theta^{(t-1)}(s,a_{1})\big]-1\Big\}\\
	& \overset{\mathrm{(a)}}{\leq} \frac{\exp\big[\theta^{(t-1)}(s,a_{0})-\theta^{(t-1)}(s,a_{1})\big]-1}{\exp\big[\theta^{(t-1)}(s,a_{0})-\theta^{(t-1)}(s,a_{1})\big]+1}\\
 & \overset{\mathrm{(b)}}{\leq} \theta^{(t-1)}(s,a_{0})-\theta^{(t-1)}(s,a_{1}), 
\end{align*}
where (b) follows since $\frac{e^x-1}{e^x+1}\leq x$ for all $x\geq 0$, and the validity of (a) is guaranteed since
\begin{align*}
\pi^{(t-1)}(a_{1}\mymid s) & =\frac{\exp\big(\theta^{(t-1)}(s,a_{1})\big)}{\sum_{a}\exp\big(\theta^{(t-1)}(s,a)\big)}\leq\frac{\exp\big(\theta^{(t-1)}(s,a_{1})\big)}{\exp\big(\theta^{(t-1)}(s,a_{0})\big)+\exp\big(\theta^{(t-1)}(s,a_{1})\big)}\\
 & =\frac{1}{\exp\big[\theta^{(t-1)}(s,a_{0})-\theta^{(t-1)}(s,a_{1})\big]+1}. 
\end{align*}
To conclude, the above result \eqref{eq:theta-t-s-a10-bound7984} implies that
\begin{align}
	\pi^{(t)}(a_{0}\mymid s) \geq \pi^{(t)}(a_{1}\mymid s). 
	\label{eq:pi-t-s-a10-bound7984}
\end{align}

Repeating the above argument immediately reveals that: if 
\[
	 Q^{(t-1)}(s,a_{2}) \geq Q^{(t-1)}(s,a_{1})
	\quad \text{and} \quad 
	\pi^{(t-1)}(a_{2} \mymid s ) \geq \pi^{(t-1)}(a_{1} \mymid s ),
\]
then one has $\pi^{(t)}(a_{2} \mymid s )  \geq \pi^{(t)}(a_{1} \mymid s ) $, which together with \eqref{eq:pi-t-s-a10-bound7984} indicates that 
\[
	\pi^{(t)}(a_{1} \mymid s ) \leq \min \big\{ \pi^{(t)}(a_{0} \mymid s ) , \pi^{(t)}(a_{2} \mymid s ) \big\}
\]
\[
	\Longrightarrow \qquad \pi^{(t)}(a_{1} \mymid s )\leq \frac{ \pi^{(t)}(a_{0} \mymid s )+ \pi^{(t)}(a_{1} \mymid s )+ \pi^{(t)}(a_{2} \mymid s )}{3}= \frac{1}{3}.
\]
This establishes Part (i) of Lemma~\ref{lem:monotone-relation-pi-q}.

The proofs of Parts (ii) and (iii) follow from exactly the same argument as for Part (i), and are hence omitted for the sake of brevity.

\section{Crossing times of the first few states (Lemma~\ref{lem:init-step})}
\label{sec:proof-lem:init-step}

This section presents the proof of Lemma~\ref{lem:init-step} regarding the crossing times w.r.t.~$\cS_1$, $\cS_2$, and state $\overline{1}$.

\subsection{Crossing times for the buffer states in $\cS_1$ and $\cS_2$}

We first present the proof of the relation \eqref{eqn:t1-t2-scaling} regarding several quantities about $t_1$ and $t_2$. 

\paragraph{Step 1: characterize the policy gradients.}

Our analysis largely relies on understanding the policy gradient dynamics, towards which we need to first characterize the gradient. 
Recalling that the gradient of $V^{(t)}$ w.r.t.~$\theta_t(1, a_1)$ (cf.~\eqref{eq:policy-grad-softmax-expression}) is given by
\begin{align}
\notag \frac{\partial V^{(t)}(\mu)}{\partial \theta(1, a_1)} 
&= \frac{1}{1-\gamma} d^{(t)}_{\mu}(1) \pi^{(t)}(a_1 \mymid 1) \Big\{ Q^{(t)}(1, a_1) - V^{(t)}(1) \Big\} \\
&= \frac{1}{1-\gamma} d^{(t)}_{\mu}(1) \pi^{(t)}(a_1 \mymid 1) \Big\{ Q^{(t)}(1, a_1) - \pi^{(t)}(a_0 \mymid 1) Q^{(t)}(1, a_0) - \pi^{(t)}(a_1 \mymid 1) Q^{(t)}(1, a_1) \Big\} \notag\\
\notag 
&= \frac{1}{1-\gamma}d^{(t)}_{\mu}(1)\pi^{(t)}(a_1 \mymid 1)\pi^{(t)}(a_0 \mymid 1) \Big\{ Q^{(t)}(1, a_1) - Q^{(t)}(1, a_0) \Big\}\\
&= \frac{2\gamma^2}{1-\gamma}d^{(t)}_{\mu}(1)\pi^{(t)}(a_1 \mymid 1)\pi^{(t)}(a_0 \mymid 1) > 0, 
\label{eqn:gradient-simple-dvorak}
\end{align}
where in the last step we use $Q^{(t)}(1, a_1) - Q^{(t)}(1, a_0) = 2\gamma^{2}$ (see \eqref{eq:Q-pi-s12-a0-S1S2}). 
The same calculation also yields
\begin{align}
\label{eqn:gradient-simple-dvorak-V2}
\frac{\partial V^{(t)}(\mu)}{\partial \theta(2, a_1)} 
&= \frac{2\gamma^4}{1-\gamma}d^{(t)}_{\mu}(2)\pi^{(t)}(a_1 \mymid 2)\pi^{(t)}(a_0 \mymid 2) > 0.  
\end{align}
As an immediate consequence, the PG update rule~\eqref{eq:PG-update} reveals that
both $\theta^{(t)}(1,a_1)$ (resp.~$\pi^{(t)}(1,a_1)$) and $\theta^{(t)}(2,a_1)$ (resp.~$\pi^{(t)}(2,a_1)$) 
are monotonically increasing with $t$ throughout the execution of the algorithm, 
which together with the initial condition $\pi^{(0)}(a_0 \mymid 1) = \pi^{(0)}(a_1\mymid 1)=\pi^{(0)}(a_0 \mymid 2) = \pi^{(0)}(a_1\mymid 2)$ 
as well as the identities $\theta^{(t)}(1,a_1)=-\theta^{(t)}(1,a_0)$ and $\theta^{(t)}(2,a_1)=-\theta^{(t)}(2,a_0)$ (due to~\eqref{eq:zero-sum}) gives 
\begin{align}
	\pi^{(t)}(a_0 \mymid 1) \leq \pi^{(t)}(a_1 \mymid 1) \qquad \text{and} \qquad
	\pi^{(t)}(a_0 \mymid 2) \leq \pi^{(t)}(a_1 \mymid 2) \qquad \text{for all }t\geq 0.
	\label{eq:policy-buffer-state-order-a1-a0}
\end{align}

\paragraph{Step 2: determine the range of $\pi^{(t)}(\cdot\mymid 1)$ and $\pi^{(t)}(\cdot\mymid 2)$.}

From the basic property~\eqref{eq:Q-pi-s12-a0-S1S2}, the value function of the buffer states in $\cS_1$ --- abbreviated by $V^{(t)}(1)$ as in the notation convention \eqref{eq:convenient-notation-V1-V2-Q1-Q2} --- satisfies 
\begin{align}
	 V^{(t)}(1) = -\gamma^2 \pi^{(t)}(a_0 \mymid 1) + \gamma^2 \pi^{(t)}(a_1 \mymid 1)
		= -\gamma^2 + 2 \gamma^2 \pi^{(t)}(a_1 \mymid 1), 
\end{align}
given that $\pi^{(t)}(a_0 \mymid 1) + \pi^{(t)}(a_1 \mymid 1) = 1$. 
Therefore, for any $t < t_{1}(\gamma^2-1/4)$ --- which means $V^{(t)}(1)< \gamma^2-1/4$ according to the definition \eqref{eqn:approx-optimal-time-buffer} --- one has the following upper bound:
\begin{align*}
	 V^{(t)}(1) 
	= -\gamma^2 + 2 \gamma^2 \pi^{(t)}(a_1 \mymid 1) < \gamma^2- 1 /4.
\end{align*}
This is equivalent to requiring that
\begin{align}
	\pi^{(t)}(a_1 \mymid 1) < 1 - (8\gamma^2)^{-1} \leq 7/8
	\label{eq:pi-t-a1-1-UB-gamma}
\end{align}
and, consequently, $\pi^{(t)}(a_0 \mymid 1) = 1 - \pi^{(t)}(a_1 \mymid 1) \geq 1/8$ for any $t < t_{1}(\gamma^2-1/4)$. 
Putting this and \eqref{eq:policy-buffer-state-order-a1-a0} together further implies --- for every $t < t_{1}(\gamma^2-1/4)$ --- that:  
\begin{align}
\label{eqn:initi-pi-sandwich}
	1/8 \leq \pi^{(t)}(a_0 \mymid 1 )\leq \pi^{(t)}(a_1 \mymid 1)\leq 7/8. 
\end{align}

\paragraph{Step 3:  determine the range of policy gradients.}

In addition to showing the non-negativity of $\frac{\partial V^{(t)}(\mu)}{\partial \theta(1, a_1)}$ and $\frac{\partial V^{(t)}(\mu)}{\partial \theta(2, a_1)}$ for all $t\geq 0$,
we are also in need of bounding their magnitudes. 
Towards this, invoke the property~\eqref{eqn:initi-pi-sandwich} to bound the derivative \eqref{eqn:gradient-simple-dvorak} by
\begin{align}
\frac{7\gamma^2}{32(1-\gamma)}d^{(t)}_{\mu}(1) \leq \frac{\partial V^{(t)}(\mu)}{\partial \theta(1, a_1)} 
\leq \frac{\gamma^2}{2(1-\gamma)}d^{(t)}_{\mu}(1) 
	\label{eq:V-grad-range-S1-buffer}
\end{align}
for any $t < t_{1}(\gamma^2-1/4)$, where we have used the elementary facts 
$$\min_{1/8\leq x\leq7/8}x(1-x)=7/64 \qquad \text{and} \qquad \max_{0\leq x\leq 1}x(1-x)=1/4.$$
Similarly, repeating the above argument with the gradient expression \eqref{eqn:gradient-simple-dvorak-V2} leads to
\begin{subequations}
\label{eq:V-grad-range-S2-buffer}
\begin{align}
\frac{\partial V^{(t)}(\mu)}{\partial \theta(2, a_1)} 
	&\leq \frac{\gamma^4}{2(1-\gamma)}d^{(t)}_{\mu}(2)   &&\text{for all }t\geq 0; \\
\frac{\partial V^{(t)}(\mu)}{\partial \theta(2, a_1)} 
	&\geq \frac{7\gamma^4}{32(1-\gamma)}d^{(t)}_{\mu}(2)  &&\text{for all }0\leq t < t_{2}(\gamma^4-1/4).
\end{align}
\end{subequations}
Further, note that Lemma~\ref{lem:facts-d-pi-s-LB} and Lemma~\ref{lem:facts-d-pi-s}  deliver upper and lower bounds on the quantities $d^{(t)}_{\mu}(1)$ and $d^{(t)}_{\mu}(2)$,
which allow us to deduce that
\begin{subequations}
\label{eqn:init-bounded-gradient}
\begin{align}
	\frac{\partial V^{(t)}(\mu)}{\partial\theta(1,a_{1})} & \geq \frac{7\gamma^{3}\cm}{32\cbone|\cS|}  && \text{for all } t<t_{1}\big(\gamma^{2}-1/4\big);\\
	\frac{\partial V^{(t)}(\mu)}{\partial\theta(2,a_{1})} & \geq \frac{7\gamma^{5}\cm}{32\cbtwo|\cS|} && \text{for all }  t<t_{2}(\gamma^4 - 1/4) ;\\
	\frac{\partial V^{(t)}(\mu)}{\partial\theta(1,a_{1})} & \leq \frac{\gamma^{2}(1+17\cm/\cbone)}{2|\cS|}  &&\text{for all }  t<\min\big\{ t_{1}(\tau_{1}),t_{2}(\tau_{2})\big\};\\
	\frac{\partial V^{(t)}(\mu)}{\partial\theta(2,a_{1})} & \leq \frac{\gamma^{4}(1+8\cm/\cbtwo)}{2|\cS|} && \text{for all }  t<t_{2}(\tau_{2}) .
\end{align}
\end{subequations}
%

\paragraph{Step 4: develop an upper bound on $t_{1}(\gamma^2-1/4)$.} 

The preceding bounds allow us to develop an upper bound on $t_{1}(\gamma^2-1/4)$. 
To do so, it is first observed from the fact $\theta^{(t)}(1,a_{0})=-\theta^{(t)}(1,a_{1})$ (due to~\eqref{eq:zero-sum}) that
\[
\pi^{(t)}(a_{1}\mymid1)=\frac{\exp\left(\theta^{(t)}(1,a_{1})\right)}{\exp\left(\theta^{(t)}(1,a_{0})\right)+\exp\left(\theta^{(t)}(1,a_{1})\right)}
= 1 - \frac{1}{1+\exp\left(2\theta^{(t)}(1,a_{1})\right)}.
\]
Recognizing that $V^{(t)}(1) < \gamma^2-1/4$ occurs if and only if $\pi^{(t)}(a_1 \mymid 1) < 1- (8\gamma^2)^{-1}$ (see \eqref{eq:pi-t-a1-1-UB-gamma}),
we can easily demonstrate that
\begin{align}
	\theta^{(t)}(1,a_{1})\leq\frac{1}{2}\log\left(8\gamma^{2}-1\right) \leq \frac{1}{2} \log 7 \qquad \text{for all }t < t_1\big( \gamma^2 - 1/4 \big).
	\label{eq:theta-t-1-a1-UB-log7}
\end{align}
If $t_{1}\big(\gamma^{2}-1/4\big) \geq \big\lceil \frac{32\log(7)\cbone|\cS|}{7\gamma^{3}\cm\eta} \big\rceil$, 
then taking $t= \big\lceil \frac{32\log(7)\cbone|\cS|}{7\gamma^{3}\cm\eta} \big\rceil $ together with \eqref{eqn:init-bounded-gradient} and \eqref{eq:PG-update} yields
\[
	\theta^{(t)}(1,a_{1})\geq\theta^{(0)}(1,a_{1})+\eta\frac{7\gamma^{3}\cm}{32\cbone|\cS|}t=\eta\frac{7\gamma^{3}\cm}{32\cbone|\cS|}t \geq \log 7,
\]
thus leading to contradiction with \eqref{eq:theta-t-1-a1-UB-log7}. As a result, one arrives at the following upper bound:
\begin{align}
	t_1(\taun_1)\leq t_{1}\big(\gamma^{2}-1/4\big) \leq  \frac{32\log(7)\cbone|\cS|}{7\gamma^{3}\cm\eta} \leq \frac{15\cbone|\cS|}{\cm\eta}, 
	\label{eq:t1-gamma2-025-UB}
\end{align}
with the proviso that $\gamma \geq 0.85$ (so that $\tau_1\leq \gamma^{2}-1/4$).

An upper bound on $t_2(\gamma^4 - 1/4)$ (and hence $t_2(\tau_2)$) can be obtained in a completely analogous manner
\[
	t_2(\taun_2)\leq t_{2}\big(\gamma^{4}-1/4\big) \leq  \frac{15\cbtwo|\cS|}{\cm\eta},
\]
provided that $\gamma\geq 0.95$ (so that $\tau_2\leq \gamma^{4}-1/4$). We omit the proof of this part for the sake of brevity.

\paragraph{Step 5: develop a lower bound on $t_{2}(\taun_2)$.}

Repeating the argument in \eqref{eq:pi-t-a1-1-UB-gamma} and \eqref{eq:theta-t-1-a1-UB-log7}, 
we see that $V^{(t)}(2) \geq  \taun_2$ if and only if $\pi^{(t)}(a_1\mymid 2) \geq \frac{1}{2}+\frac{\taun_2}{2\gamma^4}$, which is also equivalent to
\[
	\theta^{(t)}(2,a_{1}) \geq \frac{1}{2}\log\left(\frac{1}{\frac{1}{2}-\frac{\taun_2}{2\gamma^{4}}}-1\right)
	> \frac{1}{2} \log 3,
\]
as long as $2\taun_2 > \gamma^{4}$. 
Of necessity, this implies that $\theta^{(t)}(2,a_{1})>\frac{1}{2} \log 3$ when $t=t_2 (\taun_2)$. 
If $t_{2}(\taun_2) \leq \frac{|\cS|\log3}{2\eta\gamma^{4}\left(1+8\cm/\cbtwo\right)}$, 
then invoking \eqref{eqn:init-bounded-gradient} and \eqref{eq:PG-update} and taking $t= t_{2} (\taun_2)$ yield
\[
	\theta^{(t)}(2,a_{1}) \leq \theta^{(0)}(2,a_{1}) + \eta \frac{\gamma^{4}}{2|\cS|}\left(1+\frac{8\cm}{\cbtwo}\right)  t
		= \frac{\eta \gamma^{4}t}{2|\cS|}\left(1+\frac{8\cm}{\cbtwo}\right)   \leq \frac{1}{2}\log 3,
\]
thus resulting in contradiction. We can thus conclude that
\begin{align}
	t_{2}( \taun_2 ) >  \frac{|\cS|\log3}{\eta\gamma^{4}\left(1+8\cm/\cbtwo\right)}  > \frac{|\cS| \log 3}{\eta (1+8\cm/\cbtwo)}.
	\label{eq:t2-tau2-LB-17}
\end{align}
As an important byproduct, comparing \eqref{eq:t2-tau2-LB-17} with \eqref{eq:t1-gamma2-025-UB} immediately reveals that
\begin{align}
	t_{2}( \taun_2 ) \geq t_{1}\big(\gamma^{2}-1/4\big)  \geq t_{1}\big(\taun_1) ,
	\label{eq:t2-t1-order-proof}
\end{align}
with the proviso that $\frac{\log3}{1+8\cm/\cbtwo}\geq\frac{15\cbone}{\cm}$ and $\gamma \geq 0.87$ (so that $\gamma^{2}-1/4>\taun_1$).

\paragraph{Step 6: develop a lower bound on $t_{1}(\taun_1)$.}
   
Repeat the analysis in \eqref{eq:pi-t-a1-1-UB-gamma} and \eqref{eq:theta-t-1-a1-UB-log7} to show that: $V^{(t)}(1) \geq \taun_1$ if and only if 
\[
	\theta^{(t)}(1,a_{1}) \geq \frac{1}{2}\log\left(\frac{1}{\frac{1}{2}-\frac{\taun_1}{2\gamma^{2}}}-1\right)
	> \frac{1}{2} \log 3.
\]
Clearly, this lower bound should hold if $t=t_{1} (\taun_1)$. 
In addition, in view of \eqref{eq:t2-t1-order-proof}, one has $\min\{t_1(\taun_1), t_2(\taun_2)\}=t_1(\taun_1)$.
If $t_1(\taun_1) \leq \frac{|\cS|\log3}{\eta\gamma^{2}(1+17\cm/\cbone)}$, 
then setting $t= t_1(\taun_1) =\min\{t_1(\taun_1), t_2(\taun_2)\} $ and applying \eqref{eqn:init-bounded-gradient} and \eqref{eq:PG-update} lead to
\[
	\theta^{(t)}(1,a_{1}) \leq \theta^{(0)}(1,a_{1})+\eta\frac{\gamma^{2}(1+17\cm/\cbone)}{2|\cS|} t
		= \frac{\eta\gamma^{2} t (1+17\cm/\cbone)}{2|\cS|} \leq \frac{1}{2}\log 3,
\]
which is contradictory to the preceding lower bound. This in turn implies that 
\begin{align}
	t_{1}( \tau_1 ) \geq  \frac{|\cS|\log3}{\eta\gamma^{2}(1+17\cm/\cbone)} > \frac{|\cS| \log 3}{\eta (1+17\cm/\cbone)}. 
	\label{eq:t1-tau1-LB-17}
\end{align}

\subsection{Crossing times for the adjoint state  $\overline{1}$}

We now move on to the proof of \eqref{eqn:t1-t2-bar-scaling}. 
Note that we have developed a lower bound on $t_{2}(\taun_2)$ in \eqref{eq:t2-tau2-LB-17}. 
In order to justify the advertised result $t_{2}(\taun_2) > t_{\overline{1}}\big(\gamma^3-1/4\big)$, it thus suffices to demonstrate that
\begin{align}
	\label{eqn:init-t1-bar-upper}
	t_{\overline{1}}\big(\gamma^3-1/4\big) \leq \frac{|\cS| \log 3}{\eta (1+8\cm/\cbtwo)},	
\end{align}
a goal we aim to accomplish in this subsection.

To do so, we divide into two cases. 
In the scenario where $t_{1} (\taun_1 ) \ge t_{\overline{1}}\big(\gamma^3-1/4\big)$,  
the bound \eqref{eq:t1-gamma2-025-UB} derived previously immediately leads to the desired bound:  
\[
	t_{\overline{1}}\big(\gamma^3-1/4\big) \leq t_{1} (\taun_1 )  \leq  \frac{15\cbone|\cS|}{\cm\eta} \leq \frac{|\cS| \log 3}{\eta (1+8\cm/\cbtwo)},
\]
with the proviso that $\frac{15\cbone}{\cm}\leq\frac{\log3}{1+8\cm/\cbtwo}$. 
Consequently, the subsequent analysis concentrates on establishing \eqref{eqn:init-t1-bar-upper} for the case where 
$$ t_{1} (\taun_1 ) < t_{\overline{1}} \big(\gamma^3-1/4\big). $$
In what follows, we divide into three stages and investigate each one separately, after presenting some basic gradient calculations that shall be invoked frequently.

\paragraph{Gradient characterizations.}

To begin with, observe from \eqref{eq:PG-update-all} that
\begin{subequations}
\begin{align}
\frac{\partial V^{(t)}(\mu)}{\partial\theta(\overline{1},a_{1})} & =\frac{1}{1-\gamma}d_{\mu}^{(t)}(\overline{1})\pi^{(t)}\big(a_{1}|\overline{1}\big)\Big(Q^{(t)}(\overline{1},a_{1})-V^{(t)}(\overline{1})\Big) \notag\\
 & =\frac{1}{1-\gamma}d_{\mu}^{(t)}(\overline{1})\pi^{(t)}\big(a_{1}\mymid\overline{1}\big)\Big(Q^{(t)}(\overline{1},a_{1})-\sum_{a\in\{a_{0},a_{1}\}}\pi^{(t)}\big(a\mymid\overline{1}\big)Q^{(t)}(\overline{1},a)\Big) \notag\\
 & =\frac{1}{1-\gamma}d_{\mu}^{(t)}(\overline{1})\pi^{(t)}\big(a_{1}\mymid\overline{1}\big)\pi^{(t)}\big(a_{0}\mymid\overline{1}\big)\Big(Q^{(t)}(\overline{1},a_{1})-Q^{(t)}(\overline{1},a_{0})\Big),
 \label{eq:derivative-Vt-1bar-LB-123}
\end{align}
which makes use of  the fact $\pi^{(t)}\big(a_{0}\mymid\overline{1}\big)+ \pi^{(t)}\big(a_{1}\mymid\overline{1}\big) =1$.
Analogously, we have
\begin{align}
	\frac{\partial V^{(t)}(\mu)}{\partial\theta(\overline{1},a_{0})}
	= - \frac{1}{1-\gamma}d_{\mu}^{(t)}(\overline{s})\pi^{(t)}(a_{1}\mymid\overline{1})\pi^{(t)}(a_{0}\mymid\overline{1})\Big(Q^{(t)}(\overline{1},a_{1})-Q^{(t)}(\overline{1},a_{0})\Big).
	\label{eq:derivative-Vt-1bar-LB-123-a0}
\end{align}
\end{subequations}

\paragraph{Stage 1: any $t$ obeying $t <  t_{1} (\taun_1 )$.} 
 
We start by looking at each term in the gradient expression \eqref{eq:derivative-Vt-1bar-LB-123} separately. 
First, note that when $t <  t_{1} (\taun_1 )$, one has $V^{(t)}(1)<\taun_1$, 
which combined with \eqref{eq:Qpi-s-adj} in Lemma~\ref{lem:basic-properties-MDP-pi} indicates that
$Q^{(t)}(\overline{1}, a_1) = \gamma V^{(t)}(1)  < \gamma\tau_1 = Q^{(t)}(\overline{1}, a_0)$. 
In fact, from the definition \eqref{defn:tau-s} of $\tau_1$, the property \eqref{eq:Qpi-s-adj} and Lemma~\ref{lem:non-negativity-PG}, we have
\[
	1/2 \geq Q^{(t)}(\overline{1}, a_0) > Q^{(t)}(\overline{1}, a_1) = \gamma V^{(t)}(1) \geq 0.
\]
Additionally,  recall that $t_{1} (\taun_1 ) < t_{2} (\taun_2 )$ (see \eqref{eq:t2-t1-order-proof}). 
Lemma~\ref{lem:facts-d-pi-s} then tells us that $d^{(t)}_{\mu}(\overline{1}) \le 14\cm(1-\gamma)^2$ during this stage. 
Substituting these into \eqref{eq:derivative-Vt-1bar-LB-123} and using $\pi^{(t)}(a_0 \mymid \overline{1})\leq 1$, we arrive at
\begin{align}
	0 \geq \frac{\partial V^{(t)}(\mu)}{\partial\theta(\overline{1},a_{1})} \geq - 7\cm(1-\gamma)\pi^{(t)}\big(a_{1}\mymid\overline{1}\big),
	\label{eq:Vt-derivative-1bar-a1-UB-LB-1}
\end{align}
%
which together with the PG update rule \eqref{eq:PG-update-all} also indicates that $\theta^{(t)}(\overline{1},a_{1})$ (and hence $\pi^{(t)}\big(a_{1}\mymid\overline{1}\big)$) 
is monotonically non-increasing with $t$ in this stage. 
Invoke the auxiliary fact in Lemma~\ref{lem:init-pi-t-diff} to reach
\begin{align*}
\pi^{(t+1)}\big(a_{1}\mymid\overline{1}\big)-\pi^{(t)}\big(a_{1}\mymid\overline{1}\big) 
	\geq 2\eta\pi^{(t)}\big(a_{1}\mymid\overline{1}\big)\frac{\partial \soft{V}^{(t)}(\mu)}{\partial\theta(\overline{1},a_{1})}
  	\geq -14 \eta\cm(1-\gamma)\Big[\pi^{(t)}\big(a_{1}\mymid\overline{1}\big)\Big]^{2}.
\end{align*}
Taking the preceding recursive relation together with Lemma~\ref{lem:opt-lemma} and recalling the initialization $\pi^{(0)}\big(a_{1}\mymid\overline{1}\big)=1/2$, 
we can guarantee that
\begin{align}
	\pi^{(t)}\big(a_{1}\mymid\overline{1}\big)\geq\frac{1}{28\eta\cm(1-\gamma)t+2} 
	\qquad \text{for all } t \leq t_1(\taun_1)
	\label{eqn:init-first-stage-pi-first}
\end{align}
provided that $14\eta\cm(1-\gamma)\leq 1$. 
In conclusion, the above calculation precludes $\pi^{(t)}\big(a_{1}\mymid\overline{1}\big)$ from decaying to zero too quickly, 
an observation that is particularly useful for our analysis in Stage 3.

\paragraph{Stage 2: any $t$ obeying $t_{1}(\taun_1) \le t < t_{1}(\gamma^2-1/4)$. }


The only step lies in extending the lower bound \eqref{eqn:init-first-stage-pi-first} to this stage. 
From the definition \eqref{eqn:approx-optimal-time-buffer} of $t_{1}(\taun_1)$ as well as the monotonicity of $V^{(t)}(1)$  (see Lemma~\ref{lem:ascent-lemma-PG}), we know that
\[
	V^{(t)}(1) \geq V^{(t_1(\taun_1))}(1) \geq  \taun_1 \qquad \text{for all }t\geq t_1(\taun_1),
\]
provided that $\eta < (1-\gamma)^2/5$. This taken together with the property \eqref{eq:Qpi-s-adj} in Lemma~\ref{lem:basic-properties-MDP-pi} reveals that 
\[
Q^{(t)}(\overline{1}, a_1) - Q^{(t)}(\overline{1}, a_0) \ge 0
\qquad \text{for all }t\geq t_1(\taun_1),
\] 
and hence $\pi^{(t)}\big(a_{1}\mymid\overline{1}\big)$ is non-decreasing in $t$ during this stage. 
Therefore, we have
\begin{align}
	\pi^{(t)}\big(a_{1}\mymid\overline{1}\big) \geq \pi^{(t_1(\tau_1))}\big(a_{1}\mymid\overline{1}\big) \geq
	\frac{1}{28\eta\cm(1-\gamma)t_1(\tau_1)+2}, 
	\qquad    t \geq t_1(\tau_1), 
	\label{eqn:init-first-stage-pi}
\end{align}
where the first inequality follows from the non-decreasing property established above, and the second inequality follows from the lower bound \eqref{eqn:init-first-stage-pi-first}. %
In fact, we have established a lower bound on $\pi^{(t)}\big(a_{1}\mymid\overline{1}\big)$ that holds for the entire trajectory of the algorithm.

\paragraph{Stage 3: any $t$ obeying $t_{1}(\gamma^2-1/4) \le t \le t_{\overline{1}}(\gamma^3-1/4)$. }  
To facilitate analysis, we single out a time threshold $t^{\prime}$ as follows:
\begin{align}
	t^{\prime} \coloneqq \min \left\{ t \mymid \pi^{(t)} (a_0 \mymid \overline{1}) < 1/2 \right\}.
	\label{eq:defn-tprime-pi-t-a0}
\end{align}
We begin by developing an upper bound on $\pi^{(t)}\big(a_{0}\mymid\overline{1}\big) $ for any $t \geq \max\{t^{\prime}, t_{1}(\gamma^2-1/4)\}$. 
Towards this, with the help of \eqref{eq:Qpi-s-adj} in Lemma~\ref{lem:basic-properties-MDP-pi} we make the observation that: for any $t \geq t_{1}(\gamma^2-1/4)$, one has
\begin{align}
	Q^{(t)}(\overline{1}, a_1) - Q^{(t)}(\overline{1}, a_0)  = \gamma V^{(t)}(1) - \gamma \tau_1 \geq 
	\gamma \big(\gamma^2-1/4\big) - \gamma \tau_1 \ge 0.1
	\label{eq:Qt-a1-a0-gap-tau1}
\end{align}
as long as $\gamma\geq 0.92$, which combined with \eqref{eq:derivative-Vt-1bar-LB-123-a0} indicates that
\begin{align}
	\frac{\partial V^{(t)}(\mu)}{\partial\theta(\overline{1},a_{0})}
	< 0 \qquad \text{for all } t \geq t_{1}(\gamma^2-1/4).
	\label{eq:V-derivative-a0-1bar-negative}
\end{align}
Recognizing that $d^{\pi}_{\mu}(\overline{1}) \ge \cm\gamma(1-\gamma)^2$ (see Lemma \ref{lem:facts-d-pi-s-LB}), 
we can continue the derivation \eqref{eq:derivative-Vt-1bar-LB-123-a0} to derive
\begin{align*}
\frac{\partial V^{(t)}(\mu)}{\partial\theta(\overline{1},a_{0})} & \leq-\frac{1}{1-\gamma}\gamma\cm(1-\gamma)^{2}\cdot\frac{1}{2}\cdot\pi^{(t)}(a_{0}\mymid\overline{1})\cdot0.1=-0.05\cm\gamma(1-\gamma)\pi^{(t)}(a_{0}\mymid\overline{1})
\end{align*}
%
for any $t \geq \max\{t^{\prime}, t_{1}(\gamma^2-1/4)\}$, which implies
%
\[
	\pi^{(t)} (a_1 \mymid \overline{1}) \geq \pi^{(t')} (a_1 \mymid \overline{1}) = 1 - \pi^{(t')} (a_0 \mymid \overline{1}) \geq  1/2 \qquad \text{for any }t\geq t^{\prime}. 
\]
Invoke Lemma~\ref{lem:init-pi-t-diff} to arrive at 
\[
	\pi^{(t+1)}\big(a_{0}\mymid\overline{1}\big)-\pi^{(t)}\big(a_{0}\mymid\overline{1}\big)
	\leq \frac{\eta}{2}\pi^{(t)}\big(a_{0}\mymid\overline{1}\big)  \frac{\partial \soft{V}^{(t)}(\mu)}{\partial\theta(\overline{1},a_{0})}
	\leq - \frac{\eta}{40}\cm\gamma(1-\gamma)\Big[\pi^{(t)}(a_{0}\mymid\overline{1})\Big]^{2},
\]
provided that $2\eta\frac{\partial \soft{V}^{(t)}(\mu)}{\partial\theta(\overline{1},a_{0})}\geq -1$, which is guaranteed by $\eta < (1-\gamma)/2$.
Recalling that $\pi^{(t)}\big(a_{0}\mymid\overline{1}\big) \leq 1/2$ for this entire stage, 
one can apply Lemma~\ref{lem:opt-lemma} to obtain
\begin{equation}
	\pi^{(t)}\big(a_{0}\mymid\overline{1}\big) \leq \frac{1}{\frac{\eta}{40}\cm\gamma(1-\gamma)\Big(t-\max\big\{t^{\prime}, t_{1}(\gamma^2-1/4)\big\} \Big)  +  2}
	\label{eq:pit-a0-1bar-decay-rate-UB}
\end{equation}
for any $t \geq \max\{t^{\prime}, t_{1}(\gamma^2-1/4)\}$.

With the above upper bound \eqref{eq:pit-a0-1bar-decay-rate-UB} in place, we are capable of showing that the target quantity $t_{\overline{1}}\big( \gamma^{3}-1/4 \big)$ is not much larger than $\max\{t^{\prime}, t_{1}(\gamma^2-1/4)\}$. 
To show this, we first note that the value function of the adjoint state $\overline{1}$ obeys (see Part (iii) in Lemma~\ref{lem:basic-properties-MDP-pi}) 
\begin{align*}
V^{(t)}(\overline{1}) & =\pi^{(t)}(a_{0}\mymid\overline{1})Q^{(t)}(\overline{1},a_{0})+\pi^{(t)}(a_{1}\mymid\overline{1})Q^{(t)}(\overline{1},a_{1})
=\gamma\tau_{1}\pi^{(t)}(a_{0}\mymid\overline{1})+\gamma\pi^{(t)}(a_{1}\mymid\overline{1})V^{(t)}(1)\\
 & =\gamma\tau_{1}\pi^{(t)}(a_{0}\mymid\overline{1})+\gamma V^{(t)}(1)\Big\{1-\pi(a_{0}\mymid\overline{1})\Big\}\geq\gamma\tau_{1}\pi^{(t)}(a_{0}\mymid\overline{1})+\gamma\left(\gamma^{2}-1/4\right)\Big\{1-\pi^{(t)}(a_{0}\mymid\overline{1})\Big\}\\
 & =\gamma\left\{ \tau_{1}-\gamma^{2}+1/4\right\} \pi^{(t)}(a_{0}\mymid\overline{1})+\gamma^{3}-\gamma/4,
\end{align*}
where the inequality holds since $V^{(t)}(1)\geq\gamma^{2}-1/4$ in this stage (given that $t\geq t_1(\gamma^2-1/4)$). 
Recognizing that $0.5\gamma^{2/3}-\gamma^{2}+1/4<0$ for any $\gamma \geq 0.85$, we can rearrange terms to demonstrate that $V^{(t)}(\overline{1})\geq \gamma^3 - 1/4$ holds once
\[
	\pi^{(t)}(a_{0}\mymid\overline{1})\leq\frac{1-\gamma}{4\gamma\left(\gamma^{2}-1/4-0.5\gamma^{2/3}\right)}.
\]
In fact, for any $\gamma \geq 0.85$, the above inequality is guaranteed to hold as long as $\pi^{(t)}(a_{0}\mymid\overline{1})\leq 1-\gamma$ since $4\gamma\left(\gamma^{2}-1/4-0.5\gamma^{2/3}\right)<1$. 
In view of \eqref{eq:pit-a0-1bar-decay-rate-UB}, we can achieve $\pi^{(t)}(a_0 \mymid \overline{1}) \le 1-\gamma$ as soon as
$t - \max\left\{t^{\prime}, t_{1}(\gamma^2-1/4)\right\}$ surpasses $\frac{40}{\cm\gamma\eta(1-\gamma)^2}$. 
As a consequence, we reach
\begin{align}
	t_{\overline{1}}\big( \gamma^{3}-1/4 \big) \leq  \max\left\{t^{\prime}, t_{1}(\gamma^2-1/4)\right\} + \frac{40}{\cm\gamma\eta(1-\gamma)^2} .
	\label{eq:t-bar1-gamma3-relation-tprime-1}
\end{align}

Armed with the relation \eqref{eq:t-bar1-gamma3-relation-tprime-1}, 
the goal of upper bounding $t_{\overline{1}}\big( \gamma^{3}-1/4 \big)$ can be accomplished by 
controlling $t^{\prime}$. To this end,  we claim for the moment that 
\begin{align}
\label{eqn:init-upper-bound-t-prime}
	t^{\prime} \le \frac{1121t_{1}(\gamma^2-1/4)}{\gamma}.
\end{align}
If this claim holds, then combining it with \eqref{eq:t-bar1-gamma3-relation-tprime-1} and \eqref{eq:t1-gamma2-025-UB} would result in the advertised bound \eqref{eqn:init-t1-bar-upper}:
\begin{align*}
	t_{\overline{1}}\big( \gamma^{3}-1/4 \big) \le \frac{9972\cbone|\cS|}{\gamma^4\cm\eta} + \frac{40}{\cm\gamma\eta(1-\gamma)^2}
	\leq 
	\frac{|\cS|}{4\gamma^4\eta}
	\leq \frac{|\cS|\log3}{\eta(1+8\cm/\cbtwo)},
\end{align*}
where the penultimate inequality relies on the assumptions $\frac{\cbone}{\cm} \le \frac{1}{79776}$ and $|\cS| \geq \frac{320\gamma^3}{\cm(1-\gamma)^{2}}$, 
and the last one holds as long as $\frac{1}{4\gamma^{4}}\leq\frac{\log3}{1+8\cm/\cbtwo}$.
To finish up, it suffices to establish the claim \eqref{eqn:init-upper-bound-t-prime}.

\paragraph{Proof of the claim~\eqref{eqn:init-upper-bound-t-prime}.}

It is sufficient to consider the case where $t^{\prime} > t_{1}(\gamma^2-1/4)$;
otherwise the inequality \eqref{eqn:init-upper-bound-t-prime} is trivially satisfied. 
Since Lemma \ref{lem:facts-d-pi-s-LB} tells us that $d^{\pi}_{\mu}(\overline{1}) \ge \cm\gamma(1-\gamma)^2$, we can see from \eqref{eq:derivative-Vt-1bar-LB-123} that,  for any $t$ with $t_{1}(\gamma^2-1/4) \leq t < t^{\prime}$, 
\begin{align*}
\frac{\partial V^{(t)}(\mu)}{\partial\theta(\overline{1},a_{1})} & =\frac{1}{1-\gamma}d_{\mu}^{(t)}(\overline{1})\pi^{(t)}(a_{1}\mymid\overline{1})\pi^{(t)}(a_{0}\mymid\overline{1})\Big\{ Q^{(t)}(\overline{1},a_{1})-Q^{(t)}(\overline{1},a_{0})\Big\}\\
 & \ge0.05\cm\gamma(1-\gamma)\pi^{(t)}(a_{1}\mymid\overline{1}) >0,
\end{align*}
where the last line follows by combining \eqref{eq:Qt-a1-a0-gap-tau1} and the fact that $\pi^{(t)}(a_{0}\mymid\overline{1})\geq 1/2$ for any $t<t^{\prime}$ (see the definition \eqref{eq:defn-tprime-pi-t-a0} of $t^{\prime}$). 
According to Lemma~\ref{lem:init-pi-t-diff}, we can demonstrate that
\begin{align*}
\pi^{(t+1)}\big(a_{1}\mymid\overline{1}\big)-\pi^{(t)}\big(a_{1}\mymid\overline{1}\big) & \geq\eta\pi^{(t)}\big(a_{1}\mymid\overline{1}\big)\frac{\partial \soft{V}^{(t)}(\mu)}{\partial\theta(\overline{1},a_{1})}\geq 0.05\eta\cm\gamma(1-\gamma)\Big[\pi^{(t)}\big(a_{1}\mymid\overline{1}\big)\Big]^{2}
\end{align*}
for any $t$ obeying $t_{1}(\gamma^2-1/4) \leq t < t^{\prime}$. Invoking Lemma~\ref{lem:opt-lemma}, we then have
\begin{align*}
	t^{\prime} \leq \frac{1+0.025\eta\cm\gamma(1-\gamma)}{0.05\eta\cm\gamma(1-\gamma)\pi^{( t_{1}(\gamma^2-1/4))}\big(a_{1}\mymid\overline{1}\big)}+ t_{1}(\gamma^2-1/4) &<\frac{40}{ \eta\cm\gamma(1-\gamma)\pi^{( t_{1}(\gamma^2-1/4))}\big(a_{1}\mymid\overline{1}\big)}+ t_{1}(\gamma^2-1/4) \\
	&\leq \frac{40\big(28t_{1}(\tau_1)+2\big)}{\gamma} + t_{1}(\gamma^2-1/4) \\
	&\leq  \frac{1121 t_1(\gamma^2-1/4)}{\gamma}
\end{align*}
as claimed, where the second line follows from \eqref{eqn:init-first-stage-pi}. 
%

\subsection{Auxiliary facts}
\label{sec:covering-time-auxiliary}

In this subsection, we collect some elementary facts that have been used multiple times in the proof of Lemma~\ref{lem:init-step}. 
Specifically, the lemma below makes clear an explicit link between the gradient  $\nabla_{\theta} V^{(t)}(\mu)$ and the difference between two consecutive policy iterates. 
\begin{lemma}
\label{lem:init-pi-t-diff}
Consider any $s$ whose associated action space is $\{a_0,a_1\}$. 
\begin{itemize}
	\item  If $\frac{\partial V^{(t)}(\mu)}{\partial\theta(s,a_{1})}\leq 0$, then one has 
\begin{align}
	\pi^{(t+1)}\big(a_{1}\mymid s\big)-\pi^{(t)}\big(a_{1}\mymid s\big) 
	\geq 2 \eta \pi^{(t)}\big(a_{1}\mymid s\big) \frac{\partial V^{(t)}(\mu)}{\partial\theta(s,a_{1})}. \label{eqn:init-pi-t-diff}
\end{align}
	\item If $\pi^{(t+1)}\big(a_{0}\mymid s\big)\geq 1/2$ and $-1\leq 2\eta \frac{\partial V^{(t)}(\mu)}{\partial\theta(s,a_{1})}\leq 0$, then we have
\begin{align}	
	\pi^{(t+1)}\big(a_{1}\mymid s\big)-\pi^{(t)}\big(a_{1}\mymid s\big) 	
	\leq\frac{\eta}{2}\pi^{(t)}\big(a_{1}\mymid s\big)\frac{\partial V^{(t)}(\mu)}{\partial\theta(s,a_{1})}. \label{eqn:init-pi-t-diff-UB}
\end{align}
	\item If $\frac{\partial V^{(t)}(\mu)}{\partial\theta(s,a_{1})}\geq 0$ and if $\pi^{(t+1)}\big(a_{0}\mymid s\big)\geq 1/2$, then one has 
\begin{align}
	\pi^{(t+1)}\big(a_{1}\mymid s\big)-\pi^{(t)}\big(a_{1}\mymid s\big) 
	\geq  \eta \pi^{(t)}\big(a_{1}\mymid s\big) \frac{\partial V^{(t)}(\mu)}{\partial\theta(s,a_{1})}. \label{eqn:init-pi-t-diff-pos}
\end{align}
\end{itemize}
\end{lemma}


\begin{proof}[Proof of Lemma~\ref{lem:init-pi-t-diff}]
We make note of the following elementary identity
\begin{align*}
\frac{e^{\theta_{1}}}{e^{\theta_{1}}+e^{-\theta_{1}}}-\frac{e^{\theta_{2}}}{e^{\theta_{2}}+e^{-\theta_{2}}} & =\frac{e^{\theta_{1}-\theta_{2}}-e^{-\theta_{1}+\theta_{2}}}{\big(e^{\theta_{1}}+e^{-\theta_{1}}\big)\big(e^{\theta_{2}}+e^{-\theta_{2}}\big)}=\frac{e^{-\theta_{1}}}{e^{\theta_{1}}+e^{-\theta_{1}}} \frac{e^{\theta_{2}}}{e^{\theta_{2}}+e^{-\theta_{2}}}\left(e^{2(\theta_{1}-\theta_{2})}-1\right)\\
 & =\left(1-\frac{e^{\theta_{1}}}{e^{\theta_{1}}+e^{-\theta_{1}}}\right) \frac{e^{\theta_{2}}}{e^{\theta_{2}}+e^{-\theta_{2}}}\left(e^{2(\theta_{1}-\theta_{2})}-1\right),
\end{align*}
which allows us to write 
\begin{align}
\pi^{(t+1)}\big(a_{1}\mymid s\big)-\pi^{(t)}\big(a_{1}\mymid s\big) & =\pi^{(t+1)}\big(a_{0}\mymid s\big)\pi^{(t)}\big(a_{1}\mymid s\big)\left\{ \exp\Big[2\theta_{t+1}(s,a_{1})-2\theta_{t}(s,a_{1})\Big]-1\right\} \nonumber \\
 & =\pi^{(t+1)}\big(a_{0}\mymid s\big)\pi^{(t)}\big(a_{1}\mymid s\big)\left\{ \exp\Big[2\eta\frac{\partial V^{(t)}(\mu)}{\partial\theta(s,a_{1})}\Big]-1\right\}  .
	\label{eq:pit-pi-tminus1-a1-diff}
\end{align}

\begin{itemize}
	\item If $\frac{\partial V^{(t)}(\mu)}{\partial\theta(s,a_{1})} \leq 0$, then one can deduce that
\begin{align*}
\eqref{eq:pit-pi-tminus1-a1-diff} 
 & \geq 2\eta \pi^{(t+1)}\big(a_{0}\mymid s\big)\pi^{(t)}\big(a_{1}\mymid s\big)  \frac{\partial V^{(t)}(\mu)}{\partial\theta(s,a_{1})} 
  \geq 2\eta \pi^{(t)}\big(a_{1}\mymid s\big) \frac{\partial V^{(t)}(\mu)}{\partial\theta(s,a_{1})},
\end{align*}
where the first inequality relies on the elementary fact $e^{x}-1\geq x$ for all $x\in \mathbb{R}$, and the second one holds since $\pi^{(t+1)}\big(a_{0}\mymid s\big) \leq 1$ and $\frac{\partial V^{(t)}(\mu)}{\partial\theta(s,a_{1})} \leq 0$.  

	\item If $-1\leq 2\eta \frac{\partial V^{(t)}(\mu)}{\partial\theta(s,a_{1})}\leq 0$  and $\pi^{(t+1)}\big(a_{0}\mymid s\big) \geq 1/2$, then one has
\begin{align*}	
	\eqref{eq:pit-pi-tminus1-a1-diff} \leq\eta\pi^{(t+1)}\big(a_{0}\mymid s\big)\pi^{(t)}\big(a_{1}\mymid s\big)\frac{\partial V^{(t)}(\mu)}{\partial\theta(s,a_{1})}
	\leq\frac{\eta}{2}\pi^{(t)}\big(a_{1}\mymid s\big)\frac{\partial V^{(t)}(\mu)}{\partial\theta(s,a_{1})},
\end{align*}
where the first inequality comes from the elementary inequality $e^{x}-1\leq 0.5 x$ for any $-1\leq x \leq 0$,
and the last inequality is valid since $\pi^{(t+1)}\big(a_{0}\mymid s\big) \geq 1/2$ and $\frac{\partial V^{(t)}(\mu)}{\partial\theta(s,a_{1})}\leq 0$.

	\item If $\frac{\partial V^{(t)}(\mu)}{\partial\theta(s,a_{1})} \geq 0$ and if $\pi^{(t+1)}\big(a_{0}\mymid s\big) \geq 1/2$, then it follows that
\begin{align*}
\eqref{eq:pit-pi-tminus1-a1-diff} & \geq2\eta\pi^{(t+1)}\big(a_{0}\mymid s\big)\pi^{(t)}\big(a_{1}\mymid s\big)\frac{\partial V^{(t)}(\mu)}{\partial\theta(s,a_{1})}\geq\eta\pi^{(t)}\big(a_{1}\mymid s\big)\frac{\partial V^{(t)}(\mu)}{\partial\theta(s,a_{1})}, 
\end{align*}
as claimed in \eqref{eqn:init-pi-t-diff-pos}.
\end{itemize}
\end{proof}

\section{Analysis for the initial stage (Lemma~\ref{lem:induc-theta-t-s-2})}
\label{sec:analysis-parameter-crossing}

This section establishes Lemma~\ref{lem:induc-theta-t-s-2}, which investigates the dynamics of $\theta^{(t)}(s, a)$ prior to the threshold $t_{s-2}(\taun_{s-2})$.
Before proceeding, let us introduce a rescaled version of $\pi^{(t)}(s, a)$ that is sometimes convenient to work with: 
\begin{align}
	\label{eqn:defn-widehat-pi-soft-soft}
	\widehat{\pi}^{(t)}(s, a) \defn 
	\exp \Big(\theta^{(t)}(s,a) - \max_{a' \in \mathcal{A}_s}\theta^{(t)}(s, a')\Big)
\end{align}
for any state-action pair $(s,a)$. 
This is orderwise equivalent to $\pi^{(t)}(s, a)$ since
\begin{subequations}
	\label{eq:relation-pihat-pi}
\begin{align}
\widehat{\pi}^{(t)}(s,a) & =\frac{\exp\big(\theta^{(t)}(s,a)\big)}{\max_{a'\in\mathcal{A}_{s}}\exp\left(\theta^{(t)}(s,a')\right)}
	\geq\frac{\exp\big(\theta^{(t)}(s,a)\big)}{\sum_{a' \in\mathcal{A}_{s}}\exp\left(\theta^{(t)}(s,a')\right)}=\pi^{(t)}(s,a);\\
\widehat{\pi}^{(t)}(s,a) & =\frac{\exp\big(\theta^{(t)}(s,a)\big)}{\max_{a'\in\mathcal{A}_{s}}\exp\left(\theta^{(t)}(s,a')\right)}
	\leq\frac{\exp\big(\theta^{(t)}(s,a)\big)}{\frac{1}{3}\sum_{a'\in\mathcal{A}_{s}}\exp\left(\theta^{(t)}(s,a')\right)}=3\pi^{(t)}(s,a).
\end{align}
\end{subequations}

\subsection{Two key properties}

Our proof is based on the following claim: in order to establish the advertised results of Lemma~\ref{lem:induc-theta-t-s-2},
it suffices to justify the following two properties
\begin{align}
&\frac{1}{1+56\cm \eta(1-\gamma)t}\le \widehat{\pi}^{(t)}(a_1 \mymid s) \leq \frac{1}{1+\frac{\cm\gamma}{35} \eta(1-\gamma)^2t} \label{induc-intermediate-phat} \\
&\qquad\text{and}\qquad Q^{(t)}(s, a_2) - V^{(t)}(s) \geq  0  \label{induc-intermediate-qa2-large}
\end{align}
%
hold for any $t\leq t_{s-2}(\taun_{s-2})$. 
In light of this claim, our subsequent analysis consists of validating these two inequalities separately,
which forms the main content of Section~\ref{sec:proof-induc-intermediate-phat}.

We now move on to justify the above claim, namely, Lemma~\ref{lem:induc-theta-t-s-2} is valid as long as the two key properties \eqref{induc-intermediate-phat} and \eqref{induc-intermediate-qa2-large} hold true. 
%
First, recall that Lemma~\ref{lem:pi-t-UB-monotone} together with \eqref{eq:monotonicity-t-s-tau-s} and Lemma~\ref{lem:init-step} tells us that 
\begin{align} \label{eq:Qt-ordering-a2-a0}
	\theta^{(t)}(s, a_0) \geq \theta^{(t)}(s, a_2) \geq \theta^{(t)}(s, a_1) 
	\qquad \text{for all }t < t_{s-2}(\taun_{s-2}) \leq t_{s-1}(\taun_{s-1}) 
\end{align}
for any $3\leq s\leq H$. 
Next, note that the gradient takes the following form (cf.~\eqref{eq:PG-update-all})
\begin{align}
	\frac{\partial V^{(t)}(\mu)}{\partial \theta(s, a)} 
	= \frac{1}{1-\gamma} d^{(t)}_{\mu}(s) \pi^{(t)}(a | s) \big( Q^{(t)}(s, a) - V^{(t)}(s) \big),
	\qquad a\in \{a_0,a_1,a_2\}
	\label{eq:grad-form-following}
\end{align}
%
which together with the assumption $\soft{Q}^{(t)}(s, a_2) - \soft{V}^{(t)}(s) \geq  0$ (cf.~\eqref{induc-intermediate-qa2-large}) implies that
\begin{align*}
	\frac{\partial \soft{V}^{(t)}(\mu)}{\partial \theta(s, a_2)} \geq 0
	\qquad \text{for all }t < t_{s-2}(\taun_{s-2}). 
\end{align*}
Consequently, $\theta^{(t)}(s, a_2)$ keeps increasing before $t$ exceeds $t_{s-2}(\taun_{s-2})$. 
This combined with the relation \eqref{eq:Qt-ordering-a2-a0}, the initialization $\theta^{(0)}(s, a_0)=\theta^{(0)}(s, a_2)=0$ and the constraint $\sum_a \theta^{(t)}(s, a)=0$ (see Part (vii) of Lemma~\ref{lem:basic-properties-MDP-pi}) reveals that
\begin{align}
	\theta^{(t)}(s, a_0) \geq \theta^{(t)}(s, a_2) \geq 0 \geq \theta^{(t)}(s, a_1) 
	\qquad \text{for all }t < t_{s-2}(\tau_{s-2}), 
	\label{eqn:theta-order-brahms-n1}
\end{align}
thereby confirming the desired property \eqref{eqn:theta-a2-positive}.

Further, given the non-negativity of $\theta^{(t)}(s, a_2)$ stated in \eqref{eqn:theta-order-brahms-n1}, one can readily derive
\begin{align*}
\widehat{\pi}^{(t)}(a_{1}\mymid s) & =\exp\Big(\theta^{(t)}(s,a_{1})-\max_{a'}\theta^{(t)}(s,a')\Big)=\exp\Big(\theta^{(t)}(s,a_{1})-\theta^{(t)}(s,a_{0})\Big)\\
 & =\exp\Big(2\theta^{(t)}(s,a_{1})+\theta^{(t)}(s,a_{2})\Big)\geq\exp\Big(2\theta^{(t)}(s,a_{1})\Big),
\end{align*}
where the last line also makes use of the identity $\theta^{(t)}(s,a_{0})=-\theta^{(t)}(s,a_{1})-\theta^{(t)}(s,a_{2})$ (see Part (vii) of Lemma~\ref{lem:basic-properties-MDP-pi}). 
With this observation in mind, the assumed property~\eqref{induc-intermediate-phat} directly leads to the advertised result~\eqref{eqn:theta-t-relation-a1}.

\subsection{Proof of the properties~\eqref{induc-intermediate-phat} and \eqref{induc-intermediate-qa2-large}}
\label{sec:proof-induc-intermediate-phat}

This subsection presents the proofs of the two key properties, which are somewhat intertwined and require a bit of induction. 
Before proceeding, we make note of the initialization $\widehat{\pi}^{(0)}(a_{1}\mymid s)=1$, which clearly satisfies the property \eqref{induc-intermediate-phat} for this base case. Our proof consists of two steps to be detailed below. As can be easily seen, combining these two steps in an inductive manner immediately establishes both properties \eqref{induc-intermediate-phat} and \eqref{induc-intermediate-qa2-large} for any $t\leq t_{s-2}(\taun_{s-2})$.

\subsubsection*{Step 1: justifying \eqref{induc-intermediate-qa2-large} for the $t$-th iteration if \eqref{induc-intermediate-phat}  holds for the $t$-th iteration}

We first turn to the proof of the inequality~\eqref{induc-intermediate-qa2-large}, assuming that \eqref{induc-intermediate-phat} holds for the $t$-th iteration. 
%
%
%
%
%
According to \eqref{eq:relation-pihat-pi} and \eqref{induc-intermediate-phat}, we have
\begin{align}
	\pi^{(t)}(a_1 \mymid s) \ge \frac{1}{3}\widehat{\pi}^{(t)}(a_1 \mymid s) \ge \frac{1}{3+168\cm \eta(1-\gamma)t}.
\end{align}
%
By virtue of the auxiliary fact \eqref{eqn:induc-differenct-a0-a1} in  Lemma~\ref{lem:induc-auxi-q-diff} (see Section~\ref{sec:auxiliary-facts-induction}), one has 
\begin{align}
Q^{(t)}(s, a_0) - Q^{(t)}(s, a_2) \le 
	\frac{\gamma p}{\frac{\cm\gamma}{2}\eta(1-\gamma) t + \frac{1}{\gamma \tau_{s-2}}}.
\end{align}
Given that $p \defn \cp (1-\gamma)$ for some small constant $0 < \cp < \frac{1}{2016}$, the above two inequalities allow one to ensure that
\begin{align}
	Q^{(t)}(s, a_0) - Q^{(t)}(s, a_2) < (\gamma^{3/2} - \gamma^2)\tau_{s-1} \pi^{(t)}(a_1 \mymid s).
	\label{eq:Qt-a0-a2-gap-UB-123}
\end{align}
With the above relation in mind, we are ready to control $Q^{(t)}(s, a_2) - V^{(t)}(s)$ as follows 
\begin{align*}
Q^{(t)}(s,a_{2})-V^{(t)}(s) & =\Big(\sum_{a}\pi^{(t)}(a\mymid s)\Big) Q^{(t)}(s,a_{2}) -\sum_{a}\pi^{(t)}(a\mymid s)Q^{(t)}(s,a)\\
 & =\pi^{(t)}(a_{1}\mymid s)\Big(Q^{(t)}(s,a_{2})-Q^{(t)}(s,a_{1})\Big)-\pi^{(t)}(a_{0}\mymid s)\Big(Q^{(t)}(s,a_{0})-Q^{(t)}(s,a_{2})\Big)\\
 & \ge\pi^{(t)}(a_{1}\mymid s)(\gamma^{3/2}-\gamma^{2})\tau_{s-1}-\Big(Q^{(t)}(s,a_{0})-Q^{(t)}(s,a_{2})\Big)\\
 & >0.
\end{align*}
Here,  the second lines arise from the auxiliary facts in Lemma~\ref{lem:induc-auxi-q-diff}, while the last inequality is a consequence of \eqref{eq:Qt-a0-a2-gap-UB-123}. 
Then we complete the proof of the inequality~\eqref{induc-intermediate-qa2-large}. 

%





\subsubsection*{Step 2: justifying \eqref{induc-intermediate-phat} for the $(t+1)$-th iteration if \eqref{induc-intermediate-qa2-large} holds up to the $t$-th iteration}

Suppose that the inequality \eqref{induc-intermediate-qa2-large} holds up to the $t$-th iteration. 
To validate \eqref{induc-intermediate-phat} for the $(t+1)$-th iteration,
we claim for the moment that
\begin{align}
\label{eqn:upper-gradient-induc-violin}
	-14\cm(1-\gamma)\widehat{\pi}^{(t)}(a_1 \mymid s) \leq 
	\frac{\partial V^{(t)}(\mu)}{\partial \theta(s, a_1)} \leq 
	-\frac{\cm\gamma}{24}(1-\gamma)^2\widehat{\pi}^{(t)}(a_1 \mymid s) < 0 
\end{align}
as long as $t\leq t_s(\taun_s)$. 
Let us take this claim as given, and return to prove it shortly.

Recall from \eqref{eq:Qt-ordering-a2-a0} that
\[
	\theta^{(t)}(s, a_0) \geq \theta^{(t)}(s, a_2) \geq \theta^{(t)}(s, a_1)
\]
and hence 
$\theta_{\mathsf{max}}^{(t)}(s) = \theta^{(t)}(s,a_{0})$ is increasing with $t$ during this stage, and as a result,
\[
\theta^{(t+1)}(s,a_{1})-\theta^{(t)}(s,a_{1})+\theta_{\mathsf{max}}^{(t)}(s)-\theta_{\mathsf{max}}^{(t+1)}(s) \le \theta^{(t+1)}(s,a_{1})-\theta^{(t)}(s,a_{1}) \leq 0.
\]
The gradient expression \eqref{eq:grad-form-following} combined with the satisfaction of \eqref{induc-intermediate-qa2-large} up to the $t$-th iteration implies that
$\theta^{(t)}(s,a_{2})$ is increasing up to the $t$-th iteration. 
Given that $\sum_a \theta^{(t)}(s,a) = 0$ (see Part (vii) of Lemma~\ref{lem:basic-properties-MDP-pi}), we can derive 
\begin{align*}
\theta_{\mathsf{max}}^{(t)}(s)-\theta_{\mathsf{max}}^{(t+1)}(s) & =\theta^{(t)}(s,a_{0})-\theta^{(t+1)}(s,a_{0})=\theta^{(t+1)}(s,a_{1})-\theta^{(t)}(s,a_{1})+\theta^{(t+1)}(s,a_{2})-\theta^{(t)}(s,a_{2})\\
 & \geq\theta^{(t+1)}(s,a_{1})-\theta^{(t)}(s,a_{1}),
\end{align*}
thus indicating that 
\[
\theta^{(t+1)}(s,a_{1})-\theta^{(t)}(s,a_{1})+\theta_{\mathsf{max}}^{(t)}(s)-\theta_{\mathsf{max}}^{(t+1)}(s) 
\ge 2\Big(\theta^{(t+1)}(s,a_{1})-\theta^{(t)}(s,a_{1})\Big).
\]
These combined with Lemma~\ref{lem:widehat-pi-du-pre} in Section~\ref{sec:auxiliary-facts-induction} guarantee that
\begin{align*}
\widehat{\pi}^{(t+1)}(a_{1}\mymid s)-\widehat{\pi}^{(t)}(a_{1}\mymid s) & \geq2\widehat{\pi}^{(t)}(a_{1}\mymid s)\left(\theta^{(t+1)}(s,a_{1})-\theta^{(t)}(s,a_{1})\right), \\
\widehat{\pi}^{(t+1)}(a_{1}\mymid s)-\widehat{\pi}^{(t)}(a_{1}\mymid s) & \leq0.7\widehat{\pi}^{(t)}(a_{1}\mymid s)\left(\theta^{(t+1)}(s,a_{1})-\theta^{(t)}(s,a_{1})\right),
\end{align*}
and as a consequence,
%
\begin{align*}
2\widehat{\pi}^{(t)}(a_{1}\mymid s)\cdot\eta\frac{\partial V^{(t)}(\mu)}{\partial\theta(s,a_{1})} & \le\widehat{\pi}^{(t+1)}(a_{1}\mymid s)-\widehat{\pi}^{(t)}(a_{1}\mymid s)\le 0.7\widehat{\pi}^{(t)}(a_{1}\mymid s)\cdot\eta\frac{\partial V^{(t)}(\mu)}{\partial\theta(s,a_{1})} .
\end{align*}
Taking this collectively with \eqref{eqn:upper-gradient-induc-violin}, we reach
%
\begin{align}
-28\cm\eta(1-\gamma)\Big[\widehat{\pi}^{(t)}(a_{1}\mymid s)\Big]^{2}\leq\widehat{\pi}^{(t+1)}(a_{1}\mymid s)-\widehat{\pi}^{(t)}(a_{1}\mymid s)\leq-\frac{\cm\gamma}{35}\eta(1-\gamma)^{2}\Big[\widehat{\pi}^{(t)}(a_{1}\mymid s)\Big]^{2}.
\end{align}
Apply Lemma~\ref{lem:opt-lemma} together with the initialization $\widehat{\pi}^{(0)}(a_1 \mymid s)=1$  to arrive at 
%
\begin{align}
	\frac{1}{1+56\cm \eta(1-\gamma)(t+1)}
\leq 
	\widehat{\pi}^{(t+1)}(a_1 \mymid s) \leq \frac{1}{1+\frac{\cm\gamma}{35} \eta(1-\gamma)^2(t+1)}.
\end{align}

%
%

\paragraph{Proof of the inequality \eqref{eqn:upper-gradient-induc-violin}.}

Recall the gradient expression \eqref{eq:grad-form-following}:
\begin{align}
	\frac{\partial V^{(t)}(\mu)}{\partial \theta(s, a_1)} 
	= \frac{1}{1-\gamma} d^{(t)}_{\mu}(s) \pi^{(t)}(a_1 \mymid s) \Big( Q^{(t)}(s, a_1) - V^{(t)}(s) \Big),
	\label{eq:grad-expression-V-s-a1-567}
\end{align}
each term of which will be bounded separately.
 
The first step is to control $Q^{(t)}(s, a_1) - V^{(t)}(s)$, towards which we start with the following decomposition
\begin{align}
	& Q^{(t)}(s,a_{1})-V^{(t)}(s)  =Q^{(t)}(s,a_{1})-\sum\nolimits_{a\in\{a_{0},a_{1},a_{2}\}}  \pi^{(t)}(a \mymid s)  Q^{(t)}(s,a) \notag\\
 & \qquad \qquad =-\pi^{(t)}(a_{0}\mymid s)\Big(Q^{(t)}(s,a_{0})-Q^{(t)}(s,a_{1})\Big)-\pi^{(t)}(a_{2}\mymid s)\Big(Q^{(t)}(s,a_{2})-Q^{(t)}(s,a_{1})\Big).
	\label{eq:decomposition-Qa1-Vs}
\end{align}
The auxiliary facts stated in Lemma~\ref{lem:induc-auxi-q-diff} (see Appendix~\ref{sec:auxiliary-facts-induction}) imply that
\begin{align*}
	 Q^{(t)}(s, a_0) - Q^{(t)}(s, a_1) 
	\geq 
	Q^{(t)}(s, a_2) - Q^{(t)}(s, a_1)
	\geq 
	(\gamma^{3/2} - \gamma^2)\tau_{s-1},
\end{align*}
while Lemma~\ref{lem:basic-properties-MDP-Vstar} and Lemma~\ref{lem:non-negativity-PG} tell us that
\begin{align*}
	Q^{(t)}(s, a_0) - Q^{(t)}(s, a_1) \leq  V^{\star}(s) - 0 = \gamma^{2s}.
\end{align*}
At the same time, the auxiliary fact \eqref{eqn:induc-q-order} in Lemma~\ref{lem:induc-auxi-q-diff} (see Appendix~\ref{sec:auxiliary-facts-induction}) taken together with the gradient expression \eqref{eq:policy-grad-softmax-expression} guarantees that
\[
	\frac{\partial V^{(t)}(\mu)}{\partial \theta(s, a_0)}  \geq \frac{\partial V^{(t)}(\mu)}{\partial \theta(s, a_2)}  
	\geq \frac{\partial V^{(t)}(\mu)}{\partial \theta(s, a_1)} 
\]
and hence $\theta^{(t)}(s, a_1)\leq \theta^{(t)}(s, a_2) \leq \theta^{(t)}(s, a_0)$ 
(or equivalently $\pi^{(t)}(a_1\mymid s)\leq \pi^{(t)}(a_2\mymid s) \leq \pi^{(t)}(a_0\mymid s)$) during this stage. As a result, 
%
%
\begin{align*}
	\pi^{(t)}(a_1 \mymid s) \leq 1/3 \qquad \text{and}\qquad
	1 \geq \pi^{(t)}(a_0 \mymid s) + \pi^{(t)}(a_2 \mymid s)   \geq 2/3.
\end{align*}
Substituting the preceding bounds into the decomposition \eqref{eq:decomposition-Qa1-Vs}, we arrive at 
\begin{align*}
Q^{(t)}(s,a_{1})-V^{(t)}(s) & \leq-\left(\pi^{(t)}(a_{0}\mymid s)+\pi^{(t)}(a_{2}\mymid s)\right)\min\left\{ Q^{(t)}(s,a_{0})-Q^{(t)}(s,a_{1}),Q^{(t)}(s,a_{2})-Q^{(t)}(s,a_{1})\right\} \\
 & \leq-\frac{2}{3}(\gamma^{\frac{3}{2}}-\gamma^{2})\tau_{s-1}=-\frac{2}{3}\frac{\gamma^{\frac{3}{2}}\tau_{s-1}}{1+\sqrt{\gamma}}(1-\gamma)\le-\frac{1-\gamma}{8},
\end{align*}
provided that $\gamma \geq 0.85$. Meanwhile, it follows from Lemma \ref{lem:basic-properties-MDP-Vstar} and Lemma \ref{lem:non-negativity-PG} that
\[
	Q^{(t)}(s,a_{1})-V^{(t)}(s) \geq 0 - V^{\star}(s) \geq -1. 
\]

Further, from Lemma~\ref{lem:facts-d-pi-s}, we have learned that $\cm\gamma(1-\gamma)^2 \le d^{(t)}_{\mu}(s) \le 14\cm(1-\gamma)^2$ for any $t\leq t_s(\tau_s)$.
Substituting the above bounds into \eqref{eq:grad-expression-V-s-a1-567} and invoking \eqref{eq:relation-pihat-pi}, 
we establish the desired inequality~\eqref{eqn:upper-gradient-induc-violin}.


%

\subsection{Auxiliary facts}
\label{sec:auxiliary-facts-induction}

We now gather a few basic facts that are useful throughout this section. 
The first lemma presents some preliminary facts regarding the difference of Q-function estimates across different actions in the current setting; the proof is deferred to Appendix~\ref{sec:proof-lem:induc-auxi-q-diff}. 
\begin{lemma}
\label{lem:induc-auxi-q-diff}
Consider any $t < t_{s-2}(\taun_{s-2})$. 
Under the assumption \eqref{eq:assumptions-constants}, the following are satisfied
\begin{subequations}
\begin{align}
	Q^{(t)}(s, a_0) > Q^{(t)}(s, a_2) &> Q^{(t)}(s, a_1) , \label{eqn:induc-q-order}\\
	Q^{(t)}(s, a_2) - Q^{(t)}(s, a_1) &\geq (\gamma^{3/2} - \gamma^2)\tau_{s-1} , \label{eqn:induc-Qa2-Qa1-gap} \\
	Q^{(t)}(s, a_0) - Q^{(t)}(s, a_2) 
	&\le \frac{\gamma p}{\frac{\cm\gamma}{2}\eta(1-\gamma) t + \frac{1}{\gamma \tau_{s-2}}} . \label{eqn:induc-differenct-a0-a1}
\end{align}	
%
\end{subequations}
\end{lemma}
\begin{remark}
	Lemma~\ref{lem:induc-auxi-q-diff} makes clear that --- before $t$ exceeds $t_{s-2}(\taun_{s-2})$ --- action $a_0$ is perceived as the best choice, with $a_1$ being the least favorable one. In the meantime, it also reveals that (i) $Q^{(t)}(s, a_2)$ is considerably larger than $Q^{(t)}(s, a_1)$, while  
(ii) the gap between $Q^{(t)}(s, a_0)$  and $Q^{(t)}(s, a_2)$ decays at least as rapidly as $O(1/t)$ in this stage. 
\end{remark}

The second lemma is concerned with the consecutive difference between two rescaled policy iterates. 
The proof can be found in Appendix~\ref{sec:proof-lem:widehat-pi-du-pre}. 
\begin{lemma}
\label{lem:widehat-pi-du-pre}
Suppose that $0<\eta \le (1-\gamma) / 6$.
For any $t \ge 0$ and any $3 \leq s \leq H$, define $\theta_{\mathsf{max}}^{(t)}(s) \defn \max_a \theta^{(t)}(s, a)$. 
If we write
\begin{align}
\widehat{\pi}^{(t+1)}(a_{1}\mymid s)-\widehat{\pi}^{(t)}(a_{1}\mymid s) & =c\widehat{\pi}^{(t)}(a_{1}\mymid s) \Big(\theta^{(t+1)}(s,a_{1})-\theta^{(t)}(s,a_{1})+\theta_{\mathsf{max}}^{(t)}(s)-\theta_{\mathsf{max}}^{(t+1)}(s)\Big)
	\label{eq:hat-pit-consecutive-expression}
\end{align}
for some $c\in \mathbb{R}$, then we necessarily have
\begin{align*}
c\in[1,1.5)\qquad & \text{if }\text{\ensuremath{\theta}}^{(t+1)}(s,a_{1})\geq\theta^{(t)}(s,a_{1})~\text{ and }~\text{\ensuremath{\theta}}_{\mathsf{max}}^{(t)}(s)\geq\text{\ensuremath{\theta}}_{\mathsf{max}}^{(t+1)}(s);\\
c\in(0.72,1]\qquad & \text{if }\text{\ensuremath{\theta}}^{(t+1)}(s,a_{1})\leq\theta^{(t)}(s,a_{1})~\text{ and }~\text{\ensuremath{\theta}}_{\mathsf{max}}^{(t)}(s)\leq\text{\ensuremath{\theta}}_{\mathsf{max}}^{(t+1)}(s). 
\end{align*}
%
\end{lemma}

\subsubsection{Proof of Lemma~\ref{lem:induc-auxi-q-diff}}
\label{sec:proof-lem:induc-auxi-q-diff}
In view of Lemma~\ref{lem:non-negativity-PG}, one has $V^{(t)}(\overline{s-2}) \geq 0$ for all $t\geq 0$. 
Therefore, the relation \eqref{eq:Q-s2-pi-LB-UB-a} yields 
\begin{align*}
	Q^{(t)}(s,a_{2}) = r_s + \gamma p V^{(t)}(\overline{s-2})  \geq r_{s} = \gamma^{3/2}\tau_{s-1}.
\end{align*}
In addition, for any $t < t_{s-2}(\taun_{s-2}) \leq t_{s-1}(\taun_{s-1}) \leq t_{\overline{s-1}}(\gamma \tau_{s-1})$ (see Lemma~\ref{lem:basic-properties-MDP-pi} and Lemma~\ref{lem:init-step}), 
we have $V^{(t)}(\overline{s-1})< \gamma \tau_{s-1}$, and hence it is seen from the relation \eqref{eq:Q-s2-pi-LB-UB-a} that
\begin{align*}
	Q^{(t)}(s, a_2) - Q^{(t)}(s, a_1) = Q^{(t)}(s, a_2) - \gamma  V^{(t)}(\overline{s-1}) 
	\ge (\gamma^{3/2} - \gamma^2)\tau_{s-1} > 0,
\end{align*}
as claimed in \eqref{eqn:induc-Qa2-Qa1-gap}. Also, Part (i) of Lemma~\ref{lem:basic-properties-MDP-pi} tells us that
\begin{align*}
Q^{(t)}(s,a_{0})-Q^{(t)}(s,a_{2}) & =r_{s}+\gamma^{2}p\tau_{s-2}-r_{s}-\gamma pV^{(t)}(\overline{s-2})=\gamma p\Big(\gamma\tau_{s-2}-V^{(t)}(\overline{s-2})\Big) \geq0,
\end{align*}
where the last inequality holds for any $t < t_{s-2}(\taun_{s-2})$ (see Part (iii) of Lemma~\ref{lem:basic-properties-MDP-pi}). 
These taken together validate \eqref{eqn:induc-q-order}.



It remains to justify \eqref{eqn:induc-differenct-a0-a1}, which is the content of the rest of this proof.  
The main step lies in demonstrating that, for any $t < t_{s}(\taun_{s})$ and any $1\leq s\leq H$, 
\begin{align}
\label{eqn:induc-on-V}
\gamma \tau_{s} - V^{(t)}(\overline{s}) \le \frac{1}{\frac{\cm\gamma}{2}\eta(1-\gamma) t + \frac{1}{\gamma \tau_{s}}}. 
\end{align}
If this were true, than taking it together with the following property (which is a consequence of \eqref{eq:Q-s2-pi-LB-UB-a})
\begin{align}
	Q^{(t)}(s,a_{0})-Q^{(t)}(s,a_{2})=\gamma p\left(\gamma\tau_{s-2}-V^{(t)}(\overline{s-2})\right),
	\label{eq:Qt-a0-minus-a2-expression-123}
\end{align} 
would establish the inequality~\eqref{eqn:induc-differenct-a0-a1}.
It then boils down to justifying \eqref{eqn:induc-on-V}. Towards this, we first make the observation that 
\begin{align}
V^{(t)}(\overline{s})-\gamma\tau_{s} & =\pi^{(t)}(a_{0}\mymid\overline{s})Q^{(t)}(\overline{s},a_{0})+\pi^{(t)}(a_{1}\mymid\overline{s})Q^{(t)}(\overline{s},a_{1})-\gamma\tau_{s} \notag\\
 & =\pi^{(t)}(a_{0}\mymid\overline{s})\gamma\tau_{s}+\pi^{(t)}(a_{1}\mymid\overline{s})Q^{(t)}(\overline{s},a_{1})-\gamma\tau_{s} \notag\\
	& =\pi^{(t)}(a_{1}\mymid\overline{s})\Big(Q^{(t)}(\overline{s},a_{1})-\gamma\tau_{s}\Big), \label{eq:Vt-taus-diff-expression}
\end{align}
where the second line holds since $Q^{(t)}(\overline{s},a_{0}) =\gamma\tau_{s}$ (see \eqref{eq:Qpi-s-adj}). 
%
Additionally, recall from the definition that for any $t < t_{s}(\taun_{s}) $, one has $V^{(t)}(s)< \taun_{s}$ and hence
\begin{align}
\frac{\partial V^{(t)}(\mu)}{\partial\theta(\overline{s},a_{1})} & =\frac{1}{1-\gamma}d_{\mu}^{(t)}(\overline{s})\pi^{(t)}(a_{1}\mymid\overline{s})\Big(Q^{(t)}(\overline{s},a_{1})-\pi^{(t)}(a_{0}\mymid\overline{s})Q^{(t)}(\overline{s},a_{0})-\pi^{(t)}(a_{1}\mymid\overline{s})Q^{(t)}(\overline{s},a_{1})\Big) \notag\\
 & =\frac{1}{1-\gamma}d_{\mu}^{(t)}(\overline{s})\pi^{(t)}(a_{0}\mymid\overline{s})\pi^{(t)}(a_{1}\mymid\overline{s})\Big(Q^{(t)}(\overline{s},a_{1})-Q^{(t)}(\overline{s},a_{0})\Big)
	\label{eq:Vt-Qt-123-expression-45}\\
 & =\frac{1}{1-\gamma}d_{\mu}^{(t)}(\overline{s})\pi^{(t)}(a_{0}\mymid\overline{s})\pi^{(t)}(a_{1}\mymid\overline{s})\Big(\gamma V^{(t)}(s)-\gamma\tau_{s}\Big)<0, \nonumber
\end{align}
where the last line makes use of the identities in \eqref{eq:Qpi-s-adj}. 
This means that $\theta^{(t)}(\overline{s}, a_1)$ keeps decreasing, and hence $\theta^{(t)}(\overline{s}, a_1)\leq 0$ given the initialization $\theta^{(0)}(\overline{s}, a_1) =0$.
As an immediate consequence, one has $\theta^{(t)}( \overline{s} , a_0 ) = -\theta^{(t)}( \overline{s} , a_1 ) \ge 0$ and $\pi^{(t)}(a_0 \mymid \overline{s}) \ge 1/2$. 
Taking this observation together with \eqref{eq:Vt-Qt-123-expression-45} and Lemma~\ref{lem:facts-d-pi-s-LB} gives
\begin{align*}
\frac{\partial V^{(t)}(\mu)}{\partial\theta(\overline{s},a_{1})} 
 & =\frac{1}{1-\gamma}d_{\mu}^{(t)}(\overline{s})\pi^{(t)}(a_{0}\mymid\overline{s})\pi^{(t)}(a_{1}\mymid\overline{s})\Big(Q^{(t)}(\overline{s},a_{1})-Q^{(t)}(\overline{s},a_{0})\Big)\\
 & =\frac{1}{1-\gamma}d_{\mu}^{(t)}(\overline{s})\pi^{(t)}(a_{0}\mymid\overline{s})\pi^{(t)}(a_{1}\mymid\overline{s})\Big(Q^{(t)}(\overline{s},a_{1})-\gamma\tau_{s}\Big)\\
 & \le\frac{\cm\gamma}{2}(1-\gamma)\pi^{(t)}(a_{1}\mymid\overline{s})\Big(Q^{(t)}(\overline{s},a_{1})-\gamma\tau_{s}\Big)<0. 
\end{align*}
Moreover, combine \eqref{eq:Vt-Qt-123-expression-45} with Lemma~\ref{lem:facts-d-pi-s} and Lemma \ref{lem:basic-properties-MDP-Vstar} to yield
\begin{align*}
\left|\frac{\partial V^{(t)}(\mu)}{\partial\theta(\overline{s},a_{1})}\right| 
	& \leq\frac{1}{1-\gamma} d_{\mu}^{(t)}(\overline{s})\Big|Q^{(t)}(\overline{s},a_{1})-\gamma\tau_{s}\Big|
	\leq\frac{1}{1-\gamma} 14\cm (1-\gamma)^2 \Big(\big|Q^{(t)}(\overline{s},a_{1}) \big| +  \gamma\tau_{s} \Big)
	\leq 28\cm(1-\gamma),
\end{align*}
%
If $28\cm\eta (1-\gamma)<1/2$, then the above inequalities taken together with Lemma~\ref{lem:init-pi-t-diff} give
\begin{align}
	\pi^{(t+1)}(a_{1}\mymid\overline{s})-\pi^{(t)}(a_{1}\mymid\overline{s}) 
	& \leq\frac{\cm\gamma}{2}\eta(1-\gamma)\Big[\pi^{(t)}(a_{1}\mymid\overline{s})\Big]^{2}\Big(Q^{(t)}(\overline{s},a_{1})-\gamma\tau_{s}\Big) 
	\label{eqn:induc-pi-diff-tchaikovsky}
\end{align}
for all $t < t_{s}(\taun_{s})$. 
This combined with \eqref{eq:Vt-taus-diff-expression} and the monotonicity of $Q^{(t)}(\overline{s}, a_1)$ (see Lemma~\ref{lem:ascent-lemma-PG}) gives
\begin{align*}
\gamma\tau_{s}-V^{(t+1)}(\overline{s}) & =\pi^{(t+1)}(a_{1}\mymid\overline{s})\Big(\gamma\tau_{s}-Q^{(t+1)}(\overline{s},a_{1})\Big)\le\pi^{(t+1)}(a_{1}\mymid\overline{s})\Big(\gamma\tau_{s}-Q^{(t)}(\overline{s},a_{1})\Big)\\
 & \le\Big\{\pi^{(t)}(a_{1}\mymid\overline{s})-\frac{\cm\gamma}{2}\eta(1-\gamma)\Big[\pi^{(t)}(a_{1}\mymid\overline{s})\Big]^{2}\Big(\gamma\tau_{s}-Q^{(t)}(\overline{s},a_{1})\Big)\Big\}\Big(\gamma\tau_{s}-Q^{(t)}(\overline{s},a_{1})\Big)\\
 & =\Big\{1-\frac{\eta\cm\gamma(1-\gamma)}{2}\Big(\gamma\tau_{s}-V^{(t)}(\overline{s})\Big)\Big\}\Big(\gamma\tau_{s}-V^{(t)}(\overline{s})\Big),
\end{align*}
where the penultimate line follows from the inequality~\eqref{eqn:induc-pi-diff-tchaikovsky} for the iteration $t-1$,
and the last identity makes use of \eqref{eq:Vt-taus-diff-expression}.
In conclusion, we have arrived at the following inductive relation 
\begin{align*}
\gamma \tau_s - V^{(t+1)}(\overline{s}) \leq 
\gamma \tau_s - V^{(t)}(\overline{s}) - \frac{\eta\cm\gamma(1-\gamma)}{2} \Big(\gamma \tau_s - V^{(t)}(\overline{s}) \Big)^2,
\end{align*}
which bears resemblance to the recursive relations studied in Lemma \ref{lem:opt-lemma}. 
Recognizing that 
$\gamma \tau_{s} - V^{(0)}(\overline{s}) \leq \gamma \tau_{s-2}$ (since $V^{(0)}(\overline{s})\geq 0$ according to Lemma~\ref{lem:non-negativity-PG}), 
we can invoke Lemma \ref{lem:opt-lemma} to derive
\[
	\gamma\tau_{s}-V^{(t)}(\overline{s})\leq\frac{1}{\frac{\cm\gamma}{2}\eta(1-\gamma)t+\frac{1}{\gamma\tau_{s}}}. 
\]
%
Putting the above pieces together concludes the proof of  \eqref{eqn:induc-differenct-a0-a1}.



\subsubsection{Proof of Lemma~\ref{lem:widehat-pi-du-pre}}
\label{sec:proof-lem:widehat-pi-du-pre}
From the definition \eqref{eqn:defn-widehat-pi-soft-soft}, direct calculations lead to 
\begin{align*}
\widehat{\pi}^{(t+1)}(a_{1}\mymid s)-\widehat{\pi}^{(t)}(a_{1}\mymid s) & =\exp\Big(\theta^{(t+1)}(s,a_{1})-\theta_{\mathsf{max}}^{(t+1)}(s)\Big)-\exp\Big(\theta^{(t)}(s,a_{1})-\theta_{\mathsf{max}}^{(t)}(s)\Big)\\
 & =\widehat{\pi}^{(t)}(a_{1}\mymid s)\Big\{\exp\Big(\theta^{(t+1)}(s,a_{1})-\theta^{(t)}(s,a_{1})+\theta_{\mathsf{max}}^{(t)}(s)-\theta_{\mathsf{max}}^{(t+1)}(s)\Big)-1\Big\}.
\end{align*}
According to Lemma \ref{lem:basic-properties-MDP-Vstar}, we have $|{Q}^{(t)}(s, a)| \le 1$ and $|{V}^{(t)}(s)| \le 1$, which indicates --- for any action $a\in \{a_0,a_1,a_2\}$  --- that 
\[
	\Big|\theta^{(t+1)}(s,a)-\theta^{(t)}(s,a) \Big| = \Big| \eta \frac{\partial V^{(t)}(\mu)}{\partial\theta(s,a)} \Big|
	= \frac{\eta}{1-\gamma} d^{(t)}_{\mu}(s)\pi^{(t)}(a \mymid s) \Big| {Q}^{(t)}(s, a) - {V}^{(t)}(s) \Big| \le \frac{1}{3},
\]
provided that $\eta \le (1-\gamma) / 6$. An immediate consequence is that 
$|\theta_{\mathsf{max}}^{(t+1)}(s)-\theta_{\mathsf{max}}^{(t)}(s)| \le 1/3$ and hence 
\[
	\Big| \theta^{(t+1)}(s,a_{1})-\theta^{(t)}(s,a_{1})+\theta_{\mathsf{max}}^{(t)}(s)-\theta_{\mathsf{max}}^{(t+1)}(s) \Big| \leq 2/3.
\]
This taken together with the following elementary facts
\begin{align*}
	(e^{x}-1)/x\,\in[1,1.5) \quad\text{for}\ 0\leq x<2/3,\qquad\text{and}\qquad(e^{x}-1)/x\,\in(0.72,1]\quad\text{for}\ -2/3<x\leq0 
\end{align*}
establishes the claim~\eqref{eq:hat-pit-consecutive-expression}.

\section{Analysis for the intermediate stage (Lemma \ref{lem:inter})}
\label{sec:analysis-lem:inter}

We now turn attention to Lemma \ref{lem:inter}, which studies the dynamics during an intermediate stage between $t_{s-2}(\taun_{s-2})$ and $t_{\overline{s-1}}(\tau_{s})$.

%


\subsection{Main steps}

\paragraph{Key facts regarding crossing times.}

Our proof for Lemma \ref{lem:inter} relies on  several crucial properties regarding the crossing times for both the key primary states and the adjoint states, as stated in the following two lemmas. 
\begin{lemma}
\label{lem:sweets-for-new-year}
Suppose that \eqref{eq:assumptions-constants} holds.
There exists some constant $0<\czero \leq \frac{1222}{\cm \gamma}$ such that:
\begin{align}
\label{eqn:cake}
t_{s}\big(\gamma^{2s}-1/4 \big) - \max\Big\{t_{\overline{s-1}} \big(\gamma^{2s-1}-1/4 \big), \, t_{s}(\taun_s) \Big\} &\leq \frac{\czero}{\eta(1-\gamma)^2}
\end{align}
holds for every $3\leq s \leq H$, and 
\begin{align}
\label{eqn:ice-cream}
t_{\overline{s}}\big(\gamma^{2s+1}-1/4 \big) - \max\Big\{t_{s} \big(\gamma^{2s}-1/4 \big), \, t_{\overline{s}}(\tau_{s+1}) \Big\} &\leq \frac{\czero}{\eta (1-\gamma)^2}
\end{align}
holds for every $1\leq s \leq H$. 
\end{lemma}
\begin{lemma}
\label{lem:hotpot-for-new-year}
Suppose that \eqref{eq:assumptions-constants} holds and
\begin{align} 
\label{eqn:assmp2-intermediate-lemma}
t_{3}(\taun_3) > t_2(\gamma^4-1/4). 
\end{align}
Then for every $3\leq s \leq H$, we have
\begin{subequations}
\begin{align}
\label{eqn:t-diff-part-I}
t_{s}\big(\gamma^{2s}-1/4) - t_{s}(\taun_s\big) &\leq  \frac{2s c_0}{\eta(1-\gamma)^2}, \\
%
\label{eqn:t-diff-part-II}
t_{\overline{s-1}} \big(\gamma^{2s-1}-1/4\big)- t_{\overline{s-1}}(\tau_{s})	&\leq 
\frac{2s c_0}{\eta(1-\gamma)^2}.
\end{align}
In addition, if we further have $t_{s-1}(\taun_{s-1}) > t_{\overline{s-2}}(\tau_{s-1}) + \frac{2sc_0}{\eta(1-\gamma)^2}$, then 
\begin{align}
	t_{\overline{s-2}}\big(\gamma^{2s-3}-1/4\big) &\le t_{\overline{s-1}}(\tau_{s}). \label{eqn:horn}
\end{align}
\end{subequations}	
Furthermore, \eqref{eqn:horn} still holds for $s = 3$ without requiring the assumption~\eqref{eqn:assmp2-intermediate-lemma}.
\end{lemma}
The proofs of the above two lemmas are postponed to  Appendix~\ref{sec:proof-lem:sweets-for-new-year} and Appendix~\ref{sec:proof-lem:hotpot-for-new-year}, respectively.
Let us take a moment to explain these two lemmas; to provide some intuitions, let us treat $\gamma^{2s}\approx1$. 
Lemma~\ref{lem:sweets-for-new-year} makes clear that: once the value function estimates for states $\overline{s-1}$ and $s$ are both sufficiently large 
(i.e., $V^{(t)}(\overline{s-1})\gtrapprox0.75$ and $V^{(t)}(s)\gtrapprox 0.5$), then it does not take long for $V^{(t)}(s)$ to (approximately) exceed $0.75$. 
A similar message holds true if we replace $s$ (resp.~$\overline{s-1}$) with $\overline{s}$ (resp.~$s$). 
Built upon this observation, Lemma~\ref{lem:hotpot-for-new-year} further reveals that: the time taken for $V^{(t)}(s)$ (resp.~$V^{(t)}(\overline{s-1})$) 
to rise from  $0.5$ to $0.75$ is fairly short.

\paragraph{Proof of Lemma \ref{lem:inter}.}

We are now in a position to present the proof of Lemma \ref{lem:inter}. 
To begin with,  recall from Lemma~\ref{lem:basic-properties-MDP-pi} that: for any $t \le t_{\overline{s-1}}(\tau_{s}) 
\leq t_{\overline{s-1}}(\tau_{s-1})$, one has
\begin{align}
\label{eqn:useful-fact-inter}
	Q^{(t)}(s, a_1) = \gamma V^{(t)}(\overline{s-1}) \le \gamma \tau_{s} \le \min\Big\{Q^{(t)}(s, a_0), Q^{(t)}(s, a_2) \Big\} .
\end{align}
%
Given that $V^{(t)}(s)$ is a convex combination of $\{Q^{(t)}(s, a)\}_{a\in \{a_0,a_1,a_2\}}$, 
one has $V^{(t)}(s)-Q^{(t)}(s, a_1)\geq 0$, which together with the gradient expression \eqref{eq:grad-form-following} indicates that 
\begin{align}
	\frac{\partial{V}^{(t)}(\mu)}{\partial\theta(s,a_{1})} \leq 0
	\label{eq:negative-gradient-Vt-s-a1-intermediate}
\end{align}
and hence  $\theta^{(t)}(s, a_1)$ is non-increasing with $t$ for any $t< t_{\overline{s-1}}(\tau_{s})$.
Additionally, we have learned from Lemma~\ref{lem:hotpot-for-new-year} that
\[
	t_{\overline{s-1}}(\tau_{s}) \geq t_{\overline{s-2}}\big(\gamma^{2s-3}-1/4\big) 
	\geq  t_{\overline{s-2}}\big(\gamma \taun_{s-2}\big) = t_{s-2}\big(\taun_{s-2}\big),
\]
where the second inequality holds since $\gamma^{2s-3}-1/4 \geq \gamma \taun_{s-2}$,
and the last identity results from Part (iii) of Lemma~\ref{lem:basic-properties-MDP-pi}. 
This combined with the non-increasing nature of $\theta^{(t)}(s, a_1)$ readily establishes the advertised inequality 
$\theta^{(t_{\overline{s-1}}(\tau_{s}))}(s, a_1) \le \theta^{(t_{s-2}(\taun_{s-2}))}(s, a_1)$.  

The next step is to justify $\theta^{(t_{\overline{s-1}}(\tau_{s}))}(s, a_2)  \ge 0$. 
Notice that for $t > t_{s-2}(\tau_{s-2})$, we have $V^{(t)}(s-2) > \tau_{s-2}$, and then $V^{(t)}(\overline{s-2}) > \gamma\tau_{s-2}$ by \eqref{eq:Vpi-s-lower-bound-2683}, which leads to $Q^{(t)}(s, a_2) > Q^{(t)}(s, a_0)$ by \eqref{eq:Q-s2-pi-LB-UB-a}  in Lemma~\ref{lem:basic-properties-MDP-pi}.
Recall~\eqref{eqn:useful-fact-inter} that $Q^{(t)}(s, a_1) \le \gamma \tau_{s} \le \min\Big\{Q^{(t)}(s, a_0), Q^{(t)}(s, a_2) \Big\}$.
Then, one has $Q^{(t)}(s, a_2)-V^{(t)}(s)\geq 0$, which together with the gradient expression \eqref{eq:grad-form-following} indicates that 
\begin{align}
	\frac{\partial{V}^{(t)}(\mu)}{\partial\theta(s,a_{2})} \geq 0
	\label{eq:negative-gradient-Vt-s-a2-intermediate}
\end{align}
and hence  $\theta^{(t)}(s, a_2)$ is non-decreasing with $t$ for any $t< t_{\overline{s-1}}(\tau_{s})$.
This  establishes $\theta^{(t_{\overline{s-1}}(\tau_{s}))}(s, a_2)  \ge 0$. 

\subsection{Proof of Lemma~\ref{lem:sweets-for-new-year}}
\label{sec:proof-lem:sweets-for-new-year}

For every $t \geq \max \big\{t_{\overline{s-1}}(\gamma^{2s-1}-1/4), t_{s}(\taun_s) \big\}$, 
we isolate the following properties that will prove useful. 
\begin{itemize}
	\item The definition \eqref{eqn:approx-optimal-time} of $t_{\overline{s-1}}(\cdot)$ together with the monotonicity property in Lemma~\ref{lem:ascent-lemma-PG} requires that
 $V^{(t)}(\overline{s-1}) \geq \gamma^{2s-1}-1/4$,
 and hence it is seen from \eqref{eq:Q-s2-pi-LB-UB-a} that 
	\begin{equation}
		Q^{(t)}(s, a_1) = \gamma V^{(t)}(\overline{s-1}) \ge \gamma^{2s}-\gamma /4.
		\label{eq:Qt-s-a1-LB-gamma-2s}
	\end{equation}
	\item In the meantime, since $t \geq t_{s}(\taun_s)$, Lemma~\ref{lem:basic-properties-MDP-pi} (cf.~\eqref{eq:pi-a2-s-LB-V-large-8923}) guarantees that  
	\begin{equation}
		\pi^{(t)}(a_{1} \mymid s) \geq (1-\gamma)/2.
		\label{eq:pit-a1-lower-bound-intermediate}
	\end{equation}
	\item Given that $t_{s}(\taun_s) \geq t_{s-2}(\tau_{s-2})$ (see \eqref{eq:monotonicity-t-s-tau-s}) and the monotonicity property in Lemma~\ref{lem:ascent-lemma-PG}, one has $V^{(t)}(\overline{s-2})\geq \tau_{s-2}$, and thus we can see from \eqref{eq:Q-s2-pi-LB-UB-a} that
	\begin{equation}
		Q^{(t)}(s,a_{2}) - Q^{(t)}(s,a_{0}) = \gamma p \big( V^{(t)}(\overline{s-2}) - \tau_{s-2} \big)  \geq 0.
		\label{eq:Qt-a2-Qt-a0-order-justification}
	\end{equation}
	\item In addition, Lemma~\ref{lem:basic-properties-MDP-pi} ensures that both $Q^{(t)}(s,a_{2}) $ and $Q^{(t)}(s,a_{0})$ are bounded above by $\gamma^{1/2}\tau_{s}$. 		Therefore, it is easily seen that 
	\begin{align}
		\label{eqn:inter-q-relation}
		Q^{(t)}(s,a_{1}) \geq \gamma^{2s} - \gamma / 4  > \gamma^{1/2}\tau_{s}
		\geq Q^{(t)}(s,a_{2}) \geq Q^{(t)}(s,a_{0}), 
	\end{align}
	where the first inequality comes from \eqref{eq:Qt-s-a1-LB-gamma-2s}, the second one holds when $\gamma^{2s}>0.75$, and the last inequality has been justified in \eqref{eq:Qt-a2-Qt-a0-order-justification}. 
	\item Moreover, given that $V^{(t)}(s)\geq \taun_s$ (since $t \geq t_{s}(\taun_s)$), one further has
	\begin{align}
		Q^{(t)}(s,a_{1})-\max\big\{Q^{(t)}(s,a_{2}), Q^{(t)}(s,a_{0}) \big\} 
		&> V^{(t)}(s) -\max\big\{Q^{(t)}(s,a_{2}), Q^{(t)}(s,a_{0}) \big\}  \notag \\
		&> \taun_s - \gamma^{1/2}\tau_{s} > 0.  \label{eq:Qt-a1-Vt-Qa02-gap-intermediate}
	\end{align}
	Here, the first inequality comes from \eqref{eqn:inter-q-relation}, while the penultimate inequality is a consequence of \eqref{eqn:inter-q-relation}.
	\item We have seen from the above bullet points that 
		\begin{align}
			Q^{(t)}(s,a_{1}) > V^{(t)}(s) > \max\big\{Q^{(t)}(s,a_{2}), Q^{(t)}(s,a_{0}) \big\} ,
			\label{eq:Qt-Vt-ordering-intermediate-stage-123}
		\end{align}
		which combined with the gradient expression~\eqref{eq:grad-form-following} reveals that
	\begin{align}
		\frac{\partial V^{(t)}(\mu)}{\partial\theta(s,a_{1})} >0 
		> \max\left\{ \frac{\partial V^{(t)}(\mu)}{\partial\theta(s,a_{0})}, \frac{\partial V^{(t)}(\mu)}{\partial\theta(s,a_{2})} \right\}.
		\label{eq:grad-ordering-intermediate-stage}
	\end{align}
	%
\end{itemize}
With the above properties in place, we are now ready to prove our lemma, for which we shall look at the key primary states $3\leq s \leq H$ and the adjoint states separately.

\paragraph{\bf Analysis for the key primary states.} 
Let us start with any state $3\leq s \leq H$, and control $t_s(\gamma^{2s}-1/4)$ as claimed in \eqref{eqn:cake}. 
As before,  define $$\theta_{\mathsf{max}}^{(t)}(s) \defn \max\nolimits_a \theta^{(t)}(s, a).$$ 
From the above fact \eqref{eq:grad-ordering-intermediate-stage}, we know that $\theta^{(t)}(s, a_1)$ keeps increasing with $t$ while $\theta^{(t)}(s, a_0), \theta^{(t)}(s, a_2)$ are both decreasing with $t$. As a result, once $\theta^{(t)}(s, a_1) = \theta_{\mathsf{max}}^{(t)}(s)$, then $\theta^{(t)}(s, a_1)$ will remain equal to $\theta_{\mathsf{max}}^{(t)}(s)$ for the subsequent iterations. This allows us to divide into two stages as follows. 
\begin{itemize}
\item {\bf Stage 1: the duration when $\theta^{(t)}(s, a_1) < \theta_{\mathsf{max}}^{(t)}(s)$.} 
	Our aim is to show that this stage contains at most $O\big(\frac{1}{ \eta  (1-\gamma)^2} \big)$ iterations.
In order to prove this, the starting point is again the gradient expression~\eqref{eq:grad-form-following}:  
\begin{align}
\label{eqn:gradient-relation-inter}
	\frac{\partial V^{(t)}(\mu)}{\partial \theta(s, a_1)} &= \frac{1}{1-\gamma} d^{(t)}_{\mu}(s) \pi^{(t)}(a_1 \mymid s)\big(Q^{(t)}(s, a_1) - V^{(t)}(s) \big) \notag\\
	&\geq \cm\gamma(1-\gamma)\pi^{(t)}(a_1 \mymid s)\big( Q^{(t)}(s, a_1) - V^{(t)}(s) \big) , 
\end{align}
where the last line relies on Lemma~\ref{lem:facts-d-pi-s-LB} and the fact $Q^{(t)}(s, a_1) > V^{(t)}(s)$ (cf.~\eqref{eq:Qt-Vt-ordering-intermediate-stage-123}). 
Regarding the size of $Q^{(t)}(s, a_1) - V^{(t)}(s)$, we make the observation that
\begin{align*}
	Q^{(t)}(s, a_1) - V^{(t)}(s) &= 
	\pi^{(t)}(a_ 0 \mymid s) \Big(Q^{(t)}(s, a_1) - Q^{(t)}(s, a_0) \Big) + \pi^{(t)}(a_2 \mymid s) \Big(Q^{(t)}(s, a_1) - Q^{(t)}(s, a_2) \Big)\\
	& \geq \big(\pi^{(t)}(a_ 0 \mymid s) + \pi^{(t)}(a_2 \mymid s) \big)
	\Big(Q^{(t)}(s, a_1) - \max_{a\in \{a_0,a_2\}} Q^{(t)}(s, a)\Big)\\
	&\stackrel{(\mathrm{i})}{\geq} \frac{1}{2} \Big(Q^{(t)}(s, a_1) - \max_{a\in \{a_0,a_2\}} Q^{(t)}(s, a)\Big)
	\stackrel{(\mathrm{ii})}{\geq}  \frac{1}{2} \big(\gamma^{2s}- \gamma /4 - \gamma^{1/2}\tau_s \big) 
	\stackrel{(\mathrm{iii})}{\geq} \frac{1}{16}.
\end{align*}
Here, (i) follows since $\theta^{(t)}(s, a_1) < \theta_{\mathsf{max}}^{(t)}(s)$ during this stage and, therefore, $\pi^{(t)}(a_1 \mymid s) \leq 1/2$;
(ii) arises from the relation~\eqref{eqn:inter-q-relation}; and (iii)
holds whenever $\gamma^{2s} - \gamma / 4   > 5/8$.
Substitution into \eqref{eqn:gradient-relation-inter} yields
\begin{align}
	\frac{\partial V^{(t)}(\mu)}{\partial \theta(s, a_1)} 
	\geq 
	 \frac{1}{16} \cm\gamma(1-\gamma)\pi^{(t)}(a_1 \mymid s)
	 \geq 
	 \frac{1}{48} \cm\gamma(1-\gamma) \widehat{\pi}^{(t)}(a_1 \mymid s),
	 \label{eq:V-partial-lower-bound-intermediate-123}
\end{align}
where the last inequality comes from \eqref{eq:relation-pihat-pi}. 
In addition, recall that $\theta^{(t)}(s, a_1)$ is increasing with $t$, while $\theta^{(t)}(s, a_0)$ and $\theta^{(t)}(s, a_2)$ are both decreasing (and hence $\theta_{\mathsf{max}}^{(t)}(s)$ is also decreasing).  
Invoking Lemma~\ref{lem:widehat-pi-du-pre} then yields
\begin{align*}
\widehat{\pi}^{(t+1)}(a_{1}\mymid s)-\widehat{\pi}^{(t)}(a_{1}\mymid s) & \ge\widehat{\pi}^{(t)}(a_{1}\mymid s)\Big(\text{\ensuremath{\theta}}^{(t+1)}(s,a_{1})-\theta^{(t)}(s,a_{1})+\text{\ensuremath{\theta}}_{\mathsf{max}}^{(t)}(s)-\text{\ensuremath{\theta}}_{\mathsf{max}}^{(t+1)}(s)\Big)\\
 & \geq\widehat{\pi}^{(t)}(a_{1}\mymid s)\Big(\text{\ensuremath{\theta}}^{(t+1)}(s,a_{1})-\theta^{(t)}(s,a_{1})\Big)\\
 & =\widehat{\pi}^{(t)}(a_{1}\mymid s)\cdot\eta\frac{\partial \soft{V}^{(t)}(\mu)}{\partial\theta(s,a_{1})} 
   \geq \frac{1}{48} \cm\eta\gamma(1-\gamma) \Big[ \widehat{\pi}^{(t)}(a_1 \mymid s) \Big]^2,
\end{align*}
where the last line arises from \eqref{eq:V-partial-lower-bound-intermediate-123}. 
Given this recursive relation, Lemma~\ref{lem:opt-lemma} implies that: 
if $\widehat{\pi}^{(t)}(a_{1}\mymid s) < 1$ (or equivalently, $\theta^{(t)}(s, a_1) < \theta_{\mathsf{max}}^{(t)}(s)$), then one necessairly has 
\begin{align*}
	t - t_{0,1} \leq
	\frac{1+ \frac{1}{48} \cm\eta\gamma(1-\gamma) }{  \frac{1}{48} \cm\eta\gamma(1-\gamma) \pi^{(t_0)}(a_{1} \mymid s) } 
	\leq 
	\frac{2 }{  \frac{1}{48} \cm\eta\gamma(1-\gamma) \pi^{(t_{0,1})}(a_{1} \mymid s) }
	\leq \frac{240}{\cm\eta\gamma(1-\gamma)^2}, 
\end{align*}
with $t_{0,1} \coloneqq \max \big\{t_{\overline{s-1}}(\gamma^{2s-1}-1/4), t_{s}(\taun_s) \big\}$. 
Here, the last inequality relies on the property \eqref{eq:pit-a1-lower-bound-intermediate}.

\item {\bf Stage 2: the duration when $\theta^{(t)}(s, a_1) = \theta_{\mathsf{max}}^{(t)}(s)$.} 
	For this stage, we intend to demonstrate that it takes at most $O\big( \frac{1}{\eta(1-\gamma)^2} \big)$ iterations to achieve $\max \big\{ \pi^{(t)}(a_0 \mymid s),\pi^{(t)}(a_2 \mymid s) \big\} \leq (1-\gamma)/8$.  
To this end, we again begin by studying the gradient as follows: 
\begin{align*}
	\frac{\partial V^{(t)}(\mu)}{\partial \theta(s, a_2)} 
	&=\frac{1}{1-\gamma} d^{(t)}_{\mu}(s) \pi^{(t)}(a_2 \mymid s) \Big( Q^{(t)}(s, a_2) - V^{(t)}(s) \Big) 
	\leq  \cm\gamma (1-\gamma)\pi^{(t)}(a_2 \mymid s) \Big( Q^{(t)}(s, a_2) - V^{(t)}(s) \Big) \\
	& \leq \frac{1}{3}\cm\gamma (1-\gamma)\widehat{\pi}^{(t)}(a_2 \mymid s) \Big(Q^{(t)}(s, a_2) - V^{(t)}(s) \Big).
\end{align*}
Here, the first inequality comes from Lemma~\ref{lem:facts-d-pi-s-LB} and the fact $Q^{(t)}(s, a_2) < V^{(t)}(s)$ (see \eqref{eq:Qt-Vt-ordering-intermediate-stage-123}),
whereas the last inequality is a consequence of \eqref{eq:relation-pihat-pi}.  
%
In order to control $Q^{(t)}(s, a_2) - V^{(t)}(s)$, we observe that
\begin{align*}
Q^{(t)}(s,a_{2})-V^{(t)}(s) & =\pi^{(t)}(a_{1}\mymid s)\Big(Q^{(t)}(s,a_{2})-Q^{(t)}(s,a_{1})\Big)+\pi^{(t)}(a_{0}\mymid s)\Big(Q^{(t)}(s,a_{2})-Q^{(t)}(s,a_{0})\Big)\\
 & \leq\pi^{(t)}(a_{1}\mymid s)\Big(\gamma^{1/2}\tau_{s}-\gamma^{2s}+\gamma/4\Big)+\pi^{(t)}(a_{2}\mymid s)\gamma p\left(V^{(t)}(\overline{s-2})-\tau_{s-2}\right)\\
 & \leq\frac{1}{3}\Big(\gamma^{1/2}\tau_{s}-\gamma^{2s}+\gamma/4\Big)+\gamma p\leq-\frac{1}{24},
\end{align*}
where the second line arises from \eqref{eqn:inter-q-relation} and \eqref{eq:Qt-a2-Qt-a0-order-justification},
and the last line holds since  $V^{(t)}(\overline{s-2})\leq 1$ as well as the fact $\pi^{(t)}(a_{1}\mymid s)\geq 1/3$ during this stage (since $\theta^{(t)}(s, a_1) = \theta_{\mathsf{max}}^{(t)}(s)$). Putting the above two bounds together leads to
\begin{align}
	\frac{\partial V^{(t)}(\mu)}{\partial \theta(s, a_2)} 
	 \leq - \frac{1}{72}\cm\gamma (1-\gamma)\widehat{\pi}^{(t)}(a_2 \mymid s) .
	 \label{eq:grad-Vt-a2-bound-intermediate}
\end{align}
Next, Lemma \ref{lem:widehat-pi-du-pre} tells us that
\begin{align*}
\widehat{\pi}^{(t+1)}(a_{2}\mymid s)-\widehat{\pi}^{(t)}(a_{2}\mymid s) & \leq0.72\widehat{\pi}^{(t)}(a_{2}\mymid s)\Big(\text{\ensuremath{\theta}}^{(t+1)}(s,a_{2})-\theta^{(t)}(s,a_{2})+\text{\ensuremath{\theta}}_{\mathsf{max}}^{(t)}(s)-\text{\ensuremath{\theta}}_{\mathsf{max}}^{(t+1)}(s)\Big)\\
 & \leq0.72\widehat{\pi}^{(t)}(a_{2}\mymid s)\Big(\text{\ensuremath{\theta}}^{(t+1)}(s,a_{2})-\theta^{(t)}(s,a_{2})\Big)=0.72\widehat{\pi}^{(t)}(a_{2}\mymid s)\cdot\eta\frac{\partial \soft{V}^{(t)}(\mu)}{\partial\theta(s,a_{2})}\\
 & \leq - 0.01\eta\cm\gamma(1-\gamma)\Big[\widehat{\pi}^{(t)}(a_{2}\mymid s)\Big]^{2},
\end{align*}
where the first inequality makes use of the facts $\theta^{(t+1)}(s,a_{2})\leq \theta^{(t)}(s,a_{2})$
and $\theta_{\mathsf{max}}^{(t)}(s)= \theta^{(t)}(s,a_{1}) \leq \theta^{(t+1)}(s,a_{1}) = \theta_{\mathsf{max}}^{(t+1)}(s)$ (see \eqref{eq:grad-ordering-intermediate-stage}). 
Denoting by $t_{0,2}$ the first iteration in this stage, we can invoke Lemma~\ref{lem:opt-lemma} to reach
\begin{align} \label{eq:pi-a2-intermediate}
\widehat{\pi}^{(t-t_{0,2})}(a_{2}\mymid s)\leq\frac{1}{0.01\eta\cm\gamma(1-\gamma)(t-t_{0,2})+1} .
\end{align}
As a consequence, once $t-t_{0,2}$ exceeds 
\[
	\frac{800}{\eta\cm\gamma(1-\gamma)^{2}}, 
\]
then one has $\pi^{(t)}(a_2 \mymid s) \leq (1-\gamma)/8$. The same conclusion holds for $a_0$ as well.

\end{itemize}


Combining the above analysis for the two stages, we see that: if 
\[
	t - t_{0,1} \geq   \frac{240}{\eta\cm\gamma(1-\gamma)^{2}} + \frac{800}{\eta\cm\gamma(1-\gamma)^{2}} = \frac{1040}{\eta\cm\gamma(1-\gamma)^{2}}
\]
with $t_{0,1} \coloneqq \max \big\{t_{\overline{s-1}}(\gamma^{2s-1}-1/4), t_{s}(\taun_s) \big\}$, then one has 
\[
	\pi^{(t)}(a_1 \mymid s) = 1 - \pi^{(t)}(a_0 \mymid s) - \pi^{(t)}(a_2 \mymid s) \geq 1 -  (1-\gamma)/4, 
\]
which combined with \eqref{eqn:inter-q-relation} leads to
\begin{align*}
	V^{(t)} (s) \geq \pi^{(t)}(a_1 \mymid s) Q^{(t)}(s,a_{1}) 
	\geq \big(1 -  (1-\gamma)/4 \big) \big( \gamma^{2s} - \gamma / 4\big) \geq \gamma^{2s} - 1/4.
\end{align*}
%
%
This means that one necessairly has $t \geq t_{s}(\gamma^{2s}-1/4)$. It then follows that  
\begin{align*}
	t_{s}\big(\gamma^{2s}-1/4\big) - \max\big\{t_{\overline{s-1}}\big(\gamma^{2s-1}-1/4 \big), t_{s}(\taun_s) \big\}
	= t_{s}\big(\gamma^{2s}-1/4\big)  - t_{0,1}
	\leq \frac{1040}{\eta\cm \gamma (1-\gamma)^2},
\end{align*}
thus concluding the proof of \eqref{eqn:cake}.

\paragraph{\bf Analysis for the adjoint states.} 
We then move forward to the adjoint states $\{\overline{1},\cdots,\overline{H}\}$ and control $t_{\overline{s}}(\gamma^{2s+1}-1/4)$ as desired in \eqref{eqn:ice-cream}. 
The proof consists of studying the dynamic for any $t$ obeying 
$$
	\max\Big\{t_{s} \big(\gamma^{2s}-1/4 \big), \, t_{\overline{s}}(\tau_{s+1}) \Big\} \leq t \leq t_{\overline{s}} \big( \gamma^{2s+1}-1/4 \big).
$$
%
Once again, we divide into two stages and analyze each of them separately.

\begin{itemize}
\item {\bf Stage 1: the duration where $\theta^{(t)}(\overline{s}, a_1) < \theta^{(t)}(\overline{s}, a_0)$.} 
We aim to demonstrate that it takes no more than $O\big( \frac{1}{ \eta  (1-\gamma)^2} \big)$ iterations for $\theta^{(t)}(\overline{s}, a_1)$ to surpass $\theta^{(t)}(\overline{s}, a_0)$.
In order to do so, note that
\begin{align}
\frac{\partial V^{(t)}(\mu)}{\partial\theta(\overline{s},a_{1})} & =\frac{1}{1-\gamma}d_{\mu}^{(t)}(\overline{s})\pi^{(t)}(a_{1}\mymid\overline{s})\pi^{(t)}(a_{0}\mymid\overline{s})\Big(Q^{(t)}(\overline{s},a_{1})-Q^{(t)}(\overline{s},a_{0})\Big) \notag\\
 & \geq\frac{1}{16}\cm\gamma(1-\gamma)\pi^{(t)}(a_{1}\mymid\overline{s}) > 0.
	\label{eq:grad-V-a1-LB-5283}
\end{align}
Here, the last line applies Lemma~\ref{lem:facts-d-pi-s-LB} and makes use of the fact
\begin{align}
	Q^{(t)}(\overline{s}, a_1) - Q^{(t)}(\overline{s}, a_0) = \gamma V^{(t)}(s) - \gamma \tau_s \ge 
	\gamma (\gamma^{2s}-1/4 - \tau_s) \geq 1/8.
	\label{eq:Qt-a1-a0-ga-LB-intermediate-adjoint}
\end{align}
where the inequality comes from the assumption $t\geq t_{s} \big(\gamma^{2s}-1/4 \big)$ as well as the monotonicity property in Lemma~\ref{lem:ascent-lemma-PG}.  
As a result, the PG update rule \eqref{eq:PG-update} implies that  $\theta^{(t)}(\overline{s},a_{1})$ is increasing in $t$,
and hence $\theta^{(t)}(\overline{s},a_{0})$ is decreasing in $t$ (since $\sum_{a} \theta^{(t)}(s, a) = 0$); these taken collectively mean that
$$
		\theta^{(t+1)}(\overline{s},a_{1}) - \theta^{(t)}(\overline{s},a_{1}) + \theta^{(t)}(\overline{s},a_{0}) - \theta^{(t+1)}(\overline{s},a_{0}) 
		\geq \theta^{(t+1)}(\overline{s},a_{1}) - \theta^{(t)}(\overline{s},a_{1}) \geq 0 .
$$
Invoking Lemma~\ref{lem:widehat-pi-du-pre} then reveals that
\begin{align*}
\widehat{\pi}^{(t+1)}(a_{1}\mymid\overline{s})-\widehat{\pi}^{(t)}(a_{1}\mymid\overline{s}) & \ge\widehat{\pi}^{(t)}(a_{1}\mymid\overline{s})\Big(\theta^{(t+1)}(\overline{s},a_{1})-\theta^{(t)}(\overline{s},a_{1})+\theta^{(t)}(\overline{s},a_{0})-\theta^{(t+1)}(\overline{s},a_{0})\Big)\\
 & \geq\widehat{\pi}^{(t)}(a_{1}\mymid\overline{s})\Big(\theta^{(t+1)}(\overline{s},a_{1})-\theta^{(t)}(\overline{s},a_{1})\Big)=\eta\widehat{\pi}^{(t)}(a_{1}\mymid\overline{s})\frac{\partial \soft{V}^{(t)}(\mu)}{\partial\theta(\overline{s},a_{1})}\\
	& \geq \frac{1}{48} \eta\cm\gamma(1-\gamma) \Big[ \widehat{\pi}^{(t)}(a_{1}\mymid\overline{s}) \Big]^2,
\end{align*}
where the last inequality relies on \eqref{eq:grad-V-a1-LB-5283} and \eqref{eq:relation-pihat-pi}. 
Given this recursive relation, Lemma~\ref{lem:opt-lemma} tells us that: one has $\widehat{\pi}^{(t)}(a_{1}\mymid\overline{s}) \geq 1$ (which means $a_1$ becomes the favored action by \eqref{eqn:defn-widehat-pi-soft-soft}) as soon as $t-t_{0, 3}$ exceeds
\begin{align*}
	\frac{2}{\frac{1}{48} \eta\cm\gamma(1-\gamma) \widehat{\pi}^{(t_{0,3})}(a_{1}\mymid\overline{s}) } 
	\leq \frac{96}{\eta\cm\gamma(1-\gamma) {\pi}^{(t_{0,3})}(a_{1}\mymid\overline{s}) }
	\leq \frac{1152}{\eta\cm\gamma(1-\gamma)^2  },
\end{align*}
where $t_{0, 3}\coloneqq \max\Big\{t_{s} \big(\gamma^{2s}-1/4 \big), \, t_{\overline{s}}(\tau_{s+1}) \Big\}$. 
Here, the last inequality is valid as long as 
\begin{align}
	\pi^{(t_{0,3})}(a_{1} \mymid \overline{s}) \geq (1-\gamma)/12
	\label{eq:pit-a1-bars-12-LB}
\end{align}
holds. It thus remains to justify \eqref{eq:pit-a1-bars-12-LB}. Towards this end, observe that for any $t\geq t_{\overline{s}}(\tau_{s+1})$,
\begin{align*}
\tau_{s+1} & \leq V^{(t)}(\overline{s})=\pi^{(t)}(a_{0}\mymid\overline{s})Q^{(t)}(\overline{s},a_{0})+\pi^{(t)}(a_{1}\mymid\overline{s})Q^{(t)}(\overline{s},a_{1})=\pi^{(t)}(a_{0}\mymid\overline{s})\gamma\tau_{s}+\pi^{(t)}(a_{1}\mymid\overline{s})\gamma V^{(t)}(s)\\
 & =\gamma\tau_{s}+\pi^{(t)}(a_{1}\mymid\overline{s})\gamma\left(V^{(t)}(s)-\tau_{s}\right)\leq\gamma\tau_{s}+\pi^{(t)}(a_{1}\mymid\overline{s})\gamma,
\end{align*}
and, as a result, 
\begin{align*}
\pi^{(t)}(a_{1}\mymid\overline{s}) & \geq\frac{\tau_{s+1}}{\gamma}-\tau_{s}=\frac{1}{2}\frac{\gamma^{\frac{2}{3}s+\frac{2}{3}}-\gamma^{\frac{2}{3}s+1}}{\gamma}=\frac{\gamma^{\frac{2}{3}s-1}}{2}\left(\gamma^{\frac{2}{3}}-\gamma\right)\geq\frac{1-\gamma}{12},
\end{align*}
provided that $\gamma \geq 0.9$ (so that $\gamma^{\frac{2}{3}}-\gamma\geq 0.3(1-\gamma)$) 
and $\gamma^{\frac{2}{3}H} \geq 0.7$.  This concludes the analysis of this stage.

\item {\bf Stage 2: the duration where $\theta^{(t)}(\overline{s}, a_1) \geq \theta^{(t)}(\overline{s}, a_0)$.}  
	Similar to the above argument, we intend to show that it takes at most $O\big( \frac{1}{\eta (1-\gamma)^2} \big)$ iterations for
$\pi^{(t)}(a_0 \mymid \overline{s}) \leq 1-\gamma$ to occur.
From the gradient expression and the property \eqref{eq:Qt-a1-a0-ga-LB-intermediate-adjoint}, we obtain
\begin{align*}
	\frac{\partial V^{(t)}(\mu)}{\partial \theta(\overline{s}, a_0)} 
	&= \frac{1}{1-\gamma} d^{(t)}_{\mu}(\overline{s}) \pi^{(t)}(a_0 \mymid  \overline{s})\pi^{(t)}(a_1 \mymid \overline{s})
	\Big( Q^{(t)}(\overline{s}, a_0) - Q^{(t)}(\overline{s}, a_1) \Big) \\
	& 
	\leq -\frac{1}{16}\cm\gamma (1-\gamma) {\pi}^{(t)}(a_0 \mymid \overline{s}  )
	\leq -\frac{1}{48}\cm\gamma (1-\gamma) \widehat{\pi}^{(t)}(a_0 \mymid \overline{s} ),
\end{align*}
where the first inequality uses Lemma~\ref{lem:facts-d-pi-s-LB} and the property $\pi^{(t)}(a_1 \mymid  \overline{s})\geq 1/2$ (since $\theta^{(t)}(\overline{s}, a_1) \geq \theta^{(t)}(\overline{s}, a_0)$), and the last inequality relies on \eqref{eq:relation-pihat-pi}. 
Repeating a similar argument as above, we can demonstrate that
\begin{align}
	\widehat{\pi}^{(t+1)}(a_0 \mymid \overline{s}  ) - \widehat{\pi}^{(t)}(a_0 \mymid \overline{s} ) 
	\leq - \frac{1}{70} \eta\cm\gamma(1-\gamma) \Big[\widehat{\pi}^{(t)}(a_0 \mymid \overline{s} )\Big]^2.
\end{align}
This combined with Lemma~\ref{lem:opt-lemma} implies that
\begin{align} \label{eq:pi-a2-bar-intermediate}
	\widehat{\pi}^{(t)}(a_0 \mymid \overline{s} ) \leq \frac{1}{\frac{1}{70} \eta\cm\gamma(1-\gamma) (t-t_{0,4}) +1},
\end{align}
with $t_{0,4}$ denoting the first iteration of this stage. 
Consequently, one has $\widehat{\pi}^{(t)}(a_0 \mymid \overline{s} ) \leq 1-\gamma$ --- and therefore $\pi^{(t)}(a_0 \mymid \overline{s} )\leq  1-\gamma$ according to \eqref{eq:relation-pihat-pi} --- as soon as $t-t_{0,4}$ exceeds
$$
	\frac{70}{ \eta\cm\gamma(1-\gamma)^2 }.
$$

\end{itemize}

Finally,  if $\pi^{(t)}(a_0 \mymid \overline{s} )\leq 1-\gamma$, then one has 
\begin{align*}
V^{(t)}(\overline{s}) & =\pi^{(t)}(a_{0}\mymid\overline{s})Q^{(t)}(\overline{s},a_{0})+\pi^{(t)}(a_{1}\mymid\overline{s})Q^{(t)}(\overline{s},a_{1})=\pi^{(t)}(a_{0}\mymid\overline{s})\gamma\tau_{s}+\pi^{(t)}(a_{1}\mymid\overline{s}) \gamma V^{(t)}(s)\\
 & \geq\pi^{(t)}(a_{0}\mymid\overline{s})\gamma\tau_{s}+\Big(1-\pi^{(t)}(a_{0}\mymid\overline{s})\Big)\gamma\left(\gamma^{2s}-1/4\right)=\pi^{(t)}(a_{0}\mymid\overline{s})\left(\gamma\tau_{s}-\gamma^{2s+1}-\gamma/4\right)+\gamma^{2s+1}-\gamma/4\\
 & \geq(1-\gamma)\left(\gamma\tau_{s}-\gamma^{2s+1}-\gamma/4\right)+\gamma^{2s+1}-\gamma/4\geq\gamma^{2s+1}-1/4,
\end{align*}
where the first inequality holds by recalling that $t\geq t_{s}(\gamma^{2s}-1/4)$. 
Consequently, putting the above pieces (regarding the duration of the two stages) together allows us to conclude that
\[
	t_{\overline{s}} \big( \gamma^{2s+1}-1/4 \big) - \max\Big\{t_{s} \big(\gamma^{2s}-1/4 \big), \, t_{\overline{s}}(\tau_{s+1}) \Big\} 
	\leq \frac{1152}{\eta\cm\gamma(1-\gamma)^2  } + \frac{70}{ \eta\cm\gamma(1-\gamma)^2 } = \frac{1222}{\eta\cm\gamma(1-\gamma)^2  }
\]
as claimed.

\subsection{Proof of Lemma~\ref{lem:hotpot-for-new-year}}
\label{sec:proof-lem:hotpot-for-new-year}

Before proceeding, we first single out two properties that play a crucial role in the proof of Lemma~\ref{lem:hotpot-for-new-year}. 
\begin{lemma}
\label{lem:sunshine-etude}
	The following basic properties hold true for any $2\leq s\leq H$: 
\begin{subequations}
\begin{align}
\label{eqn:catan:round1}
	t_{s}(\taun_s) &\geq t_{\overline{s-1}}(\tau_{s});\\
\label{eqn:catan:round2}
	t_{\overline{s-1}}(\tau_{s}) &\geq t_{s-1}(\taun_{s-1}).		
\end{align}		
\end{subequations}	
\end{lemma}

The proof of this auxiliary lemma is deferred to the end of this subsection. 
Equipped with this result, we are positioned to present the proof of Lemma~\ref{lem:hotpot-for-new-year}. 
To begin with, we seek to bound the quantity $t_{s}(\gamma^{2s}-1/4) - t_{s}(\taun_s)$.  
Apply Lemma~\ref{lem:sweets-for-new-year} with a little algebra to yield
\begin{align}
\notag 
	t_{s}(\gamma^{2s}-1/4) - t_{s}(\taun_s) 
	&\le \max \Big\{t_{\overline{s-1}}(\gamma^{2s-1}-1/4), t_{s}(\taun_s) \Big\}  + \frac{c_0}{\eta(1-\gamma)^2} - t_{s}(\taun_s)\\
	&= \max\Big\{t_{\overline{s-1}}(\gamma^{2s-1}-1/4)- t_{s}(\taun_s), 0\Big\}  + \frac{c_0}{\eta(1-\gamma)^2}. 
\label{eqn:inter-t-diff-upp}
\end{align}
With the assistance of the bound \eqref{eqn:catan:round1} in Lemma~\ref{lem:sunshine-etude}, we can
continue the bound in \eqref{eqn:inter-t-diff-upp} to derive
\begin{align}
\notag t_{s}(\gamma^{2s}-1/4) - t_{s}(\taun_s) 
	& \leq  \max \Big\{t_{\overline{s-1}}(\gamma^{2s-1}-1/4)- t_{\overline{s-1}}(\tau_{s}), 0 \Big\}  + \frac{c_0}{\eta(1-\gamma)^2}\\
&= t_{\overline{s-1}}(\gamma^{2s-1}-1/4)- t_{\overline{s-1}}(\tau_{s})  + \frac{c_0}{\eta(1-\gamma)^2}.
\label{eqn:cosette}
\end{align}

To continue, we shall bound the quantity $t_{\overline{s-1}}(\gamma^{2s-1}-1/4)- t_{\overline{s-1}}(\tau_{s})$. 
Similar to the derivation of the inequality~\eqref{eqn:inter-t-diff-upp}, we can apply Lemma~\ref{lem:sweets-for-new-year} to show that 
\begin{align}
t_{\overline{s-1}}(\gamma^{2s-1}-1/4)- t_{\overline{s-1}}(\tau_{s})
	& \le \max\Big\{t_{s-1}(\gamma^{2s-2}-1/4), t_{\overline{s-1}}(\tau_{s}) \Big\}  + \frac{c_0}{\eta(1-\gamma)^2} - t_{\overline{s-1}}(\tau_{s}) \notag\\
	& = \max\Big\{t_{s-1}(\gamma^{2s-2}-1/4)- t_{\overline{s-1}}(\tau_{s}), 0 \Big\}  + \frac{c_0}{\eta(1-\gamma)^2} \notag\\
	& \leq \max\Big\{t_{s-1}(\gamma^{2s-2}-1/4)- t_{s-1}(\taun_{s-1}), 0 \Big\}  + \frac{c_0}{\eta(1-\gamma)^2} \notag\\
	& = t_{s-1}(\gamma^{2s-2}-1/4)- t_{s-1}(\taun_{s-1})   + \frac{c_0}{\eta(1-\gamma)^2}, \label{eqn:cosette2}
\end{align}
where the third line makes use of \eqref{eqn:catan:round2} in Lemma~\ref{lem:sunshine-etude}.

Applying the inequalities \eqref{eqn:cosette} and \eqref{eqn:cosette2} recursively, one arrives at 
\begin{align}
\notag 
	t_{s}(\gamma^{2s}-1/4) - t_{s}(\taun_s) & \leq  t_{s-1}(\gamma^{2s-2}-1/4)- t_{s-1}(\taun_{s-1}) + \frac{2c_0}{\eta(1-\gamma)^2} \leq \cdots \\
	& \leq  t_{3}(\gamma^{6}-1/4)- t_{3}(\taun_3) + \frac{2(s-3) c_0}{\eta(1-\gamma)^2} .\label{eq:friday}
\end{align}
To continue, note that Lemma~\ref{lem:sweets-for-new-year} and the bound \eqref{eqn:catan:round1} give
\begin{align*}
t_{3}(\gamma^{6}-1/4) \le&~\max \Big\{t_{\overline{2}}(\gamma^{5}-1/4), t_{3}(\tau_3) \Big\}  + \frac{c_0}{\eta(1-\gamma)^2}, \\
t_{\overline{2}}(\gamma^{5}-1/4) \le&~\max \Big\{t_{2}(\gamma^{4}-1/4), t_{\overline{2}}(\tau_3) \Big\}  + \frac{c_0}{\eta(1-\gamma)^2} \le \max \Big\{t_{2}(\gamma^{4}-1/4), t_{3}(\tau_3) \Big\}  + \frac{c_0}{\eta(1-\gamma)^2},
\end{align*}
which together leads to
\begin{align} \label{eq:eqn:cosette-3}
t_{3}(\gamma^{6}-1/4) \le \max \Big\{t_{2}(\gamma^{4}-1/4), t_{3}(\tau_3) \Big\}  + \frac{2c_0}{\eta(1-\gamma)^2}.
\end{align}
Plugging back to \eqref{eq:friday} leads to 
\begin{align}
\notag 	t_{s}(\gamma^{2s}-1/4) - t_{s}(\taun_s) &\leq \max\Big\{t_{2}(\gamma^{4}-1/4), t_{3}(\taun_3) \Big\}- t_{3}(\taun_3)  + \frac{2c_0}{\eta(1-\gamma)^2} + \frac{2(s-3) c_0}{\eta(1-\gamma)^2} \\
	&\leq  \frac{(2s-4) c_0}{\eta(1-\gamma)^2}, 
\end{align}
where the last step arises from the assumption \eqref{eqn:assmp2-intermediate-lemma}, that is, $t_{2}(\gamma^{4}-1/4) < t_{3}(\taun_3)$.

Further, the above inequality taken together with \eqref{eqn:cosette2} yields
\begin{align}
t_{\overline{s-1}}(\gamma^{2s-1}-1/4)- t_{\overline{s-1}}(\tau_{s})	
\leq  \frac{(2s-4) c_0}{\eta(1-\gamma)^2} + \frac{ c_0}{\eta(1-\gamma)^2}
= \frac{(2s-3) c_0}{\eta(1-\gamma)^2}.
\end{align} 
We have thus established \eqref{eqn:t-diff-part-I} and \eqref{eqn:t-diff-part-II}. 


Finally, we turn to the proof of \eqref{eqn:horn}. 
In view of \eqref{eqn:t-diff-part-II}, one has 
\begin{align*}
t_{\overline{s-2}}(\gamma^{2s-3}-1/4) - t_{\overline{s-2}}(\tau_{s-1}) 
\le \frac{2s c_0}{\eta(1-\gamma)^2}.
\end{align*}
In addition, 
\begin{align*}
t_{\overline{s-1}}(\tau_{s}) - t_{\overline{s-2}}(\tau_{s-1}) 
&\geq t_{\overline{s-1}}(\gamma \tau_{s-1}) - t_{\overline{s-2}}(\tau_{s-1}) = t_{s-1}(\tau_{s-1}) - t_{\overline{s-2}}(\tau_{s-1}) \\
&\ge t_{s-1}(\taun_{s-1}) - t_{\overline{s-2}}(\tau_{s-1}) > \frac{2s c_0}{\eta(1-\gamma)^2}, 
\end{align*}
where the identity in the first line comes from Part (iii) of Lemma~\ref{lem:basic-properties-MDP-pi},
and the last inequality uses the assumption $t_{s-1}(\taun_{s-1}) > t_{\overline{s-2}}(\tau_{s-1}) + \frac{2sc_0}{\eta(1-\gamma)^2}$. 
Combining the above two inequalities justifies the validity of the advertised inequality~\eqref{eqn:horn}.
Then, we establish \eqref{eqn:horn} for $s = 3$ through Lemma~\ref{lem:init-step}, which gives
\begin{align}
	t_{\overline{1}}\big(\gamma^{3}-1/4\big) \le t_{2}(\tau_2) \le t_{\overline{2}}(\tau_3), 
\end{align}
where the last inequality comes from~\eqref{eqn:catan:round2}.

\begin{proof}[Proof of Lemma~\ref{lem:sunshine-etude}]
To begin with, the claim \eqref{eqn:catan:round1} holds when $s=2$ as a result of  the inequality~\eqref{eqn:t1-t2-bar-scaling} in Lemma~\ref{lem:init-step}.
We now turn to the case with $3\leq s\leq H$. 
In view of the property \eqref{eq:Q-s2-pi-LB-UB-a} in Lemma~\ref{lem:basic-properties-MDP-pi}, we have
\begin{align*}
	\max\Big\{ Q^{(t)}(s,a_{0}), Q^{(t)}(s,a_{2}) \Big\} \leq \gamma^{\frac{1}{2}}\tau_{s} < \tau_s
	\qquad \text{and} \qquad
	Q^{(t)}(s,a_{1}) = \gamma V^{(t)}(\overline{s-1}).
\end{align*}
Recognizing that $V^{(t)}(s)$ is a convex combination of $\big\{ Q^{(t)}(s,a) \big\}_{a\in \{a_0,a_1,a_2\}}$, 
we know that if $V^{(t)}(s) \geq \taun_s$, then one necessarily has $Q^{(t)}(s,a_{1}) > \taun_s$, or equivalently, 
$V^{(t)}(\overline{s-1}) > \taun_s / \gamma \geq \tau_s $. 
This essentially means that $t_{s}(\taun_s) \geq t_{\overline{s-1}}(\tau_{s})$, thus establishing the claim \eqref{eqn:catan:round1}. 

Similarly,  Lemma~\ref{lem:basic-properties-MDP-pi} (cf.~\eqref{eq:Qpi-s-adj}) also tells us that
\begin{align*}
	Q^{(t)}(\overline{s},a_{0}) & =\gamma\tau_{s}
	\qquad\text{and} \qquad 
	Q^{(t)}(\overline{s},a_{1})  = \gamma V^{(t)}(s). 
\end{align*}
This means that if $V^{(t)}(s-1)\leq \taun_{s-1}$, 
then $$ V^{(t)}(\overline{s-1}) \leq \max\Big\{ Q^{(t)}(\overline{s-1},a_{0}),  Q^{(t)}(\overline{s-1},a_{1}) \Big\} \leq \gamma\taun_{s-1} \le \tau_{s} .$$ 
Consequently, we conclude that $t_{\overline{s-1}}(\tau_{s}) \geq t_{s-1}(\taun_{s-1})$, as claimed in \eqref{eqn:catan:round2}. 
\end{proof}

\section{Analysis for the blowing-up lemma (Lemma \ref{lem:grand-final-victory})}
\label{sec:proof-lem-grand-final}

In this section, we establish the blowing-up phenomenon as asserted in Lemma \ref{lem:grand-final-victory}. 
%

\subsection{Which reference point $\tprime$ shall we choose?}

Let us specify the time instance $\tprime$ as required in Lemma~\ref{lem:grand-final-victory} as follows
\begin{align}
	\label{eqn:defn-tprime}
	\tprime \defn \min~\Big\{t \in \big[\,\tlow, t_{s}(\tau_s) \,\big)~\mymid~ \cprime (1-\gamma)\pi^{(t)}(a_0 \mymid s) \leq \pi^{(t)}(a_1 \mymid s)\Big\},
\end{align}
where $\cprime \in (0,1/3)$ is some constant to be specified shortly.

\paragraph{Existence.} An important step is to justify that \eqref{eqn:defn-tprime} is well-defined, namely, there exists at least one time instance within 
$\big[\,\tlow, t_{s}(\tau_s) \,\big)$ that satisfies $\cprime (1-\gamma)\pi^{(t)}(a_0 \mymid s) \leq \pi^{(t)}(a_1 \mymid s)$. 
Towards this, we note that if the time instance $\tlow$ obeys
\begin{align*}
	\cprime (1-\gamma)\pi^{(t)}(a_0\mymid s) \leq \pi^{(t)}(a_1 \mymid s) \qquad \text{when }t=\tlow,
\end{align*}
then we simply have  $\tprime = \tlow$. We then move on to the complement case where 
\begin{align}
	\cprime (1-\gamma)\pi^{(\tlow)}(a_0\mymid s) &> \pi^{(\tlow)}(a_1 \mymid s), 
	\quad  \notag\\
	\text{or equivalently,} 
	\qquad
	\theta^{(\tlow)}(s, a_0) &> \theta^{(\tlow)}(s, a_1) - \log( \cprime (1-\gamma)).
	\label{eqn:yes-cheese-pls}
\end{align}
To justify that the construction \eqref{eqn:defn-tprime}  makes sense, 
it suffices to show that the endpoint $t_s(\taun_s)$ obeys
\begin{align}
\label{eqn:no-cheese-pls}
	\cprime (1-\gamma)\pi^{(t_s(\taun_s))}(a_0 \mymid s) < \pi^{(t_s(\taun_s))}(a_1 \mymid s).
\end{align}
In order to validate \eqref{eqn:no-cheese-pls}, recall that the inequality~\eqref{eq:pi-a2-s-LB-V-large-8923} in Lemma~\ref{lem:basic-properties-MDP-pi} ensures that 
\[
	\pi^{(t_s(\taun_s))}(a_1 \mymid s) \geq \frac{1-\gamma}{2},
\]
given that $V^{(t_s(\taun_s))}(s) \geq \taun_s$. 
Therefore, the inequality~\eqref{eqn:no-cheese-pls} must be satisfied when $\cprime<1/2$,  
given that the left-hand side of \eqref{eqn:no-cheese-pls} obeys 
\[
	\cprime (1-\gamma)\pi^{(t_s(\taun_s))}(a_0 \mymid s) \leq \cprime (1-\gamma) < \frac{1-\gamma}{2}.
\]
This in turn validates the existence of \eqref{eqn:no-cheese-pls} for this case.

\paragraph{Several immediate properties about $\tprime$ and $\tlow$.}

We pause to single out a couple of immediate properties about the $\tprime$ constructed above as well as $\tlow$.

Consider the case where $\tlow$ obeys
\begin{align*}
	\cprime (1-\gamma)\pi^{(\tlow)}(a_0\mymid s) &\leq \pi^{(\tlow)}(a_1 \mymid s), \\
	\text{or equivalently,} 
	\qquad 
	\theta^{(\tlow)}(s, a_0) &\leq \theta^{(\tlow)}(s, a_1) - \log\big(\cprime(1-\gamma) \big), 
\end{align*}
then one has $\tprime = \tlow$ (as discussed above).
As can be clearly seen, $\tlow$ satisfies the advertised inequality~\eqref{eqn:jay-concert-1} by taking $\cprime\geq \cp/8064$. 
%
%
%
Additionally, let us first recall from \eqref{eqn:horn} in  Lemma~\ref{lem:hotpot-for-new-year} that
\[
	 t_{\overline{s-1}}(\tau_s) = \max \Big\{ t_{\overline{s-2}}(\gamma^{2s-3}-1/4),~ t_{\overline{s-1}}(\tau_{s}) \Big\}.
\]
This combined with Lemma~\ref{lem:inter} (see \eqref{eqn:phantom}) tells us that
\begin{align}
	\theta^{(\tlow)}(s, a_1) =
	\theta^{(t_{\overline{s-1}}(\tau_s))}(s, a_1)
	\le \theta^{(t_{s-2}(\taun_{s-2}))}(s, a_1) 
	\leq -\frac{1}{2}\log\Big(1+\frac{\cm\gamma}{35} \eta(1-\gamma)^2t_{s-2}(\taun_{s-2}) \Big),
	\label{eq:theta-tlow-UB-Stage1}
\end{align}
where the last relation utilizes the bound \eqref{eqn:theta-t-relation-a1} in Lemma~\ref{lem:induc-theta-t-s-2}.
This leads to the advertised bound \eqref{eqn:jay-concert-2}.

As a result, the claims \eqref{eqn:jay-concert-1}-\eqref{eqn:jay-concert-2} only need to be justified under the assumption \eqref{eqn:yes-cheese-pls}.

\paragraph{Organization of the proof.} In light of the above basic facts, the subsequent proof focuses on the scenario where \eqref{eqn:yes-cheese-pls} is satisfied, namely, the case where 
\[
	\tlow < \tprime.
\]
We shall start by justifying that $\theta^{(t)}(s, a_1)$ has not increased much during $[\tlow, \tprime]$, as detailed in Appendix~\ref{sec:Stage-I-duration-a2-smaller-a0} and \ref{sec:Stage-II-duration-a2-larger-a0} (focusing on two separate stages respectively). 
This feature will then be used in Appendix~\ref{sec:proof-claim-jay-concert-12} to establish the claims \eqref{eqn:jay-concert-1}-\eqref{eqn:jay-concert-2},
and in Appendix~\ref{sec:proof-claim-jay-concert-3} to establish the claim \eqref{eqn:jay-concert-3}.

\subsection{Stage I: the duration where $\theta^{(t)}(s, a_2) < \theta^{(t)}(s, a_0)$}
\label{sec:Stage-I-duration-a2-smaller-a0}

Suppose that at the starting point we have $\theta^{(\tlow)}(s, a_2) < \theta^{(\tlow)}(s, a_0)$; otherwise the reader can proceed directly to Stage II in Appendix~\ref{sec:Stage-II-duration-a2-larger-a0}.  
The goal is to control the number of iterations taken to achieve  $\theta^{(t)}(s, a_2) \ge \theta^{(t)}(s, a_0)$. 
More specifically, let us define the transition point 
\begin{align}
	\that \defn \min ~\Big\{t \mymid \theta^{(t)}(s, a_2) \ge \theta^{(t)}(s, a_0) ,\, t\geq \tlow \Big\}.
	\label{eq:defn-that}
\end{align}
In this subsection, we seek to develop an upper bound on $\that-\tlow$, and to show that $\theta^{(t)}(s, a_1) - \theta^{(\tlow)}(s, a_1) \le 1/2$ holds throughout this stage.

%
%

\paragraph{Preparation: basic facts and rescaled policies.} 

Before moving forward, we first gather some basic facts. 
To begin with, from the definition \eqref{eqn:defn-tprime} of $\tprime$,  
we know that the inequality $\cprime (1-\gamma)\pi^{(t)}(a_0\mymid s) > \pi^{(t)}(a_1 \mymid s)$ holds true for every $t \in [\tlow,~\tprime)$, or equivalently,  
\begin{align}
	\theta^{(t)}(s, a_0) > \theta^{(t)}(s, a_1) - \log\big( \cprime(1-\gamma) \big)
	\qquad \text{for all }t \in [\tlow,~\tprime). 
	\label{eq:gap-thetaa0-thetaa1-diff-1357}
\end{align}
In the case considered here, we have --- according to \eqref{eq:gap-thetaa0-thetaa1-diff-1357} and \eqref{eq:defn-that} --- that
\begin{align}
	\theta^{(t)}(s, a_0) > \theta^{(t)}( s, a_1) - \log\big(\cprime(1-\gamma) \big)
	\qquad \text{and} \qquad 
	\theta^{(t)}(s, a_0) > \theta^{(t)}(s, a_2) 
	\label{eq:theta-t-a0-a1-a2-relative-ordering-Stage1}
\end{align}
for any $t$ obeying $\tlow\leq t < \min\{\that,\tprime\}$. 
This means that 
\begin{align}
	\theta^{(t)}(s, a_0) = \max_a\theta^{(t)}(s, a) \qquad \text{and hence} \qquad {\pi}^{(t)}(a_0 \mymid s) > 1/3
	\label{eq:theta-t-a0-a1-a2-relative-ordering-Stage1-123}
\end{align}
holds for any $t$ obeying $\tlow\leq t < \min\{\that,\tprime\}$, provided that $0< \cprime < 1$.

Moreover, let us introduce the rescaled policy $\widehat{\pi}^{(t)}(a\mymid s)$ as before 
\begin{align*}
	\widehat{\pi}^{(t)}(a \mymid s) \defn 
	\exp \Big(\theta^{(t)}(s,a) - \max_{a' \in \mathcal{A}_s}\theta^{(t)}(s, a')\Big). 
\end{align*}
In view of \eqref{eq:theta-t-a0-a1-a2-relative-ordering-Stage1-123}, the rescaled policy can therefore be written as   
\begin{align}
\label{eqn:chanchanchan}
\begin{array}{cc}
	&\widehat{\pi}^{(t)}(a_{2}\mymid s)  =\exp\big(\theta^{(t)}(s,a_{2})-\theta^{(t)}(s,a_{0})\big)=\exp\big(2\theta^{(t)}(s,a_{2})+\theta^{(t)}(s,a_{1})\big)\\
	&\widehat{\pi}^{(t)}(a_{1}\mymid s)  =\exp\big(\theta^{(t)}(s,a_{1})-\theta^{(t)}(s,a_{0})\big)=\exp\big(2\theta^{(t)}(s,a_{1})+\theta^{(t)}(s,a_{2})\big)
\end{array}
\end{align}
for any $t$ with $\tlow\leq t < \min\{\that,\tprime\}$, where we have used the constraint $\sum_{a} \theta^{(t)}(s,a) = 0$ (see Part (vii) of Lemma~\ref{lem:basic-properties-MDP-pi}).


\paragraph{Showing $\theta^{(t)}(s, a_1) - \theta^{(\tlow)}(s, a_1) \le 1/2$ by induction.} 

In the following, we seek to prove by induction the following key property
\begin{align}
	\theta^{(t)}(s, a_1) - \theta^{(\tlow)}(s, a_1) \le 1/2 
	\label{eq:claim-deviation-theta-a1-theta-a1-tlow-Stage1}
\end{align}
for any $t$ that obeys $\tlow\leq t \leq \min\{\that,\tprime\}$ and  
\begin{align}
	 t - \tlow \leq \frac{225}{\cp\cm\eta(1-\gamma)^2\exp\big(\theta^{(\tlow)}(s, a_1)\big)}
	 \eqqcolon \widetilde{t}.
	\label{eq:gap-t-tlow-UB-stage1}
\end{align}
We shall return to justify \eqref{eq:gap-t-tlow-UB-stage1} for all $t$ within this stage later on.  
In words, the claim \eqref{eq:claim-deviation-theta-a1-theta-a1-tlow-Stage1} essentially means that $\theta^{(t)}(s, a_1)$ does not deviate much from $\theta^{(\tlow)}(s, a_1)$ during this stage. 
With regards to the base case where $t = \tlow$, the hypothesis \eqref{eq:claim-deviation-theta-a1-theta-a1-tlow-Stage1} holds true trivially. 
Next, assuming that \eqref{eq:claim-deviation-theta-a1-theta-a1-tlow-Stage1} is satisfied for every integer less than or equal to $t-1$,   
we intend to establish this hypothesis for the $t$-th iteration, which is accomplished as follows.

First, Lemma~\ref{lem:basic-properties-MDP-Vstar} and Lemma~\ref{lem:non-negativity-PG} tell us that $Q^{(t)}(s, a_1) - V^{(t)}(s)\leq 1$. 
%
%
It then follows that 
\begin{align*}
\frac{\partial V^{(t)}(\mu)}{\partial \theta(s, a_1)} 
	=  \frac{1}{1-\gamma} d^{(t)}_{\mu}(s) \pi^{(t)}(a_1 \mymid s) \Big\{ Q^{(t)}(s, a_1) - V^{(t)}(s) \Big\}
\leq 14\cm\eta(1-\gamma)\pi^{(t)}(a_1 \mymid s),
\end{align*}
which relies on the bound $d_{\mu}^{(t)}(s) \leq 14 \cm (1-\gamma)^{2}$  stated in Lemma~\ref{lem:facts-d-pi-s}.
As a result, it can be derived from the PG update rule \eqref{eq:PG-update} that 
\begin{align}
\label{eqn:theta-diff-upp-viva}
	\notag \theta^{(t)}(s, a_1) - \theta^{(\tlow)}(s, a_1) 
	& =\sum_{j=\tlow}^{t-1} \eta\frac{\partial \soft{V}^{(j)}(\mu)}{\partial \theta(s, a_1)} 
\le \sum_{j=\tlow}^{t-1} 14\cm\eta(1-\gamma)\pi^{(j)}(a_1 \mymid s) \\
	& \le 14\cm\eta(1-\gamma)(t-\tlow) \max_{\tlow \leq j < t}\pi^{(j)}(a_1 \mymid s).
\end{align}
Regarding the term involving $\pi^{(j)}(a_1 \mymid s)$, we observe that for any $\tlow \leq j < t$, 
\begin{align}
 \pi^{(j)}(a_1 \mymid s)
& \overset{\mathrm{(i)}}{\leq} 
\widehat{\pi}^{(j)}(a_1 \mymid s)
	\overset{\mathrm{(ii)}}{\leq} \exp\Big(\frac{3}{2}\theta^{(j)}(s, a_1) \Big)  \label{eq:pij-a1-thetaj-a1-connection} \\
	& \overset{\mathrm{(iii)}}{\leq} \exp\Big(\frac{3}{2} \big( \theta^{(\tlow) }(s, a_1)+ 1/2 \big) \Big) .
\label{eqn:viva-brahms}
\end{align}
Here, (i) is a consequence of \eqref{eq:relation-pihat-pi}, 
(ii) holds since (in view of \eqref{eqn:chanchanchan}, $\theta^{(j)}(s,a_0)\geq 0$, and $\sum_a\theta^{(j)}(s,a)=0$)
\begin{align*}
\widehat{\pi}^{(j)}(a_{1}\mymid s) & =\exp\Big(2\theta^{(j)}(s, a_{1} )+\theta^{(j)}(s, a_{2})\Big)
  \leq\exp\Big(2\theta^{(j)}(s, a_{1})+0.5\theta^{(j)}(s, a_{2})+0.5\theta^{(j)}(s, a_{0}  )\Big) \\
	&=\exp\Big(1.5\theta^{(j)}(s, a_{1})\Big) , 
\end{align*}
whereas (iii) follows from the induction hypothesis \eqref{eq:claim-deviation-theta-a1-theta-a1-tlow-Stage1} for any $\tlow\leq j<t$.
Combine the inequalities~\eqref{eqn:theta-diff-upp-viva} and \eqref{eqn:viva-brahms} to reach
\begin{align*}
	\theta^{(t)}(s, a_1) - \theta^{(\tlow)}(s, a_1) \leq 
	14\cm\eta(1-\gamma) \big( t-\tlow \big) \exp\Big(\frac{3}{2} \big( \theta^{(\tlow) }(s, a_1)+ 1/2 \big) \Big).
\end{align*}
Consequently, under the constraint~\eqref{eq:gap-t-tlow-UB-stage1}, 
the preceding inequality implies that
\begin{align}
\label{eqn:laugh}
\notag \theta^{(t)}(s, a_1) - \theta^{(\tlow)}(s, a_1) &\leq~ 
	14\cm\eta(1-\gamma)\frac{225}{\cp\cm\eta(1-\gamma)^2\exp(\theta^{(\tlow)}(s, a_1))} \exp\Big(\frac{3}{2} \big( \theta^{(\tlow) }(s, a_1)+ 1/2 \big) \Big) \\
&=~\frac{3150 e\exp\big(\frac{1}{2}\theta^{(\tlow)}(s, a_1)\big)}{\cp(1-\gamma)} \le \frac{1}{2},
\end{align}
where the last inequality makes use of \eqref{eq:theta-tlow-UB-Stage1} and 
 the assumption~\eqref{eq:lower-bound-t-sminus2-Stage1}. 
%
%
These allow us to establish the induction hypothesis for the $t$-th iteration, namely, 
\begin{align}
\label{eqn:laugh2}
\theta^{(t)}(s, a_1) - \theta^{(\tlow)}(s, a_1)  
\le 1/2 .
\end{align}

\paragraph{Validating the constraint \eqref{eq:gap-t-tlow-UB-stage1} and upper bounding $\min\{\that,\tprime\} - \tlow$.}

It remains to justify the assumed condition~\eqref{eq:gap-t-tlow-UB-stage1} for all iteration within this stage. 
To this end, suppose instead that
\begin{align}
	\tlow + \widetilde{t} \leq \min\{\that,\tprime\} ,
	\label{eq:instead-assumption-tlow-that}
\end{align}
where $\widetilde{t}$ is defined in \eqref{eq:gap-t-tlow-UB-stage1}.
We claim that the following relation is satisfied 
\begin{align}
	\label{eqn:break-down-target}
	\widehat{\pi}^{(t)}(a_2 \mymid s) - \widehat{\pi}^{(t-1)}(a_2 \mymid s)  \ge&~\frac{\cp\cm}{150}\eta(1-\gamma)^2 \Big[\widehat{\pi}^{(t-1)}(a_2 \mymid s) \Big]^2 
\end{align}
for any $t$ obeying $\tlow \leq t \leq \tlow + \widetilde{t} \leq \min\{\that,\tprime\}$.  
Equipped with this recursive relation,
we can invoke Lemma~\ref{lem:opt-lemma} to develop a lower bound on $\widehat{\pi}^{(t)}(a_2 \mymid s)$, provided that an initial lower bound is available. 
In order to do so, in view of the expression~\eqref{eqn:chanchanchan}, we can deduce that
\begin{align*}
	\widehat{\pi}^{(\tlow)}(a_2 \mymid s) = \exp\Big(2\theta^{(\tlow)}(s, a_2)+\theta^{(\tlow)}(s, a_1) \Big) 
	\ge \exp\Big(\theta^{(\tlow)}(s, a_1) \Big),
\end{align*}
where the last relation is due to the bound $\theta^{(\tlow)}(s, a_2) \geq 0$ (see \eqref{eqn:phantom} in Lemma~\ref{lem:inter}). 
Combining the above two inequalities and applying Lemma~\ref{lem:opt-lemma} (see \eqref{eq:t0-UB-opt-lemma}), we arrive at 
$\pi^{(t)}(s, a_2) \geq 1/2$ --- and hence $\pi^{(t)}(s, a_2) \geq \pi^{(t)}(s, a_0) $ --- as soon as $t- \tlow$ exceeds 
\begin{align*}  
	  \frac{1 + \frac{\cp\cm}{100}\eta(1-\gamma)^2}{\frac{\cp\cm}{150}\eta(1-\gamma)^2 \widehat{\pi}^{(\tlow)}(a_2 \mymid s)} .
\end{align*}
This together with the definition of $\that$ thus indicates that
\begin{align*}
	 \that - \tlow 
	 & \leq \frac{1 + \frac{\cp\cm}{150}\eta(1-\gamma)^2}{\frac{\cp\cm}{150}\eta(1-\gamma)^2 \widehat{\pi}^{(\tlow)}(a_2 \mymid s)}
	  \leq 
	\frac{1 + \frac{\cp\cm}{150}\eta(1-\gamma)^2}{\frac{\cp\cm}{150}\eta(1-\gamma)^2 \exp\big( \theta^{(\tlow)}(s, a_1) \big)} \\
	& \leq 
	\frac{1.5}{\frac{\cp\cm}{150}\eta(1-\gamma)^2 \exp\big( \theta^{(\tlow)}(s, a_1) \big)},
\end{align*}
provided that $\frac{\cp\cm}{150}\eta(1-\gamma)^2 \leq 0.5$. 
This, however, contradicts the assumption \eqref{eq:instead-assumption-tlow-that}. 
As a consequence, we conclude that $\tlow + \widetilde{t} > \min\{\that,\tprime\}$, thus indicating that 
\begin{align}
	\min\{\that,\tprime\} - \tlow \leq \widetilde{t} \leq \frac{225}{\cp\cm\eta(1-\gamma)^2\exp\big(\theta^{(\tlow)}(s, a_1)\big)}. 
	\label{eq:min-that-tprime-tlow-UB}
\end{align}
%


\paragraph{Showing that $\that = \min\{\that,\tprime\} $.} 

We now justify that $\that< \tprime$, so that the upper bound \eqref{eq:min-that-tprime-tlow-UB} leads to an upper bound on $\that - \tlow$.  
Suppose instead that 
\[
	\that \geq \tprime, \qquad \text{or equivalently,} \qquad \tprime = \min\{\that,\tprime\}, 
\]
and we would like to show that this leads to contradiction. 
By definition of $\tprime$, we have
\[
	\theta^{(\tprime)}(s, a_0) \leq \theta^{(\tprime)}(s, a_1) - \log\big(\cprime(1-\gamma)\big).
\]
This further yields
\begin{align*}
	\max\Big\{ \theta^{(\tprime)}(s, a_0) , \theta^{(\tprime)}(s, a_1) \Big\}
	&\leq \theta^{(\tprime)}(s, a_1) - \log\big(\cprime(1-\gamma)\big) \notag\\
	&\leq \theta^{(\tlow)}(s, a_1) + 1/2 - \log\big(\cprime(1-\gamma)\big)  < 0 ,
\end{align*}
where the second inequality arises from \eqref{eq:claim-deviation-theta-a1-theta-a1-tlow-Stage1}, 
and the last one makes use of \eqref{eq:theta-tlow-UB-Stage1} as long as $t_{s-2}(\taun_{s-2})$.  
However, this together with the constraint $\sum_a \theta^{(\tprime)}(s, a)=0$ implies that
\[
	\theta^{(\tprime)}(s, a_2) = - \theta^{(\tprime)}(s, a_0) - \theta^{(\tprime)}(s, a_1) > 0 
	> \max\Big\{ \theta^{(\tprime)}(s, a_0) , \theta^{(\tprime)}(s, a_1) \Big\}. 
\]
which, however, implies that $\tprime > \that$ (according to the definition of $\that$) and leads to contradiction. 
As a result, we  conclude that 
\begin{align}
	\that< \tprime,
	\label{eq:that-smaller-than-tprime}
\end{align}
and the bound \eqref{eq:min-that-tprime-tlow-UB} then indicates that
\begin{align}
	\that - \tlow \leq \frac{225}{\cp\cm\eta(1-\gamma)^2\exp\big(\theta^{(\tlow)}(s, a_1)\big)}. 
	\label{eq:min-that-tlow-UB}
\end{align}

\subsubsection{Proof of the inequality~\eqref{eqn:break-down-target}}

From the relation \eqref{eqn:chanchanchan}, one can deduce that 
\begin{align}
\label{eqn:break-down-final}
\notag \widehat{\pi}^{(t)}(a_2 \mymid s) - \widehat{\pi}^{(t-1)}(a_2 \mymid s) 
=&~\exp\Big(2\theta^{(t)}(s, a_2) + \theta^{(t)}(s, a_1) \Big) - \exp\Big(2\theta^{(t-1)}(s, a_2) + \theta^{(t-1)}(s, a_1)\Big) \\
\notag =&~\widehat{\pi}^{(t-1)}(a_2 \mymid s) \Big\{ \exp \Big(2\theta^{(t)}(s, a_2) - 2\theta^{(t-1)}(s, a_2) + \theta^{(t)}(s, a_1)  - \theta^{(t-1)}(s, a_1) \Big) - 1 \Big\} \\
\notag \geq&~\widehat{\pi}^{(t-1)}(a_2 \mymid s) \Big\{ 2\theta^{(t)}(s, a_2) - 2\theta^{(t-1)}(s, a_2) + \theta^{(t)}(s, a_1)  - \theta^{(t-1)}(s, a_1)  \Big\} \\
=&~\widehat{\pi}^{(t-1)}(a_2 \mymid s) \cdot \eta\Big(2\frac{\partial \soft{V}^{(t-1)}(\mu)}{\partial \theta(s, a_2)} + \frac{\partial \soft{V}^{(t-1)}(\mu)}{\partial \theta(s, a_1)}\Big)
\end{align}
for any $t$ with $\tlow\leq t \leq \min\{\that,\tprime\}$, where the inequality above follows from the elementary fact $e^{x}-1\geq x$ for any $x\in \mathbb{R}$. 
Therefore, the difference between $\widehat{\pi}^{(t)}(a_2 \mymid s)$ and $\widehat{\pi}^{(t-1)}(a_2 \mymid s)$
depends on both $\frac{\partial \soft{V}^{(t-1)}(\mu)}{\partial \theta(s, a_2)}$ 
and $\frac{\partial \soft{V}^{(t-1)}(\mu)}{\partial \theta(s, a_1)}$,
motivating us to lower bound these two derivatives separately.

\paragraph{Step 1: bounding $\frac{\partial \soft{V}^{(t)}(\mu)}{\partial \theta(s, a_2)}$.}

First, we make the observation that for any $3 \le s < H$ and any $t \ge \tlow$, 
\begin{align}
	\label{eqn:Q-diff-a2-a0-final}
	Q^{(t)}(s, a_2) - Q^{(t)}(s, a_0) = \gamma p \big( V^{(t)}(\overline{s-2}) - \gamma \tau_{s-2} \big) 
	\ge \gamma p \big(\gamma^{2s-3} - 1/4 - \gamma \tau_{s-2} \big) \ge \frac{p}{8} = \frac{\cp(1-\gamma)}{8}
\end{align}
holds as long as $\gamma (\gamma^{2s-3} - 1/4 - \gamma \tau_s) \geq 1/8.$ 
Here, the first identity comes from \eqref{eq:Q-s2-pi-LB-UB-a} in Lemma~\ref{lem:basic-properties-MDP-pi}, 
and the first inequality holds for any $t\geq t_{\overline{s-2}}(\gamma^{2s-3}-1/4)$ --- a consequence of the monotonicity property in Lemma~\ref{lem:ascent-lemma-PG}. 
As a result, for any $t$ obeying $\tlow\leq t \leq \min\{\that,\tprime\}$ we have
\begin{align}
\notag	Q^{(t)}(s, a_2) - V^{(t)}(s) 
	&= \pi^{(t)}(a_0 \mymid s)\Big(Q^{(t)}(s, a_2) - Q^{(t)}(s, a_0)\Big) + \pi^{(t)}(a_1 \mymid s)\Big(Q^{(t)}(s, a_2) - Q^{(t)}(s, a_1)\Big)\\
	& \geq  \frac{\cp(1-\gamma)}{24} - \pi^{(t)}(a_1 \mymid s) 
	\geq \frac{\cp(1-\gamma)}{24} - \cprime (1-\gamma) \geq \frac{\cp(1-\gamma)}{48},
\end{align}
where the first inequality combines \eqref{eqn:Q-diff-a2-a0-final} with the facts that $\pi^{(t)}(a_0 \mymid s)\geq 1/3$ (see \eqref{eq:theta-t-a0-a1-a2-relative-ordering-Stage1-123}) and $0\leq Q^{(t)}(s, a_2) , Q^{(t)}(s, a_1)\leq 1$ (see Lemma~\ref{lem:basic-properties-MDP-Vstar}), and  the last line holds by observing (see \eqref{eqn:defn-tprime})
\[
	\pi^{(t)}(a_1 \mymid s) \leq \cprime (1-\gamma) \pi^{(\tlow)}(a_0 \mymid s) \leq \cprime (1-\gamma)
	\qquad \text{for all }t\in [\tlow, \tprime) 
\]
and using the assumption $\cprime \leq \cp/2$. 
Consequently, for any $t \geq \tlow$, the gradient w.r.t.~$\theta(s, a_2)$ satisfies 
%
\begin{align}
\notag 
\frac{\partial V^{(t)}(s)}{\partial \theta(s, a_2)} 
	=&~  \frac{1}{1-\gamma} d^{(t)}_{\mu}(s) \pi^{(t)}(a_2 \mymid s) \Big( Q^{(t)}(s, a_2) - V^{(t)}(s) \Big) \\
%
	\ge&~ \frac{\cp\cm\gamma}{48}(1-\gamma)^2{\pi}^{(t)}(a_2 \mymid s) \geq \frac{\cp\cm\gamma}{144}(1-\gamma)^2\widehat{\pi}^{(t)}(a_2 \mymid s)  > 0, 
	\label{eqn:brahms-forever}
\end{align}
where the first inequality above also makes use of the lower bound in Lemma~\ref{lem:facts-d-pi-s-LB}. 

 In fact, the above lower bound holds true regardless of $t$, as long as $t \geq \tlow$ where we
 have shown that $\frac{\partial V^{(t-1)}(\mu)}{\partial \theta(s, a_2)}$ is bounded from below by $0$. 
One can thus conclude that the iterate $\theta^{(t)}(s, a_2)$ increases with $t$.

\paragraph{Step 2: bounding $\frac{\partial \soft{V}^{(t)}(\mu)}{\partial \theta(s, a_1)}$.}  

Regarding the gradient w.r.t.~$\theta(s, a_1)$, we have
\begin{align*}
	 \frac{\partial V^{(t)}(\mu)}{\partial \theta(s, a_1)}
	&=  \frac{1}{1-\gamma} d^{(t)}_{\mu}(s) \pi^{(t)}(a_1 \mymid s) \Big( Q^{(t)}(s, a_1) - V^{(t)}(s) \Big) \\
	& \hspace{-0.4in} = \frac{1}{1-\gamma} d^{(t)}_{\mu}(s) \pi^{(t)}(a_1 \mymid s)
	\Big\{ \pi^{(t)}(a_0 \mymid s)\Big(Q^{(t)}(s, a_1) - Q^{(t)}(s, a_0)\Big) + \pi^{(t)}(a_2 \mymid s)\Big( Q^{(t)}(s, a_1) - Q^{(t)}(s, a_2) \Big) \Big\}\\
	& \hspace{-0.4in}  \geq  \frac{1}{1-\gamma} d^{(t)}_{\mu}(s) \pi^{(t)}(a_1 \mymid s)
	\Big(\pi^{(t)}(a_0 \mymid s) + \pi^{(t)}(a_2 \mymid s)\Big)(\gamma\tau_s - \gamma^{\frac{1}{2}}\tau_s),
\end{align*}
where the last line follows since (see Lemma~\ref{lem:basic-properties-MDP-pi} and the fact that $t\geq t_{\overline{s-1}}(\tau_s)$)
\begin{align*}
	\max\big\{ Q^{(t)}(s,a_{0}), Q^{\pi}(s,a_{2}) \big\} \leq \gamma^{\frac{1}{2}}\tau_{s}, 
	\qquad
	Q^{(t)}(s,a_{1}) = \gamma V^{(t)}(\overline{s-1}) \geq \gamma \tau_s. 
\end{align*}
In addition, recognizing that $\pi^{(t)}(a_0 \mymid s) + \pi^{(t)}(a_2 \mymid s) \leq 1$
and $d_{\mu}^{(t)}(s) \leq 14 \cm (1-\gamma)^{2}$ (see Lemma~\ref{lem:facts-d-pi-s}), 
we can continue the above bound to obtain
\begin{align}
\label{eqn:yanzi-first-day}
\frac{\partial V^{(t)}(\mu)}{\partial \theta(s, a_1)} 
\geq &~ - 14 \cm (1-\gamma)\pi^{(t)}(a_1 \mymid s)\tau_s\gamma^{\frac{1}{2}}  \frac{1-\gamma}{1+\sqrt{\gamma}} 
\geq  - 7 \cm (1-\gamma)^2 \widehat{\pi}^{(t)}(a_1 \mymid s),
\end{align}
where the last inequality is due to $\tau_{s} \leq 1/2$ and $0< \gamma < 1$ and the bound \eqref{eq:relation-pihat-pi}.

\paragraph{Step 3: connecting  $\widehat{\pi}^{(t)}(a_1 \mymid s)$ with $\widehat{\pi}^{(t)}(a_2 \mymid s)$.}

The above lower bound \eqref{eqn:yanzi-first-day} on $\frac{\partial \soft{V}^{(t)}(\mu)}{\partial \theta(s, a_1)}$ is dependent on 
$\widehat{\pi}^{(t)}(a_1 \mymid s)$. However, the desired lower bound  \eqref{eqn:break-down-target} is only a function of $\widehat{\pi}^{(t)}(a_2 \mymid s)$. 
This motivates us to investigate the connection between  $\widehat{\pi}^{(t)}(a_1 \mymid s)$ and $\widehat{\pi}^{(t)}(a_2 \mymid s)$.

To this end, let us write
\begin{align}
\label{eqn:pi-pi-relation}
	\widehat{\pi}^{(t-1)}(a_1 \mymid s)
	=
	\widehat{\pi}^{(t-1)}(a_2 \mymid s)\exp\Big(\theta^{(t-1)}(s, a_1) - \theta^{(t-1)}(s, a_2)\Big).
\end{align}
As a result, one only needs to control the quantity $\exp\big(\theta^{(t-1)}(s, a_1) - \theta^{(t-1)}(s, a_2)\big)$.
In order to do so, we make use of the induction hypothesis \eqref{eq:claim-deviation-theta-a1-theta-a1-tlow-Stage1} for the $(t-1)$-th iteration to show that 
\begin{align*}
\exp\Big(\theta^{(t-1)}(s, a_1) - \theta^{(t-1)}(s, a_2) \Big) 
&\le 
\exp\Big(\theta^{(\tlow)}(s, a_1)+ 1/2 - \theta^{(t-1)}(s, a_2)\Big) \\
	&\stackrel{\mathrm{(i)}}{\le}
\exp\Big(\theta^{(\tlow)}(s, a_1)+ 1/2 - \theta^{(\tlow)}(s, a_2)\Big) \\
	&\stackrel{\mathrm{(ii)}}{\le} 
\exp\Big(\theta^{(t_{s-2}(\taun_{s-2}))}(s, a_1)+ 1/2 \Big).
\end{align*}
Here, (i) follows from the fact that $\theta^{(t)}(s, a_2)$ increases with $t$ (see \eqref{eqn:brahms-forever}); 
and (ii) comes from the inequality \eqref{eqn:phantom} in Lemma~\ref{lem:inter} as well as \eqref{eq:theta-tlow-UB-Stage1}. 
Recalling Lemma~\ref{lem:induc-theta-t-s-2}, one has 
\begin{align}
\label{eqn:exp-diff}
\notag \exp\Big( \theta^{(t-1)}(s, a_1) - \theta^{(t-1)}(s, a_2) \Big) 
&\leq \exp\Big(\theta^{(t_{s-2}(\taun_{s-2}))}(s, a_1)+ 1/2 \Big) \\
&\leq \frac{\exp(1/2)}{\sqrt{1+\frac{\cm\gamma}{35} \eta(1-\gamma)^2 t_{s-2}(\taun_{s-2})}}
\le \frac{\gamma \cp}{1050}, 
\end{align}
where the last inequality is satisfied provided that 
$t_{s-2}(\taun_{s-2}) \geq \frac{1050^2 e}{\frac{\cm\gamma^3}{35} \eta(1-\gamma)^2 \cp^2}$.
Combining \eqref{eqn:yanzi-first-day} with \eqref{eqn:pi-pi-relation} and \eqref{eqn:exp-diff},
we arrive at 
\begin{align}
\label{eqn:gradient-part-ii-final}
\frac{\partial \soft{V}^{(t-1)}(\mu)}{\partial \theta(s, a_1)} 
\geq &~ - \frac{\cp\cm}{150} (1-\gamma)^2 \widehat{\pi}^{(t-1)}(a_2 \mymid s).	
\end{align}

\paragraph{Step 4: combining bounds.} 
Putting the above pieces together and invoking the expression~\eqref{eqn:break-down-final} yield for $\gamma > 0.96$,
\begin{align*}
\widehat{\pi}^{(t)}(a_2 \mymid s) - \widehat{\pi}^{(t-1)}(a_2 \mymid s) 
	&\geq \widehat{\pi}^{(t-1)}(a_2 \mymid s) \cdot \eta\Big(2\frac{\partial {V}^{(t-1)}(\mu)}{\partial \theta(s, a_2)} + \frac{\partial {V}^{(t-1)}(\mu)}{\partial \theta(s, a_1)}\Big) \\
	& \geq 
	\Big[ \widehat{\pi}^{(t-1)}(a_2 \mymid s) \Big]^2 \eta
\frac{\cp\cm}{150} (1-\gamma)^2 ,
\end{align*}
which concludes the proof of the advertised bound~\eqref{eqn:break-down-target}.

\subsection{Stage II: the duration where $\theta^{(t)}(s, a_2) \ge \theta^{(t)}(s, a_0)$}
\label{sec:Stage-II-duration-a2-larger-a0}

We now turn attention to the case where $t$ lies within  $[\that, \tprime)$, which is a non-empty interval according to \eqref{eq:that-smaller-than-tprime}.
In this case one has
\begin{align}
	\theta^{(t)}(s, a_2) \geq \theta^{(t)}(s, a_0), 
	\qquad \text{or equivalently,} \qquad
	\pi^{(t)}(s, a_2) \geq \pi^{(t)}(s, a_0),
	\label{eq:theta-t-a0-a0-relative-cprime-Stage2}
\end{align}
as a consequence of the definition \eqref{eq:defn-that} of $\that$. 
Again, from the definition \eqref{eqn:defn-tprime} of $\tprime$,
the inequality $\cprime (1-\gamma)\pi^{(t)}(a_0\mymid s) > \pi^{(t)}(a_1 \mymid s)$ holds true
for every $t \in [\that, ~\tprime)$, or equivalently,
\begin{align}
	\theta^{(t)}(s, a_0) \geq \theta^{(t)}(s, a_1) - \log\big(\cprime (1-\gamma) \big)
	\qquad \text{for all }t \in [\that, ~\tprime).
	\label{eq:theta-t-a0-a1-relative-cprime-Stage2}
\end{align}
The goal of this subsection is to show that 
$\theta^{(t)}(s, a_1) - \theta^{(\that)}(s, a_1) \le 1/2$ throughout this stage.

\paragraph{Preparation.} 
From the above conditions \eqref{eq:theta-t-a0-a0-relative-cprime-Stage2} and \eqref{eq:theta-t-a0-a1-relative-cprime-Stage2}, we have 
\begin{align}
	\pi^{(t)}(s, a_2)\geq \pi^{(t)}(s, a_0) \geq \pi^{(t)}(s, a_1) 
	\qquad \text{and hence}\qquad
	\pi^{(t)}(s, a_2)\geq 1/3.
	\label{eq:relative-order-pi-a2-a0-a1-Stage2-final}
\end{align}
We now look at the gradient w.r.t.~$\theta(s, a_0)$, for which we first observe that 
\begin{align}
	\notag Q^{(t)}(s, a_0) - V^{(t)}(s) =&~\pi^{(t)}(a_2 \mymid s) \Big( Q^{(t)}(s, a_0) - Q^{(t)}(s, a_2) \Big) 
	+ \pi^{(t)}(a_1 \mymid s) \Big( Q^{(t)}(s, a_0) - Q^{(t)}(s, a_1) \Big) \\
	 \stackrel{\mathrm{(i)}}{\le} &-\frac{\cp(1-\gamma)}{24} + \cprime (1-\gamma) 
	\stackrel{\mathrm{(ii)}}{\le}  -\frac{\cp(1-\gamma)}{36} < 0. \label{eqn:q-a0-final-stage}
\end{align}
Here, (i) follows from the inequalities~\eqref{eqn:Q-diff-a2-a0-final} and \eqref{eq:relative-order-pi-a2-a0-a1-Stage2-final}, whereas (ii) holds true as long as $\cprime\leq {\cp}/{72}$. Consequently, 
\[
	\frac{\partial {V}^{(t)}(\mu)}{\partial \theta(s, a_0)} =  \frac{1}{1-\gamma} d^{(t)}_{\mu}(s) \pi^{(t)}(a_0 \mymid s) \Big( Q^{(t)}(s, a_0) - V^{(t)}(s) \Big)
	< 0,
\]
%
thus indicating that $\theta^{(t)}(s, a_0)$ is decreasing with $t$.

\paragraph{Key induction hypotheses.}
Again, we seek to prove by induction that 
\begin{align}
	\theta^{(t)}(s, a_1) - \theta^{(\that)}(s, a_1) \le 1/2, 
	\qquad  t\in [\that, \tprime).
	\label{eq:theta-a1-t-that-gap-half-Stage2}
\end{align}
For the base case where $t = \that$, this claim trivially holds true. 
Now suppose that the induction hypothesis \eqref{eq:theta-a1-t-that-gap-half-Stage2} is satisfied for every iteration up to $t-1$,
and we would like to establish it for the $t$-th iteration. 
Towards this, we find it helpful to introduce another auxiliary induction hypothesis 
\begin{align}
\label{eqn:recursive-for-a0}
\widehat{\pi}^{(i)}(a_0 \mymid s) \le \frac{1}{1+\frac{\cp\cm}{288}\eta(1-\gamma)^2(i-\that)} 
	\qquad \text{for all } i\in [\that, t) .
\end{align}
As an immediate remark, this hypothesis trivially holds true when $t=\that +1$. In what follows, we shall first establish \eqref{eq:theta-a1-t-that-gap-half-Stage2} for the $t$-th iteration assuming satisfaction of \eqref{eqn:recursive-for-a0},  and then use to demonstrate that \eqref{eqn:recursive-for-a0} holds for $i=t$ as well.

\paragraph{Inductive step 1: showing that $\theta^{(t)}(s, a_1) - \theta^{(\that)}(s, a_1) \le 1/2$.}

Towards this, let us introduce for convenience another  time instance
\begin{align}
	\ttilde \defn \arg \max_{i:\,  \that \leq i < t} \theta^{(i)}(s, a_1),
	\label{eq:defn-ttilde-Stage2-finale}
\end{align}
which reflects the time when $\theta^{(i)}(s, a_1)$ reaches its maximum before iteration $t$. 
In order to establish the induction hypothesis \eqref{eq:theta-a1-t-that-gap-half-Stage2} for the $t$-th iteration, 
it is sufficient to demonstrate that 
\begin{align}
	\theta^{(\ttilde)}(s, a_1) - \theta^{(\that)}(s, a_1) \leq 1/2.
	\label{eq:theta-a1-t-that-gap-half-Stage2-tilde}
\end{align}
As before, let us employ the PG update rule \eqref{eq:PG-update} to expand  $\theta^{(\ttilde)}(s, a_1) - \theta^{(\that)}(s, a_1) $ as follows
\begin{align}
\label{eqn:standard-break-final}
\theta^{(\ttilde)}(s, a_1) - \theta^{(\that)}(s, a_1) 
=&~\sum_{i = \that}^{\ttilde -1} \eta\frac{\partial \soft{V}^{(i)}(\mu)}{\partial \theta(s, a_1)}.
\end{align}
For each gradient $\frac{\partial \soft{V}^{(i)}(\mu)}{\partial \theta(s, a_1)}$,  invoking Lemma~\ref{lem:facts-d-pi-s}, Lemma~\ref{lem:basic-properties-MDP-Vstar} and Lemma~\ref{lem:non-negativity-PG}  tells us that
\begin{align}
\label{eqn:schumann-melody}
\frac{\partial \soft{V}^{(i)}(\mu)}{\partial \theta(s, a_1)} 
	=&~  \frac{1}{1-\gamma} d^{(i)}_{\mu}(s) \pi^{(i)}(a_1 \mymid s) \big( \soft{Q}^{(i)}(s, a_1) - \soft{V}^{(i)}(s) \big)
\le 14\cm(1-\gamma)\pi^{(i)}(a_1 \mymid s).
\end{align}
In addition, a little algebra together with \eqref{eq:relative-order-pi-a2-a0-a1-Stage2-final} leads to 
\begin{align*}
	\pi^{(i)}(a_1 \mymid s) &\leq \widehat{\pi}^{(i)}(a_1 \mymid s) 
= \exp \Big(\theta^{(i)}(s, a_1) - \theta^{(i)}(s, a_2)\Big) 
	\stackrel{\mathrm{(i)}}{=}  \exp\Big( \frac{3}{2}\theta^{(i)}(s,a_{1}) + \frac{1}{2}\theta^{(i)}(s,a_{0})-\frac{1}{2}\theta^{(i)}(s,a_{2}) \Big)  \\
	&= \exp\Big(\frac{3}{2}\theta^{(i)}(s, a_1)\Big)\sqrt{\widehat{\pi}^{(i)}(a_0 \mymid s)} 
	\stackrel{\mathrm{(ii)}}{\leq} \exp\Big(\frac{3}{2}\theta^{(\ttilde)}(s, a_1)\Big)\frac{1}{\sqrt{1+ \frac{\cp\cm}{288}\eta(1-\gamma)^2(i-\that)}}
\end{align*}
for any $i$ obeying $\tlow \leq i < \ttilde$, 
where the first inequality comes from \eqref{eq:relation-pihat-pi}, (i) makes use of $\sum_a \theta^{(i)}(s, a) = 0$, and (ii) follows from the induction hypothesis~\eqref{eqn:recursive-for-a0} along 
with the definition \eqref{eq:defn-ttilde-Stage2-finale} of $\ttilde$. 
%


Putting the above bounds together with \eqref{eqn:standard-break-final} and \eqref{eqn:schumann-melody} guarantees that
\begin{align}
\notag \theta^{(\ttilde)}(s, a_1) - \theta^{(\that)}(s, a_1) 
%
\le&~\sum_{i = \that}^{\ttilde-1} 14\cm\eta(1-\gamma)\exp\Big(\frac{3}{2}\theta^{(\ttilde)}(s, a_1)\Big)\frac{1}{\sqrt{1+\frac{\cp\cm}{288}\eta(1-\gamma)^2(i-\that)}}\\
\notag =&~
\frac{14\cm\eta\exp\big(\frac{3}{2}\theta^{(\ttilde)}(s, a_1)\big)}{\sqrt{\frac{\cp\cm}{288}\eta}
	} \left\{ 1 + \sum_{i = \that+1}^{\ttilde-1} \frac{1}{\sqrt{i-\that}} \right\} \\
\label{eqn:schumann-piano-quartet}
\leq&~ 
\sqrt{\frac{225792 \cm\eta(\ttilde - \that)}{\cp}}\exp\Big(\frac{3}{2}\theta^{(\ttilde)}(s, a_1)\Big).
\end{align}
Given that $\theta^{(\ttilde)}(s, a_0) \geq \theta^{(\ttilde)}(s, a_1) - \log\big(\cprime(1-\gamma)\big)$ 
(see \eqref{eq:theta-t-a0-a1-relative-cprime-Stage2}) and $\sum_a \theta^{(\ttilde)}(s, a) = 0$, one obtains 
\begin{align*}
	\widehat{\pi}^{(\ttilde)}(a_0 \mymid s) &= \exp\Big(\theta^{(\ttilde)}(s, a_0) - \theta^{(\ttilde)}(s, a_2)\Big) 
	= \exp\Big(2\theta^{(\ttilde)}(s, a_0) + \theta^{(\ttilde)}(s, a_1)\Big)  \\
	&\ge \exp\Big(3\theta^{(\ttilde)}(s, a_1) - 2\log\big(\cprime(1-\gamma) \big) \Big),
\end{align*}
which combined with the inequality~\eqref{eqn:recursive-for-a0} thus implies that
\begin{align}
\label{eqn:exp-term-bound}
\exp\Big(\frac{3}{2}\theta^{(\ttilde)}(s, a_1) \Big) 
	\leq \cprime(1-\gamma)  \sqrt{\widehat{\pi}^{(\ttilde)}(a_0 \mymid s)} 
\leq 
\frac{\cprime(1-\gamma)}{\sqrt{\frac{\cp\cm}{288}\eta(1-\gamma)^2(\ttilde-\that)}}.
\end{align}
As a consequence of the inequalities~\eqref{eqn:schumann-piano-quartet} and \eqref{eqn:exp-term-bound}, we obtain
\begin{align}
\notag \theta^{(\ttilde)}(s, a_1) - \theta^{(\that)}(s, a_1) 
%
\le&~\sqrt{\frac{225792\cm\eta(\ttilde-\that)}{\cp}}\frac{\cprime (1-\gamma)}{\sqrt{\frac{\cp\cm}{288}\eta(1-\gamma)^2(\ttilde-\that)}} \\
\le&~\frac{8064\cprime}{\cp} < \frac{1}{2},
\end{align}
where the last line holds as long as $\cprime < \cp / 16128$. 
This in turn establishes our induction hypothesis \eqref{eq:theta-a1-t-that-gap-half-Stage2-tilde} --- and hence 
\eqref{eq:defn-ttilde-Stage2-finale} for the $t$-th iteration --- assuming satisfaction of the hypothesis \eqref{eqn:recursive-for-a0}.

\paragraph{Inductive step 2: establishing the upper bound \eqref{eqn:recursive-for-a0}.}

The next step is thus to justify the induction hypothesis \eqref{eqn:recursive-for-a0} when $i=t$. 
To do so, we first pay attention to the dynamics of $\theta^{(i)}(s, a_0)$ for any $\that \leq i\leq t$. 
Recognizing that  $\theta^{(i)}(s, a_2)=\max_a \theta^{(i)}(s, a)$ (see \eqref{eq:relative-order-pi-a2-a0-a1-Stage2-final}) and $\sum_a \theta^{(i)}(s, a)=0$,  we can express 
\begin{align*}
\widehat{\pi}^{(i)}(a_0 \mymid s) = \exp\Big(\theta^{(i)}(s, a_0) - \theta^{(i)}(s, a_2)\Big) = \exp\Big(2\theta^{(i)}(s, a_0) + \theta^{(i)}(s, a_1) \Big).
\end{align*}
This allows one to obtain 
\begin{align}
	& \widehat{\pi}^{(i)}(a_0 \mymid s) - \widehat{\pi}^{(i+1)}(a_0 \mymid s) 
	= \exp\Big(2\theta^{(i)}(s, a_0) + \theta^{(i)}(s, a_1) \Big) - \exp\Big(2\theta^{(i+1)}(s, a_0) + \theta^{(i+1)}(s, a_1) \Big) \notag\\
	& \qquad = \widehat{\pi}^{(i)}(a_0 \mymid s) \Big\{ 1 - \exp\Big( 2\theta^{(i+1)}(s, a_0) - 2\theta^{(i)}(s, a_0) + \theta^{(i+1)}(s, a_1)  - \theta^{(i)}(s, a_1) \Big) \Big\} \notag\\
	& \qquad = \widehat{\pi}^{(i)}(a_0 \mymid s) \Big\{ 1 - \exp\Big( 2\eta \frac{\partial \soft{V}^{(i)}(\mu)}{\partial \theta(s, a_0)} + \eta \frac{\partial \soft{V}^{(i)}(\mu)}{\partial \theta(s, a_1)} \Big) \Big\}.
	\label{eq:pi-i-pi-iplus1-a0-diff-Stage2}
\end{align}

With the above observation in mind, we claim for the moment the following recursive relation 
\begin{align}
\label{eqn:break-down-final-part2}
	\widehat{\pi}^{(i)}(a_0 \mymid s) - \widehat{\pi}^{(i+1)}(a_0 \mymid s) 
	\ge&~\frac{\cp\cm}{288}\eta(1-\gamma)^2 \Big[ \widehat{\pi}^{(i)}(a_0 \mymid s) \Big] ^2
\end{align}
for any $i$ obeying $\that \leq i < t$, 
whose proof is deferred to the end of this section. 
If this claim were true, then \eqref{eq:xt-sequence-LB-123} in Lemma~\ref{lem:opt-lemma} allows us to conclude the desired bound
\begin{align}
\label{eqn:recursive-for-a0-temp}
\widehat{\pi}^{(t)}(a_0 \mymid s) \le \frac{1}{1+\frac{\cp\cm}{288}\eta(1-\gamma)^2(t-\that)} .
\end{align}

\paragraph{Proof of the inequality~\eqref{eqn:break-down-final-part2}.} 

Combining~\eqref{eqn:q-a0-final-stage} and the lower bound on $d^{(i)}_{\mu}(s)$ in Lemma~\ref{lem:facts-d-pi-s-LB}, we have 
\begin{align*}
	\frac{\partial V^{(i)}(\mu)}{\partial \theta(s, a_0)} 
	\le&~\cm\gamma(1-\gamma)\pi^{(i)}(a_0 \mymid s)\Big(Q^{(i)}(s, a_0) - V^{(i)}(s)\Big) 
	\le -\frac{\cp\cm}{108}(1-\gamma)^2\widehat{\pi}^{(i)}(a_0 \mymid s),
\end{align*}
where the last inequality also makes use of \eqref{eq:relation-pihat-pi}. 
In addition, invoking the inequalities~\eqref{eqn:schumann-melody} and \eqref{eq:relation-pihat-pi} gives 
\begin{align}
\frac{\partial \soft{V}^{(i)}(\mu)}{\partial \theta(s, a_1)} 
	\le&~14\cm(1-\gamma)\pi^{(i)}(a_1 \mymid s) \le 14\cm(1-\gamma)\widehat{\pi}^{(i)}(a_1 \mymid s)   \label{eq:V-a1-UB-123456}\\
=&~ 14\cm(1-\gamma)\widehat{\pi}^{(i)}(a_0 \mymid s)\exp\Big( \theta^{(i)}(s, a_1) - \theta^{(i)}(s, a_0) \Big). \notag
\end{align}
Recall that for any $i \in [\that,~\tprime)$, one has $\theta^{(i)}(s, a_0) \geq \theta^{(i)}(s, a_1) - \log\big(\cprime(1-\gamma)\big)$,
or equivalently, 
$$
	\exp\Big(\theta^{(i)}(s, a_1) - \theta^{(i)}(s, a_0) \Big) \le \cprime (1-\gamma).
$$
%
It thus follows that 
\begin{align*}
\frac{\partial \soft{V}^{(i)}(\mu)}{\partial \theta(s, a_1)} 
\leq &~ 14\cprime \cm(1-\gamma)^2\widehat{\pi}^{(i)}(a_0 \mymid s).
\end{align*}
As a result, the above bounds taken collectively lead to
\begin{align*}
	2\frac{\partial \soft{V}^{(i)}(\mu)}{\partial \theta(s, a_0)} + \frac{\partial \soft{V}^{(i)}(\mu)}{\partial \theta(s, a_1)} 
	\leq \Big[-\frac{\cp\cm}{56}(1-\gamma)^2 + 14\cprime \cm(1-\gamma)^2 \Big]\widehat{\pi}^{(i)}(a_0 \mymid s)
	\leq -\frac{\cp\cm}{112}(1-\gamma)^2\widehat{\pi}^{(i)}(a_0 \mymid s), 
\end{align*}
provided that $\cprime / \cp < 1/1568$. 
In addition, similar to \eqref{eq:V-a1-UB-123456}, we can easily see that
\begin{subequations}
	\label{eq:v-grad-UB-13579}
\begin{align}
	\eta \Big|\frac{\partial \soft{V}^{(i)}(\mu)}{\partial \theta(s, a_1)}\Big| \le&~14\eta\cm(1-\gamma)\pi^{(i)}(a_1 \mymid s) \le 14\eta\cm(1-\gamma) \leq 1/3 ,\\
	\eta\Big|\frac{\partial \soft{V}^{(i)}(\mu)}{\partial \theta(s, a_0)}\Big| \le&~14\eta\cm(1-\gamma)\pi^{(i)}(a_0 \mymid s) \le 14\eta\cm(1-\gamma) \leq 1/3 
\end{align}
\end{subequations}
as long as $\eta\cm(1-\gamma) \leq 1/42$.

Substituting the preceding bounds into \eqref{eq:pi-i-pi-iplus1-a0-diff-Stage2}, 
we immediately arrive at
\begin{align}
\widehat{\pi}^{(i)}(a_{0}\mymid s)-\widehat{\pi}^{(i+1)}(a_{0}\mymid s) & =\widehat{\pi}^{(i)}(a_{0}\mymid s)\left\{ 1-\exp\Big(2\eta\frac{\partial \soft{V}^{(i)}(\mu)}{\partial\theta(s,a_{0})}+\eta\frac{\partial \soft{V}^{(i)}(\mu)}{\partial\theta(s,a_{1})}\Big)\right\} \notag\\
 & \ge\widehat{\pi}^{(i)}(a_{0}\mymid s)\frac{\eta}{2}\Big(-2\frac{\partial \soft{V}^{(i)}(\mu)}{\partial\theta(s,a_{0})}-\frac{\partial \soft{V}^{(i)}(\mu)}{\partial\theta(s,a_{1})}\Big)\notag\\
 & \geq\frac{\eta\cp\cm}{224}(1-\gamma)^{2}\Big[\widehat{\pi}^{(i)}(a_{0}\mymid s)\Big]^{2} , \notag
\end{align}
where the first inequality holds due to the fact $-1 \leq 2\eta\frac{\partial{V}^{(i)}(\mu)}{\partial\theta(s,a_{0})}+\eta\frac{\partial{V}^{(i)}(\mu)}{\partial\theta(s,a_{1})} \leq 0$ 
as well as the elementary inequality $1-e^x \geq -x/2$ as long as $-1\leq x\leq 0$. 
This establishes the inequality~\eqref{eqn:break-down-final-part2}. 


\subsection{Proof of the claims~\eqref{eqn:jay-concert-1} and \eqref{eqn:jay-concert-2}}
\label{sec:proof-claim-jay-concert-12}

We are now ready to justify the claims~\eqref{eqn:jay-concert-1} and \eqref{eqn:jay-concert-2}. 
Combining \eqref{eq:claim-deviation-theta-a1-theta-a1-tlow-Stage1} and \eqref{eq:theta-a1-t-that-gap-half-Stage2}, we reach
\begin{align*}
\theta^{(t)}(s,a_{1})-\theta^{(\tlow)}(s,a_{1}) & =\begin{cases}
\theta^{(t)}(s,a_{1})-\theta^{(\tlow)}(s,a_{1}), & \text{if }\tlow\leq t\leq\that\\
\Big(\theta^{(\that)}(s,a_{1})-\theta^{(\tlow)}(s,a_{1})\Big)+\Big(\theta^{(t)}(s,a_{1})-\theta^{(\that)}(s,a_{1})\Big), & \text{if }\that\leq t \leq \tprime
\end{cases}\\
 & \leq\max_{\tlow\leq i\leq\that}\Big(\theta^{(i)}(s,a_{1})-\theta^{(\tlow)}(s,a_{1})\Big)+\max_{\that\leq i<\tprime}\Big(\theta^{(i)}(s,a_{1})-\theta^{(\that)}(s,a_{1})\Big)\leq1.
\end{align*}
This taken collectively with (\ref{eq:theta-tlow-UB-Stage1}) leads to
\begin{align*}
\theta^{(\tprime)}(s,a_{1}) & \leq\theta^{(\tlow)}(s,a_{1})+1\leq-\frac{1}{2}\log\left(1+\frac{\cm\gamma}{35}\eta(1-\gamma)t_{s-2}(\tau_{s-2})\right)+1,
\end{align*}
as claimed in  \eqref{eqn:jay-concert-2}. 

In addition, recalling the definition \eqref{eqn:defn-tprime} of $\tprime$, we have
\[
	\theta^{(\tprime)}(s,a_{0})\leq\theta^{(\tprime)}(s,a_{1})-\log\big( \cprime(1-\gamma) \big), 
\]
which clearly satisfies \eqref{eqn:jay-concert-1} as long as $\cprime\geq \cp/16128$.

\subsection{Proof of the claim~\eqref{eqn:jay-concert-3}}
\label{sec:proof-claim-jay-concert-3}

Finally, we move on to analyze what happens after iteration $\tprime$, for which we focus on 
tracking the changes of $\widehat{\pi}^{(t)}(a_1 \mymid s)$. 
In this part, let us only consider the set of $t$ satisfying
$$
	\pi^{(t)}(a_1 \mymid s) \leq \pi^{(t)}(a_2 \mymid s).
$$
Note that at time $\tprime$, the inequalities~\eqref{eqn:jay-concert-1} and \eqref{eqn:jay-concert-2} 
are both satisfied, which together with the property $\pi^{(t)}(a_1 \mymid s) \leq \pi^{(t)}(a_2 \mymid s)$ yield
\begin{align*}
	\widehat{\pi}^{(t)}(a_1 \mymid s) \defn \exp\Big(\theta^{(t)}(s, a_1) - \max_a\theta^{(t)}(s, a)\Big)
	= \exp\Big(\theta^{(t)}(s, a_1) - \theta^{(t)}(s, a_2)\Big).
\end{align*}
Then, if $\cprime < \cp / 1000$, we have
\begin{align}
\notag \widehat{\pi}^{(t+1)}(a_1 \mymid s) - \widehat{\pi}^{(t)}(a_1 \mymid s) 
=&~\exp\Big(\theta^{(t+1)}(s, a_1) - \theta^{(t+1)}(s, a_2) \Big) - \exp\Big(\theta^{(t)}(s, a_1) - \theta^{(t)}(s, a_2)\Big) \\
	\notag =&~\widehat{\pi}^{(t)}(a_1 \mymid s) 
	\Big\{ \exp\Big(\theta^{(t+1)}(s, a_1) - \theta^{(t+1)}(s, a_2) - \theta^{(t)}(s, a_1) + \theta^{(t)}(s, a_2)\Big) - 1 \Big\} \\
	\notag =&~\widehat{\pi}^{(t)}(a_1 \mymid s) 
	\max \Big\{ \exp\Big(\eta \frac{\partial \soft{V}^{(t)}(\mu)}{\partial \theta(s, a_1)} - \eta \frac{\partial \soft{V}^{(t)}(\mu)}{\partial \theta(s, a_2)} \Big) - 1, ~0\Big\} \\
	\notag \leq &~\widehat{\pi}^{(t)}(a_1 \mymid s)\cdot 2\eta \max\left\{ \frac{\partial \soft{V}^{(t)}(\mu)}{\partial \theta(s, a_1)} - \frac{\partial \soft{V}^{(t)}(\mu)}{\partial \theta(s, a_2)}, ~0 \right\} \\
	\le&~56\cm\eta(1-\gamma)^2 \Big[ \widehat{\pi}^{(t)}(a_1 \mymid s) \Big]^2.
	\label{eq:finale-pihat-recursion-a1-proof-claim}
\end{align}
Here, the first inequality holds if $\eta \frac{\partial {V}^{(t)}(\mu)}{\partial \theta(s, a_1)} - \eta \frac{\partial {V}^{(t)}(\mu)}{\partial \theta(s, a_2)} \leq 1$ (given the elementary fact $e^x - 1 \leq 2x$ for any $0\leq x\leq 1$), and the last line is valid since
\begin{subequations}
\begin{align*}
\frac{\partial V^{(t)}(\mu)}{\partial\theta(s,a_{1})} & =\frac{1}{1-\gamma}d_{\mu}^{(t)}(s)\pi^{(t)}(a_{1}\mymid s)\Big(Q^{(t)}(s,a_{1})-V^{(t)}(s)\Big)\overset{(\mathrm{i})}{\leq}14\cm(1-\gamma)\pi^{(t)}(a_{1}\mymid s),\\
\frac{\partial V^{(t)}(\mu)}{\partial\theta(s,a_{2})} & =\frac{1}{1-\gamma}d_{\mu}^{(t)}(s)\left\{ \pi^{(t)}(a_{1}\mymid s)\Big(Q^{(t)}(s,a_{2})-Q^{(t)}(s,a_{1})\Big)+\pi^{(t)}(a_{0}\mymid s)\Big(Q^{(t)}(s,a_{2})-Q^{(t)}(s,a_{0})\Big)\right\} \\
 & \overset{(\mathrm{ii})}{\geq}\frac{1}{1-\gamma}d_{\mu}^{(t)}(s)\pi^{(t)}(a_{1}\mymid s)\Big(Q^{(t)}(s,a_{2})-Q^{(t)}(s,a_{1})\Big)\overset{(\mathrm{iii})}{\geq}-14\cm(1-\gamma)\pi^{(t)}(a_{1}\mymid s),
\end{align*}
\end{subequations}
where (ii) holds since $Q^{(t)}(s,a_{2})\geq Q^{(t)}(s,a_{0})$ (cf.~\eqref{eqn:Q-diff-a2-a0-final}), 
and (i) and (iii) make use of Lemma~\ref{lem:basic-properties-MDP-Vstar} and Lemma~\ref{lem:facts-d-pi-s}. 
In addition, these bounds also imply that $\eta \frac{\partial {V}^{(t)}(\mu)}{\partial \theta(s, a_1)} - \eta \frac{\partial {V}^{(t)}(\mu)}{\partial \theta(s, a_2)} \leq 1$ hold as long as $28\eta  \cm (1-\gamma)\leq 1$, thus validating the argument for the first inequality in \eqref{eq:finale-pihat-recursion-a1-proof-claim}.

Armed with the above recursive relation \eqref{eq:finale-pihat-recursion-a1-proof-claim}, 
we can invoke Lemma~\ref{lem:opt-lemma} to show that
\begin{align}
	t_{s}(\taun_s) - \tprime 
	\ge \frac{ \frac{1}{\widehat{\pi}^{(\tprime)}(a_1 \mymid s)} - \frac{1}{ \widehat{\pi}^{(t_{s}(\taun_s))}(a_1\mymid s) } }{56\cm\eta(1-\gamma)^2} 
	\ge \frac{\frac{1}{\widehat{\pi}^{(\tprime)}(a_1 \mymid s)}  - \frac{ 2}{ 1-\gamma} }{56\cm\eta(1-\gamma)^2}, 
	\label{eqn:t-diff}
\end{align}
where the last inequality holds since (in view of \eqref{eq:relation-pihat-pi} and \eqref{eq:pi-a2-s-LB-V-large-8923}).
\[
	\widehat{\pi}^{(t)}(a_1\mymid s) \geq {\pi}^{(t)}(a_1\mymid s) \geq (1-\gamma)/2
	\qquad \text{for any } t\geq t_s(\tau_s). 
\]
In order to control $t_{s}(\taun_s) - \tprime$ via \eqref{eqn:t-diff}, it remains to upper bound $\widehat{\pi}^{(\tprime)}(a_1 \mymid s)$. 
Towards this end, it is seen that
\begin{align}
\widehat{\pi}^{(\tprime)}(a_1 \mymid s) 
&= \exp\Big(\theta^{(\tprime)}(s, a_1) - \theta^{(\tprime)}(s, a_2) \Big) 
 = \exp\Big(2\theta^{(\tprime)}(s, a_1) + \theta^{(\tprime)}(s, a_0) \Big) \nonumber\\
	& \leq \exp\Big( 3\theta^{(\tprime)}(s, a_1) -\log \Big( \frac{\cp(1-\gamma)}{16128}\Big) \Big) \nonumber\\
& \le \frac{16128e^3}{\cp(1-\gamma)\Big(1+\frac{\cm\gamma}{35} \eta(1-\gamma)^2t_{s-2}(\taun_{s-2}) \Big)^{1.5}}
 \le \frac{16128e^3}{\cp(1-\gamma)\Big( \frac{\cm\gamma}{35} \eta(1-\gamma)^2t_{s-2}(\taun_{s-2}) \Big)^{1.5}}, \label{eq:pi-tref-a1}
\end{align}
where the first line uses $\sum_a \theta^{(\tprime)}(s, a)=0$, the second line relies on the inequality \eqref{eqn:jay-concert-1},
and the last one applies the inequality \eqref{eqn:jay-concert-2}. 
Substitution into the relation \eqref{eqn:t-diff} yields
\begin{align*}
t_{s}(\taun_s) - \tprime \ge 10^{-10}\cp\cm^{0.5}\eta^{0.5}(1-\gamma)^2\Big(t_{s-2}(\taun_{s-2}) \Big)^{1.5},
\end{align*}
thus establishing the advertised bound.

\end{document}